%% file: eiregret.tex
\documentclass[twoside,11pt,preprint]{article}
\usepackage{blindtext}

\usepackage{jmlr2e}

\usepackage[utf8]{inputenc} % allow utf-8 input
\usepackage[T1]{fontenc}    % use 8-bit T1 fonts
\usepackage{hyperref}       % hyperlinks
\usepackage{url}            % simple URL typesetting
\usepackage{booktabs}       % professional-quality tables
\usepackage{amsfonts}       % blackboard math symbols
\usepackage{nicefrac}       % compact symbols for 1/2, etc.
\usepackage{microtype}      % microtypography
\usepackage{xcolor}         % colors

\usepackage{lipsum}
\usepackage{graphicx}
\usepackage{epstopdf}
\usepackage{mathtools}        
\usepackage{algorithm}
\usepackage[noend]{algpseudocode}
\usepackage{bm}
\usepackage{cleveref}

\let\proof\relax

\usepackage{amsthm}
\usepackage{amsmath}
\usepackage{amssymb}
\newtheorem{assumption}{Assumption}
\newtheorem{thm}{Theorem}
\newtheorem{lem}{Lemma}
\newtheorem{cor}{Corollary}

\newtheorem{defi}{Definition}
\newtheorem{rem}{Remark}

\usepackage{subcaption} % Provides the subfigure environment

\newcommand\wbm{{\ensuremath{\bm{w}}}}
\newcommand\ubm{{\ensuremath{\bm{u}}}}
\newcommand\xbm{{\ensuremath{\bm{x}}}}
\newcommand\kbm{{\ensuremath{\bm{k}}}}
\newcommand\Kbm{{\ensuremath{\bm{K}}}}
\newcommand\fbm{{\ensuremath{\bm{f}}}}

\newcommand\ybm{{\ensuremath{\bm{y}}}}
\newcommand\vbm{{\ensuremath{\bm{v}}}}

\newcommand\hbm{{\ensuremath{\bm{h}}}}
\newcommand\Ibm{{\ensuremath{\bm{I}}}}
\newcommand\Qbm{{\ensuremath{\bm{Q}}}}
\newcommand\Dbm{{\ensuremath{\bm{D}}}}
\newcommand\Pbm{{\ensuremath{\bm{P}}}}
\newcommand\pbm{{\ensuremath{\bm{p}}}}
\newcommand\zerobold{\ensuremath{\mathbf{0}}}

\title{Bayesian Optimization with Expected Improvement: No Regret and the Choice of Incumbent}

\author{ \quad \quad Jingyi Wang  $^{*, 1}$,  Haowei Wang $^{*, 2}$, Szu Hui Ng $^{3}$, Cosmin G. Petra $^{4}$
}

\usepackage{amsopn}
\begin{document}
\newcommand{\Rbb}{\ensuremath{\mathbb{R} }}
\newcommand{\Pbb}{\ensuremath{\mathbb{P} }}
\newcommand{\Cbb}{\ensuremath{\mathbb{C} }}
\newcommand{\Ebb}{\ensuremath{\mathbb{E} }}
\newcommand{\Sbb}{\ensuremath{\mathbb{S} }}
\newcommand{\Vbb}{\ensuremath{\mathbb{V} }}
\newcommand{\Nbb}{\ensuremath{\mathbb{N} }}
\newcommand{\norm}[1]{\left\lVert {#1} \right\rVert}
\newcommand\epsbold{{\ensuremath{\boldsymbol{\epsilon}}}}

\newcommand{\coscom}[1]{\textcolor{orange}{[#1]}}
\newcommand{\cosadd}[1]{\textcolor{blue}{#1}}

\newcommand{\frank}[1]{\textcolor{purple}{[#1]}}

\maketitle

\begin{abstract}
  Expected improvement (EI) is one of the most widely used acquisition functions in Bayesian
optimization (BO). Despite its proven empirical success in applications, the cumulative regret
upper bound of EI remains an open question. In this paper, we analyze the classic noisy
Gaussian process expected improvement (GP-EI) algorithm. 
We consider the Bayesian setting, where the objective is a sample from a GP. Three commonly used  incumbents, namely the best posterior
mean incumbent (BPMI), the best sampled posterior mean incumbent (BSPMI), and the
best observation incumbent (BOI) are considered as the choices of the current best value in GP-EI. We present for the first time the cumulative regret upper
bounds of GP-EI with BPMI and BSPMI. Importantly, we show that in both cases, GP-EI
is a no-regret algorithm for both squared exponential (SE) and Matérn kernels. Further, we present for the first time that GP-EI with BOI either
achieves  a sublinear cumulative regret upper bound or has a fast converging noisy simple regret bound for SE and Matérn kernels. 
Our results provide theoretical guidance to the choice of incumbent when practitioners apply GP-EI
in the noisy setting. Numerical experiments are conducted to validate our findings. 
%  Numerical experiments of benchmark problems are presented to demonstrate the theories.
\end{abstract}


\begingroup
\renewcommand{\thefootnote}{\fnsymbol{footnote}} % symbols mode
\setcounter{footnote}{1}         % start with *
  \footnotetext{Two authors contributed equally to this work.}
  
  \renewcommand{\thefootnote}{\arabic{footnote}}   % numbers
  \setcounter{footnote}{0} 
 \stepcounter{footnote} \footnotetext{Center for Applied Scientific Computing,Lawrence Livermore National Laboratory, Livermore, CA. Email: wang125@llnl.gov}
 \stepcounter{footnote} \footnotetext{Department of Industrial Systems Engineering and Management,National University of Singapore, Singapore. Email: whw@nus.edu.sg}
  \stepcounter{footnote}\footnotetext{Department of Industrial Systems Engineering and Management,National University of Singapore, Singapore. Email: isensh@nus.edu.sg}
  \stepcounter{footnote}\footnotetext{Center for Applied Scientific Computing,Lawrence Livermore National Laboratory, Livermore, CA. Email: petra1@llnl.gov}
\endgroup
\setcounter{footnote}{0}                           
\renewcommand{\thefootnote}{\arabic{footnote}} 

\input{Sections/Introduction.tex}

\input{Sections/Bayesian.tex}

\input{Sections/Cumulative.tex}

\input{Sections/Pmregret.tex}
\input{Sections/SPMregret.tex}
\input{Sections/BOregret.tex}

%\input{Sections/EGO.tex}
%\input{Sections/Examples.tex}
\input{Sections/Conclusions.tex}

%\section{Numerical experiments}\label{se:example}
%We test extensively BO-EI with varying $\alpha$.

\section*{Acknowledgments}
This work was performed under the auspices of the U.S. Department of Energy by Lawrence Livermore National Laboratory under contract DE--AC52--07NA27344.  Release number LLNL-JRNL-2004380. This work is partially supported by Singapore Ministry of Education (MOE) Academic Research Fund (AcRF) [R-266-149–114]

\medskip
\clearpage

\appendix 
\input{Sections/Appx-Background.tex}

\input{Sections/Appx-Prep.tex}
\input{Sections/Appx-Instant.tex}

\input{Sections/Appx-regret.tex}
\input{Sections/Appx-experiments.tex}
%\setcitestyle{numbers}
\clearpage
\bibliography{bibliography}

%%%%%%%%%%%%%%%%%%%%%%%%%%%%%%%%%%%%%%%%%%%%%%%%%%%%%%%%%%%%

\end{document}

%% file: Sections/Introduction.tex
\section{Introduction}\label{se:introduction}
Bayesian optimization (BO) is a derivative-free optimization method for black-box functions~\citep{frazier2018}. 
BO has seen huge success in many applications  
such as hyperparameter tuning in machine learning~\citep{letham2019bayesian}, structural design~\citep{mathern2021}, robotics~\citep{rai2019using},  etc. %additive manufacturing~\citep{wang2023optimization}, inertial confinement fusion design~\citep{wang2023multifidelity}, etc.
 In recent years, new BO methodologies in the optimization of expensive integral objectives~\citep{frazier2022}, constrained optimization~\citep{picheny2016bayesian,ariafar2019admmbo,losalka2024no}, multi-objective optimization~\citep{binois2020kalai}, noise-free Gaussian Process bandits problem~\citep{iwazaki2025gaussianprocessupperconfidence,iwazaki2025improvedregretanalysisgaussian}, etc, continue to be developed. 
In the classic form, BO aims to solve the optimization problem 
\begin{equation} \label{eqn:opt-prob}
 \centering
  \begin{aligned}
	  &\underset{\substack{\xbm}\in C}{\text{minimize}} 
	  & & f(\xbm), \\
  \end{aligned}
\end{equation}
where $\xbm$ is the decision variable, $C\subset \Rbb^d$ represents the bound constraints on $\xbm$, and $f:\Rbb^d\to \Rbb$ is the objective function. 

A Gaussian process (GP) is used to approximate the objective function in \eqref{eqn:opt-prob}  and an acquisition function is then optimized to select the candidate sample for
the next observation ~\citep{frazier2018}. One of the most successful and widely used acquisition functions is expected improvement (EI)~\citep{jones1998efficient,schonlau1998global}.
EI computes the conditional expectation of an improvement function with respect to an incumbent (namely, the best current value) in search of the next sample. 
With a closed form, EI is simple to implement and only requires the cumulative distribution function (CDF) and probability density function (PDF) of the standard normal distribution. 
%Thanks to its effectiveness and efficiency, EI has been extended to scalable methods~\citep{eriksson2021scalable}, constrained EI~\citep{gardner2014}, etc. 
The classic Bayesian optimization with expected improvement as the acquisition function is termed as GP-EI \citep{mockus}.
%The classic noise-free BO algorithm using EI is also referred to as the Efficient Global Optimization (EGO) algorithm~\citep{jones1998efficient}.

Despite its wide adoption and success, there remain several important open questions on the performance and applications of GP-EI. In particular, existing works on the convergence behavior of GP-EI are still limited. There exist two main streams of convergence analysis for GP-EI in literature.
The first one focuses on the asymptotic convergence of its simple regret, which measures the error between the current best observation and the optimal objective function value. Under the assumptions that $f$ is bounded in norm in reproducing kernel Hilbert space (RKHS) of the GP kernel with no observation noise, the asymptotic convergence rates of GP-EI is established in~\cite{bull2011convergence}. 
Its extension to the noisy Bayesian setting, where 
$f$ is assumed as a sample from a Gaussian process (GP), has only recently been introduced by ~\cite{wang2025convergence}. 
Moreover, ~\cite{wang2025convergence} provides a rigorous analysis of the exploration–exploitation trade-off inherent in expected improvement (EI), which also forms an important component of the methodology adopted in this study.
%The RKHS assumption of the objective is also referred to as the frequentist setting. 
%Another common and important assumption of $f$ is the Bayesian setting, where $f$ is assumed to be a sample from a GP~\citep{narcowich2006sobolev,van2011information,bect2019supermartingale,srinivas2009gaussian,lederer2019uniform,steinwart2001influence}.

The other stream has focused on the cumulative regret upper bound of GP-EI in the bandit setting~\citep{lattimore2020bandit}, following the seminal work of~\cite{srinivas2009gaussian}. 
The cumulative regret, denoted as $R_T$ with $T$ samples,  measures the sum of the error between the objective value at each sample and the optimal value, thus providing a performance measure over the entire optimization process. 
However, existing works often modify the GP-EI algorithm in order to achieve the desirable sublinear cumulative regret bounds, \textit{e.g.}, by introducing multiple additional hyperparameters. 
In~\cite{wang2014theoreybo}, the authors studied a modified EI function with additional hyperparameters.  By using the lower and upper bounds of hyperparameters in the algorithm, they achieved a sublinear cumulative regret bound.
In ~\cite{nguyen17a}, the authors added an additional stopping criterion to bound the instantaneous regret in terms of the posterior standard deviation. However, their analysis depends on a pre-defined constant $\kappa > 0$
which is set to be small for good performances, yet the constant term in their regret
upper bound will explode quickly as $\kappa \to \infty$.
%it is unclear how the bound on the summation of standard deviation at the optimal point was obtained {\color{red} so what does this impacts the bounds?}. 
In ~\cite{tran2022regret}, the authors also modified the EI acquisition function
by including additional control parameters. Additionally, they used a different cumulative regret definition from this work and other previous works. In ~\cite{hu2025}, the authors
adjusted GP-EI and introduced an evaluation cost to better balance the exploration and exploitation
trade-off.  
%In~\cite{lederer2019uniform}, the author pointed out that the RKHS analysis of GP is similar to the work of error bounds for radial basis function interpolation of scattered data~\cite{wu1993local,wendland2004scattered}.
%In~\cite{wu1993local}, the local error estimate is shown to be bounded above by the densities of data points, which is defined at a given point $\xbm$ by the maximum Euclidean distance of closest points in a neighborhood of $\xbm$. This quantity is also referred to as filled distance in~\cite{stuart2018posterior} and is directly correlated with the standard deviation at $\xbm$. In~\cite{hu2022adjustedeiregret}, the authors established the upper bound of global posterior variance with filled distance.

Furthermore, despite being a critical component of GP-EI, the choice of incumbent has yet to be comprehensively analyzed in literature. Instead, several incumbents have been intuitively suggested and applied. In \cite{wang2014theoreybo}, the authors used BPMI, \textit{i.e.}, best posterior mean incumbent, while BOI, \textit{i.e.}, best observation incumbent, is used in ~\cite{nguyen17a}. Both~\cite{tran2022regret} and ~\cite{hu2025} adopted BSPMI, \textit{i.e.}, best sampled posterior mean incumbent. In all the works mentioned above, the authors do not discuss or analyze their choice of the incumbent. Both \cite{wang2014theoreybo} and ~\cite{tran2022regret} have claimed that the choice of BOI is \textit{brittle} without further analysis. Since the incumbent is fundamental to the definition of the improvement function, it is reasonable to conjecture that the cumulative regret behavior of GP-EI is significantly impacted  by it. 
Therefore, the choice of incumbent and the cumulative regret behavior are invariably intertwined. This then leads to the following fundamental questions: \textit{1) Is there some unified framework to establish the cumulative regret upper bounds of GP-EI? and 2) Can these results then be used to guide the choice of incumbent?}

To the best of our knowledge, the cumulative regret upper bound for classic GP-EI has
not been established. Here, we use the term classic GP-EI to highlight that no changes or modifications are made to the algorithm, particularly the EI function. Further, it is not clear whether GP-EI is a no-regret algorithm, where the cumulative regret satisfies $\lim_{T\to\infty}R_T/T = 0$, \textit{i.e.}, a sublinear $R_T$.
Additionally, the choice of the incumbent is not well studied and little theoretical guidance exists to help the practitioners select an appropriate incumbent in the noisy setting.  

In this paper, we aim to answer the two questions raised above under the Bayesian setting. Our contributions can be summarized as follows. 
\begin{itemize}
    \item First, we establish the cumulative regret bounds for GP-EI with different choices of incumbents. Importantly, we demonstrate for the first time that GP-EI with BPMI and BSPMI
are no-regret algorithms for squared exponential (SE) and Matérn kernels. Specifically, BPMI achieves regret upper bounds $\mathcal{O}(T^{\frac{3}{4}}\log^{\frac{d+2}{2}}(T))$ for SE kernels and $\mathcal{O}(T^{\frac{3\nu+2d}{4\nu+2d}}\log^{\frac{4\nu+d}{4\nu+2d}}(T))$  for Matérn kernels, while BSPMI has regret bounds $\mathcal{O}(T^{\frac{3}{4}} \log^{\frac{d+3}{2}}(T))$ and $\mathcal{O}(T^{\frac{3\nu+2d}{4\nu+2d}}\log^{\frac{3\nu+d}{2\nu+d}}(T))$ for SE and Matérn kernels, respectively. 
\item Second, we prove for the first time that GP-EI with BOI has either a sublinear cumulative regret bound or a fast converging noisy simple regret. This finding that GP-EI with BOI might not be an outright no-regret algorithm is consistent with the \textit{`brittle'} description of BOI in \cite{wang2014theoreybo} and ~\cite{tran2022regret}.
 \item  Third, with the obtained cumulative regret upper bounds, we compare the three incumbents based on their convergence results, and provide guidance on the choice of incumbent from a practical use perspective. Our theoretical guidance complements previous empirical benchmark results on EI algorithms~\citep{picheny2013benchmark}. 
\end{itemize}
From a technical perspective, establishing cumulative regret bounds for the classical GP-EI algorithm involves several nontrivial challenges. To address these, we introduce a set of core analytical insights and proof techniques, which might be of independent interests to other algorithms and are detailed in Section~\ref{se:challenge}. 

This paper is organized as follows. In Section~\ref{se:bo}, we describe the relevant backgrounds. 
In Section ~\ref{se:inst-regret}, we analyze the preliminary instantaneous regret for the three incumbents. We transform these instantaneous regret bounds to a summable form in Section~\ref{se:instant-reg}.
In Section ~\ref{se:regret}, we present the cumulative regret bounds of GP-EI and their rates. Numerical experiments are used to validate our findings in Section ~\ref{se:exp}. Finally, we conclude in Section ~\ref{se:conclusion}.

%{\color{red} the key research question: whether in the noisy setting, the EI strategy with a standard incumbent converges is still an open question of the Bayesian optimization problem. And if true then what convergence rate GP-EI can reach? 
%}

%% file: Sections/Bayesian.tex
\section{Background}\label{se:bo}
BO consists of two critical components, the GP surrogate model that approximates the objective $f$ in~\eqref{eqn:opt-prob} and the acquisition function that determines the iterative sample points. In this section, we first provide the basics of GP model and EI acquisition function in Section~\ref{se:gp} and~\ref{se:ei}. Then, we introduce the regret bound, maximum information gain and other relevant background information in Section~\ref{se:additionalback}. 
The discretization of $C$ and assumptions of $f$ are given in Section~\ref{se:discretizationsmoothness}. 
%{\color{red} Overview of BO should be introduced first, i.e. objective function/problem can be defined or explained formally. Now we do not formally define the problem. Then the elements of GP and incumbents make sense}
%\frank{F: I don't under this comment. The problem is defined in (1)}.

\subsection{Gaussian process}\label{se:gp}
We consider a zero mean GP with the kernel   $k(\xbm,\xbm'):\Rbb^d\times\Rbb^d\to\Rbb$, denoted as $GP(0,k(\xbm,\xbm'))$. 
At the $t$th sample $\xbm_t\in C$, we admit observation noise $\epsilon_t$ for $f$ so that the observed function value is 
$y_t = f(\xbm_t)+\epsilon_t$.  
We assume that the noise follows an independent zero mean Gaussian distribution with variance $\sigma^2$, \textit{i.e.}, $\epsilon_i \sim\mathcal{N}(0,\sigma^2), i=1,\dots,t$.
The prior distribution on 
$\xbm_{1:t}=[\xbm_1,\dots,\xbm_t]^T$, $\fbm_{1:t}=[f(\xbm_1),\dots,f(\xbm_t)]^T$
is $\fbm_{1:t}\sim\mathcal{N} (\mathbf{0}, \Kbm_t)$, 
where $\mathcal{N}$ denotes the normal distribution 
and $\Kbm_t = [k(\xbm_1,\xbm_1),\dots,k(\xbm_1,\xbm_t); \dots; k(\xbm_t,\xbm_1),\dots,k(\xbm_t,\xbm_t)]$. 
 
%In section~\ref{se:rkhs}, $\epsilon_t$ will admit a more general assumption.
Let $\ybm_{1:t}=[y_1,\dots,y_t]^T$.
The posterior mean $\mu_t$ and variance  $\sigma^2_t$  of  $\xbm_{1:t},\ybm_{1:t}$ can be inferred using Bayes' rule as
\begin{equation} \label{eqn:GP-post}
 \centering
  \begin{aligned}
  &\mu_t(\xbm)\ =\ \kbm_t(\xbm) \left(\Kbm_t+ \sigma^2 \Ibm\right)^{-1} \ybm_{1:t} \\
  &\sigma^2_t(\xbm)\ =\
k(\xbm,\xbm)-\kbm_t(\xbm)^T \left(\Kbm_t+\sigma^2 \Ibm\right)^{-1}\kbm_t(\xbm)\ ,
\end{aligned}
\end{equation}
where $\kbm_t(\xbm)= [k(\xbm_1,\xbm),\dots,k(\xbm_t,\xbm)]^T$.
We shall make the standard kernel assumption in literature~\citep{srinivas2009gaussian} that $k(\xbm,\xbm')\leq 1$ and $k(\xbm,\xbm)=1$, for $\forall \xbm,\xbm'\in C$.
Two of the most popular kernels are the SE and Matérn kernels.

\subsection{Expected improvement and incumbents }\label{se:ei}
The improvement function of $f$ given $t$ samples is defined as 
\begin{equation} \label{eqn:improvement}
 \centering
  \begin{aligned}
    I_t(\xbm) = \max\{ \xi^+_{t} -f(\xbm),0  \}, 
  \end{aligned}
\end{equation}
where $\xi^+_t$ denotes the incumbent, which serves as a measure of the best current output value.
We consider three widely adopted $\xi_t^+$ in this paper. 
First, the best posterior mean incumbent (BPMI), denoted as $\mu_t^+$, is defined by 
\begin{equation*} \label{def:best-post-mean}
 \centering
  \begin{aligned}
   \mu^+_{t}= \underset{\substack{x\in C}}{\text{min}} \ \mu_t(\xbm).
  \end{aligned}
\end{equation*}
Second, the best sampled posterior mean incumbent (BSPMI), denoted as $\mu_t^m$, is given by 
\begin{equation*} \label{def:best-sample-post-mean}
 \centering
  \begin{aligned}
   \mu^m_{t}= \underset{\substack{\xbm_i \in \xbm_{1:t}}}{\text{min}} \ \mu_t(\xbm_i).
  \end{aligned}
\end{equation*}
Third, the best observation incumbent (BOI), denoted as $y^+_t$, is defined as   
\begin{equation*} \label{def:best-obs}
 \centering
  \begin{aligned}
   y^+_{t}= \underset{\substack{y_i\in \ybm_{1:t}}}{\text{min}} \ y_i.
  \end{aligned}
\end{equation*}
The point that generates $\xi_t^+$ is denoted as $\xbm_t^+$.

The EI acquisition function is defined as the expectation of~\eqref{eqn:improvement} conditioned on $t$ samples, with a closed form expression:
\begin{equation} \label{eqn:EI-1}
 \centering
  \begin{aligned}
       EI_t(\xbm) =   (\xi_{t}^+-\mu_{t}(\xbm))\Phi(z_{t}(\xbm))+\sigma_{t}(\xbm)\phi(z_{t}(\xbm)),\\ 
  \end{aligned}
\end{equation}
where 
\begin{equation*} \label{eqn:EI-z}
 \centering
  \begin{aligned}
  z_{t}(\xbm) = \frac{\xi^+_{t}-\mu_{t}(\xbm)}{\sigma_{t}(\xbm)}.
  \end{aligned}
\end{equation*}
The functions 
$\phi$ and~$\Phi$ are the PDF and CDF of the standard normal distribution, respectively.
Throughout this paper, we refer to $\xi_t^+-\mu_t(\xbm)$ and $\sigma_t(\xbm)$ as the exploitation and exploration parts of $EI_t(\xbm)$. Equivalent forms of $EI_t(\xbm)$ useful for analysis are defined in Section~\ref{appdx:back} in the appendix.
%To further analyze the properties of EI, we use the \textit{trade-off} and $\tau$ forms of EI, given in Appendix~\ref{appdx:back}.
The next sample of GP-EI is chosen by maximizing EI over $C$, \textit{i.e.},  
\begin{equation} \label{eqn:acquisition-1}
 \centering
  \begin{aligned}
      \xbm_{t} = \underset{\substack{\xbm\in C}}{\text{argmax}}  EI_{t-1} (\xbm).
  \end{aligned}
\end{equation}
%In order to solve~\eqref{eqn:acquisition-1}, optimization algorithms including L-BFGS or random search can be used.
The GP-EI algorithm is given in Algorithm~\ref{alg:boei}.

\begin{algorithm}[H]
 \caption{GP-EI algorithm}\label{alg:boei}
  \begin{algorithmic}[1]
	  \State{Choose $k(\cdot,\cdot)$ and $T_0$ initial samples $\xbm_i, i=0,\dots,T_0$. Observe $y_i$.}   
	  \State{Train the Gaussian process model for $f$ on the initial samples.}
  \For{$t=T_0+1,T_0+2,\dots$}
	  \State{Find $\xbm_{t}$ based on~\eqref{eqn:acquisition-1}.}
	  \State{Observe $y_{t}$. \;}
	  \State{Update the GP model with the addition of $\xbm_{t}$ and $y_{t}$.\;}
%	  \State{Solve the optimization problem~\eqref{eqn:opt-prob} with the surrogate model if needed.
	  \If {Evaluation budget is exhausted} 
              \State{Exit}
	  \EndIf 
  \EndFor
  \end{algorithmic}
\end{algorithm}
\subsection{Regret measures and maximum information gain}\label{se:additionalback}
We consider the probability space $(\Omega, \mathcal{F}, \Pbb)$, where $\Omega$ is the sample
space, $\mathcal{F}$ is the $\sigma$-algebra generated by the subspace of $\Omega$ and $\Pbb$ the probability measure on $\mathcal{F}$. 
The filtration on $t-1$ samples, which is a non-increasingly ordered $\sigma$-algebra family, is denoted as $\mathcal{F}_{t-1}$.
The history $\mathcal{H}_{t-1}$ is defined by the sample $\xbm_i$ and its corresponding observation $y_i$ for all  $i=1,\dots,t-1$.
We define a filtration $\mathcal{F}_{t-1}'$ as the history $\mathcal{H}_{t-1}$ up to $t-1$ samples.
We note that $\mu_{t-1}(\xbm)$ and $\sigma_{t-1}(\xbm)$ in \eqref{eqn:GP-post} are deterministic given $\xbm_{1:t-1}$ and $\ybm_{1:t-1}$, and therefore  $\mathcal{F}'_{t-1}$. 
We use the standard notations $\mathcal{O}$ and $o$ to denote the orders of quantities, while $\tilde{\mathcal{O}}$ is the same as $\mathcal{O}$ with suppressed $\log$ terms.

The instantaneous regret $r_t$ is defined as 
 \begin{equation} \label{eqn:inst-regret}
 \centering
  \begin{aligned}
       r_t = f(\xbm_t)-f(\xbm^*)\geq 0,
  \end{aligned}
\end{equation}
where $\xbm^* = \underset{\substack{\xbm}\in C}{\text{argmin}}  f(\xbm)$ is an optimal solution.
Furthermore, the cumulative regret at iteration $T$ is 
\begin{equation} \label{eqn:cumu-regret}
 \centering
  \begin{aligned}
      R_T = \sum_{t=1}^T r_t =\sum_{t=1}^T (f(\xbm_t)-f(\xbm^*)).
  \end{aligned}
\end{equation}
The simple regret in the noise-free case is defined by $f_t^+-f(\xbm^*)$, where $f_t^+ = \underset{\substack{i}=1,\dots,t}{\text{min}}  f(\xbm_i)$ is the best current  objective. We note that in the bandit setting, where the performance of each sample point is often taken into account, cumulative regret is preferred as a measure of algorithmic performance~\citep{srinivas2009gaussian, LattimoreJMLR, lattimore2020bandit}. 
On the other hand, for an optimization algorithm seeking one best output, simple regret is considered an appropriate performance metric~\citep{pmlr-v9-grunewalder10a,bull2011convergence}. 
%\coscom{cite here as well since these are subjective statements and you do not want to take the burden of subjectivity on you}
%\frank{Haowei, can you find citations for thes= two regrets? }

%{\color{red} the purpose of this paragraph is not quite clear. rethink please whether this is required or replaced to somewhere else}

To establish our main results, namely, the cumulative regret upper bounds in the noisy case, we use existing results on maximum information gain~\citep{srinivas2009gaussian}. Information gain measures the informativeness of a set of sampling points in $C$ about $f$. 
The maximum information gain is defined below.
\begin{defi}\label{def:infogain}
  Consider a set of sample points $A\subset C$. Given $\xbm_{A}$ and its function values $\fbm_A=[f(\xbm)]_{\xbm\in A}$, the mutual information between $\fbm_A$ and the observation $\ybm_A$ is $I(\ybm_A;\fbm_A)=H(\ybm_A)-H(\ybm_A|\fbm_A)$, where $H$ is the entropy. The maximum information gain $\gamma_T$ after $T$ samples is $\gamma_T = \max_{A\subset C,|A|=T} I(\ybm_A;\fbm_A)$. 
\end{defi}
It is often used to bound the summation of posterior standard deviation $\sigma_t(\xbm)$ and is dependent on the kernel.
For common kernels such as the SE kernel and the Mat\'{e}rn kernel, $\gamma_t$ and its order of increase have been investigated previously~\citep{vakili2021information}; the best known rates for $\gamma_t$ for the SE and Mat\'{e}rn kernels are provided for reference in Lemma~\ref{lem:gammarate} in Section~\ref{appdx:back} in the appendix.
In addition, our analysis employs super-martingales with respect to filtration $\mathcal{F}'_{t-1}$ (see Definition~\ref{def:supmart} and Lemma~\ref{lem:supmart}), similar to the analysis framework of Gaussian process Thompson Sampling (GP-TS) in ~\cite{chowdhury2017kernelized}.

%{\color{red}we need to move the assumptions before our main theorem. }
%\frank{frank: we have so many theorems. It's much simpler to state it once and for all. It is possible to move this all the way to section 4.4. But I think they are ok here.}
\subsection{Discretization and smoothness assumption}\label{se:discretizationsmoothness}
As is common in the analysis of BO in the Bayesian setting, we adopt a time-varying discretization $\Cbb_t$ of $C$ with increasing cardinality that is standard in literature~\citep{srinivas2009gaussian}. 
Denote the closest point of $\xbm$ to $\Cbb_t$ as 
\begin{equation*} \label{eqn:close-point}
 \centering
  \begin{aligned}
 	[\xbm]_t := \underset{\substack{\wbm}\in \Cbb_t}{\text{minimize}} 
	   \norm{\xbm-\wbm}. 
  \end{aligned}
\end{equation*}
For any $t\in\Nbb$, $\Cbb_t$ is designed to have enough points covering $C$ so that the distance between $\xbm$ and $[\xbm]_t$ is small enough.
The cardinality of $\Cbb_t$ is chosen to be 
\begin{equation} \label{eqn:ccard}
 \centering
  \begin{aligned}
  |\Cbb_t| = (Lrdt^2)^d
  \end{aligned}
\end{equation}
in this paper. 
Details on $\Cbb_t$ can be found in~\ref{se:discretization}.

 For ease of reference, we list the assumptions used in our analysis below.
Without losing generality, the bound constraint assumption is given below.
\begin{assumption}\label{assp:constraint}
   The set $C\subseteq [0,r]^d$ is compact for $r>0$.
\end{assumption}
Next, we make smoothness assumptions on the objective function. Under the Bayesian setting, $f$ is assumed to be sampled from a GP.
We further assume that $f$ is Lipschitz continuous, as in~\cite{srinivas2009gaussian}.
\begin{assumption}\label{assp:gp}
The objective $f$ is a sample from the Gaussian process $GP(0,k(\xbm,\xbm'))$, where $k(\xbm,\xbm')\leq 1$ and $k(\xbm,\xbm)=1$. 
The objective function $f$ is Lipschitz continuous with Lipschitz constant $L$ (in 1-norm). Without losing generality, assume $L\geq\frac{1}{rd}$, where $r$ and $d$ are from Assumption~\ref{assp:constraint}.
\end{assumption}
We note that Assumption~\ref{assp:gp} can be further generalized to $f$ being Lipschitz with a $L$-dependent probability as in~\cite{srinivas2009gaussian}. 
%The RKHS assumption of $f$ is given below.
%\begin{assumption}\label{assp:rkhs}
%The objective function $f$ is in the RKHS $\mathcal{H}_k$ associated with the kernel $k(\xbm,\xbm')\leq 1$ and $k(\xbm,\xbm)=1$. Further, the RKHS norm of $f$ is bounded $\norm{f}_{\mathcal{H}_k}\leq B$ for some $B\geq 1$. 
%\end{assumption}
The Gaussian assumption on the noise is given below.
\begin{assumption}\label{assp:gaussiannoise}
  The observation noise are independent and identically distributed (i.i.d.) random samples from a zero mean Gaussian, \textit{i.e.}, $\epsilon_t\sim\mathcal{N}(0,\sigma^2),\sigma>0$ for all $t\in\Nbb$.
\end{assumption}
Assumptions~\ref{assp:constraint},~\ref{assp:gp},  and~\ref{assp:gaussiannoise} hold throughout the paper.

%% file: Sections/Cumulative.tex
\section{Instantaneous regret upper bounds}\label{se:inst-regret}
In this section, we begin by summarizing the core technical challenges and our corresponding solutions, providing an overall conceptual analysis framework and high-level intuition. We also define several key events that are essential for the subsequent analysis. We then proceed to a detailed derivation of the preliminary instantaneous regret upper bounds for the GP-EI algorithms.
\subsection{Technical challenges}\label{se:challenge}
Given the theoretical nature of this work and the presence of several technical obstacles in deriving a cumulative regret upper bound, we begin by outlining the principal challenges to facilitate the reader’s understanding of the paper. For each challenge, we also provide a brief overview of our proposed approaches to overcome it.

\paragraph{Challenge 1.} No effective instantaneous regret upper bound of GP-EI has been established for any incumbent. This is mainly due to the difficulty in bounding $\mu_{t-1}(\xbm_t)-\xi^+_{t-1}$ using $\sigma_{t-1}(\xbm_t)$, when $\xi_{t-1}^+-\mu_{t-1}(\xbm_t)<0$, or equivalently when the exploitation part of $EI_{t-1}(\xbm_t)$ is negative. 

We overcome this challenge by
using the global lower bound of $\sigma_t(\xbm)$ in Lemma~\ref{thm:gp-sigma-bound} to establish a lower bound for $EI_{t-1}(\xbm_t)$. Then, $\mu_{t-1}(\xbm_t)-\xi^+_{t-1}$ can be bounded via the definition of $EI_{t-1}$ (Lemma~\ref{lem:mu-bounded-EI}). 
We prove the instantaneous regret bounds in terms of $\sigma_{t-1}(\xbm_t)$ and $\sigma_{t-1}([\xbm]^*_t)$ for BPMI in Section~\ref{se:postmeanreg-prep}, for BSPMI in Section~\ref{se:sampledpostmeanreg-prep}, and for BOI in Section~\ref{se:bestobs-prep}. 
We note that the bound of $r_t$ for BOI requires additional conditions, as we elaborate in challenge 5. 

\paragraph{Challenge 2.}
The instantaneous regret bounds emerging from challenge 1 are still not desirable because they consist of $\sigma_{t-1}([\xbm]^*_t)$, whose summation is problem-dependent and unknown.

We address this challenge in Section~\ref{se:instant-reg}, where we eliminate terms involving $\sigma_{t-1}([\xbm]^*_t)$ but also add exploitation terms $\max\{\xi_{t-1}^+-\mu_{t-1}(\xbm_t),0\}$ in the transformed instantaneous regret bounds. We achieve this transformation through an in-depth analysis of the exploration and exploitation trade-off of EI ~\citep{wang2025convergence}. In this paper, we make major changes to the proof by adopting $t$-dependent parameters and maintaining the exploitation term $\max\{\xi_{t-1}^+-\mu_{t-1}(\xbm_t),0\}$. We discuss the sum of the exploitation term in challenge 4.

\paragraph{Challenge 3.} The third technical challenge is to analyze and choose the probability that the instantaneous regret bound from challenge 2 holds. 
At a given $t$, a tighter instantaneous regret bound holds with a smaller probability, leading to a trade-off that is crucial to obtain a sublinear cumulative regret with high probability. In particular, the probability often adopted in literature $1-\frac{1}{t^2}$~\citep{chowdhury2017kernelized} does not lead to a sublinear cumulative regret.

To address this challenge, we deliberately balance the parameters in the instantaneous regret bound and its probability of holding in Theorem~\ref{theorem:EI-sigma-star}. Further, we adopt the super-martingale theories instead of union bound for the probability analysis of cumulative regret bound. 

\paragraph{Challenge 4.} The fourth challenge is to show that the exploitation term $\max\{\xi_{t-1}^+-\mu_{t-1}(\xbm_t),0\}$ in the instantaneous regret  bound amounts to a sublinear sum, as part of the cumulative regret bound. 

For each incumbent, we develop novel and different techniques in the proof of Lemma~\ref{lemma:ei-pm-regret-1} in Section~\ref{se:postmeanreg-reg}, Lemma~\ref{lemma:bspm-inst-regret} in Section~\ref{se:bspm-reg}, and Lemma~\ref{lemma:boi-inst-regret} in Section~\ref{se:bestobs-reg}. 

\paragraph{Challenge 5.} Finally, BOI poses its own challenge that BPMI and BSPMI do not share in the noisy setting. For BOI, the exploitation $y_{t-1}^+-\mu_{t-1}(\xbm_t)$ is not guaranteed to have a decreasing lower bound in terms of $\sigma_{t-1}(\xbm_t)$ in our analysis due to the difficulty of bounding the random noise associated with $y_{t-1}^+$. Thus, we do not have a single unified instantaneous regret bound for BOI. 

We partially resolve this challenge by introducing a new inequality and consider two separate cases based on the inequality in Section~\ref{se:bestobs-prep}.  Then, we prove BOI either achieves sublinear cumulative regret as BPMI and BSPMI do, or has a converging noisy simple regret.

\subsection{Events}\label{se:discretizationevents}
%In the remainder of the paper, we describe several important confidence intervals, \textit{i.e.}, inequalities on $|f(\xbm)-\mu_{t-1}(\xbm)|$, as \textit{events}. 
%We adopt the event definitions for easy of reference and presentation.As mentioned in Section~\ref{se:additionalback}, we use the super-martingale theories to derive the bounds and probabilities in Section~\ref{se:regret}. Describing these inequalities as events simplifies the writing related to filtration and super-martingale. For instance, we can simply use the complement notation $\bar{}$ of an event to denote the case when an inequality does not hold. A similar notation choice was made in~\cite{chowdhury2017kernelized}.

{
In the remainder of the paper, we present several important confidence intervals, \textit{i.e.},  inequalities of the form 
$|f(\xbm)-\mu_{t-1}(\xbm)|$ by formulating them as events. This choice facilitates clarity and conciseness in both reference and exposition. As noted in Section~\ref{se:additionalback}, we employ super-martingale based techniques to derive the associated bounds and probabilistic guarantees in Section~\ref{se:regret}. Expressing these inequalities as events streamlines the discussion involving filtration and super-martingale. For example, we can conveniently denote the failure of an inequality using the complement of the corresponding event. The same notational convention was adopted in ~\cite{chowdhury2017kernelized}.

}

The definition of events and our choice of the confidence interval parameter $\beta_t$ is given next.  
\begin{defi}\label{def:eft} (Definition of the events $E^f(t)$ and $E^s(t)$) \\
 Let
 \begin{equation} \label{def:beta}
  \centering
  \begin{aligned}
  \beta_t = 2\log(8|\Cbb_t|\pi_t /\delta),
  \end{aligned}
  \end{equation} 
  where $\pi_t=\frac{\pi^2 t^2}{6}$ and $\Cbb_t$ satisfies~\eqref{eqn:ccard}.
  The event $E^f(t)$ is defined as follows: for all $ \xbm \in \Cbb_t$ and $t\in\Nbb$,  
   \begin{equation} \label{eqn:eft-1}
  \centering
  \begin{aligned}
        |f(\xbm)-\mu_{t-1}(\xbm)|\leq \beta_t^{1/2} \sigma_{t-1}(\xbm),
   \end{aligned}
  \end{equation} 
  The event $E^{s}(t)$ is defined as follows: for $\forall t\in\Nbb$ and a given sequence $\{\ubm_t\}$, where $\ubm_t\in C$,
   \begin{equation} \label{eqn:eft-2}
  \centering
  \begin{aligned}
        |f(\ubm_t)-\mu_{t-1}(\ubm_t)|\leq \beta_t^{1/2} \sigma_{t-1}(\ubm_t).
   \end{aligned}
  \end{equation}
\end{defi}
\begin{rem}[Choice of $\beta_t$.]
    Our choice of $\beta_t$ is a common one in literature~\citep{srinivas2009gaussian} that  ensures a high probability of $E^f(t)$ and $E^s(t)$ for all the three incumbents. We note that $\beta_t>1$. 
From Lemma~\ref{lem:discrete-fmu} and~\ref{lem:fmu-t}~\citep{srinivas2009gaussian}, we can directly infer the probability of the events $E^f(t)$, $E^s(t)$, and their intersection, which will be stated in Lemma~\ref{lem:eftprob}.
\end{rem}

\begin{rem}[Use of $\ubm_t$.]\label{remark:sevent}
   In the definition of event $E^s(t)$, we use $\ubm_t$ instead of $\xbm_t$ to emphasize that  $\ubm_t$ in the given sequence does not need to be the sample point $\xbm_t$ generated by the GP-EI algorithm. For instance, one can choose $\ubm_t = \xbm_0$ for $t=1,2,\dots$ such that $\{\ubm_t\}$ is a constant sequence. The choice of $\ubm_t$ is different and will be specified for each incumbent. In the case of BSPMI and BOI, we use two different sequences and, thus, have two different events $E^s_1(t)$ and $E^s_2(t)$ based on~\eqref{eqn:eft-2}. 
%   Additionally, we shall later define the event $E^r(t)$ to be when both $E^f(t)$ and  $E^s(t)$ are true. 
\end{rem}

\subsection{Instantaneous regret upper bounds of the three incumbents}
In this section, we derive preliminary upper bounds on the instantaneous regret of the GP-EI algorithm. In the next three subsections, we present the formal  bounds and accompanying discussions for each choice of incumbent $\xi^+_t$. 
For all three incumbents, we obtain an upper bound of $r_t$ in terms of $\sigma_{t-1}(\xbm_t)$ and $\sigma_{t-1}([\xbm^*]_t)$, where the sum of the former can be bounded by the information gain $\gamma_t$. Then Section~\ref{se:instant-reg} introduces a transformation bound that removes the dependence on $\sigma_{t-1}([\xbm^*]_t)$.

%The latter will be transformed to not depend on $\sigma_{t-1}([\xbm^*]_t)$ in Section~\ref{se:instant-reg}.

\subsubsection{BPMI instantaneous regret bound}\label{se:postmeanreg-prep}
 We define the event $E^r(t)$ based on  Definition~\ref{def:eft} for BPMI as follows. 
\begin{defi}\label{def:eft-bpmi}
  For BPMI, we use definition~\eqref{eqn:eft-1} of the event $E^f(t)$, while the event $E^{s}(t)$ is defined as~\eqref{eqn:eft-2} with $\ubm_t=\xbm_t$.
  Furthermore, the event $E^{r}(t)$ is defined as the intersection of $E^f(t)$ and $E^s(t)$. That is, $E^r(t)$ is true if and only if both~\eqref{eqn:eft-1} and~\eqref{eqn:eft-2} hold.
\end{defi}
The instantaneous regret bound for GP-EI with BPMI is given below. 
\begin{lem}\label{lem:post-mean-instregret}
  When $E^r(t)$ is true, the following bound holds for GP-EI with BPMI:  
  \begin{equation} \label{eqn:post-mean-instregret-1}
  \centering
  \begin{aligned}
          r_t \leq& (c_{\mu}(t)  +\phi(0)  +\beta_t^{1/2})\sigma_{t-1}(\xbm_{t}) + c_{\alpha}\beta_t^{1/2} \sigma_{t-1}([\xbm^*]_t)+\frac{1}{t^2},
  \end{aligned}
  \end{equation}
   where $c_{\mu }(t)=\log^{1/2}\left(\frac{t-1+\sigma^2}{2\pi \phi^2(0)c_{\sigma} \sigma^2}  \right)$ and $c_{\alpha}=1.328$.
\end{lem}

\paragraph{Proof sketch of Lemma~\ref{lem:post-mean-instregret}.}
%We first note that from here on, our analysis is written in terms of events according to Definition~\ref{def:eft}, which are equivalent to inequalities~\eqref{eqn:eft-1} and~\eqref{eqn:eft-2}.
%The events notations allow us to apply the super-martingale theories more seamlessly and follow the same event analysis notations in the GP-TS analysis in~\cite{chowdhury2017kernelized}.
First, we note that by the definition of $E^s(t)$ for BPMI (Definition~\ref{def:eft} and~\ref{def:eft-bpmi}), 
we have $|f(\xbm_t)-\mu_{t-1}(\xbm_t)|\leq \beta_t^{1/2}\sigma_{t-1}(\xbm_t)$.
Next, we examine $r_t$. Using this bound and the bounds on $|I_{t-1}(\xbm^*)-EI_{t-1}([\xbm]_t^*)|$ (Lemma~\ref{lem:compact-IL}), we obtain an initial upper bound of $r_t$: $ f(\xbm_t) - \mu^+_{t-1}+EI_{t-1}(\xbm_{t})  + \beta_t^{1/2}\sigma_{t-1}(\xbm_t)+ c_{\alpha}\beta_t^{1/2} \sigma_{t-1}([\xbm^*]_t)+\frac{1}{t^2}$.
From here, we consider two different cases: $\mu_{t-1}^+-f(\xbm_t)\geq 0$ and $\mu^+_{t-1}-f(\xbm_t)< 0$.

For the first case where $\mu_{t-1}^+-f(\xbm_t)\geq 0$,
 we can further bound $EI_{t-1}(\xbm_t)$ via again the confidence interval of $|f(\xbm_t)-\mu_{t-1}(\xbm_t)|$ and the properties of EI (Lemma~\ref{lem:EI}). Thus, the upper bound for $r_t$ becomes $\phi(0)\sigma_{t-1}(\xbm_{t})+ c_{\alpha} \beta_t^{1/2}\sigma_{t-1}([\xbm^*]_t)+\frac{1}{t^2}$.  

For the second case, we decompose $f(\xbm_t) - \mu^+_{t-1}$ into $f(\xbm_t) -\mu_{t-1}(\xbm_t)+\mu_{t-1}(\xbm_t) - \mu^+_{t-1}$, where the former has been bounded above.  Hence, our focus is on bounding $\mu_{t-1}(\xbm_t)-\mu_{t-1}^+$.
To do so, we establish a lower bound of $EI_{t-1}(\xbm_t)$ for $\forall t\in\Nbb$. 
We consider $EI_{t-1}(\xbm^+)$ and show that $EI_{t-1}(\xbm^+)\geq \sigma_{t-1}(\xbm^+)\phi(0)$.
Recall that $EI_{t-1}(\xbm_t)\geq EI_{t-1}(\xbm^+)$.
Using the global lower bound on $\sigma_{t-1}(\xbm)$ in the appendix (Lemma~\ref{thm:gp-sigma-bound}) and the properties of EI (Lemma~\ref{lem:mu-bounded-EI}), the upper bound for $\mu_{t-1}(\xbm_t)-\mu_{t-1}^+$ can be derived.
The upper bound on $r_t$ becomes $(c_{\mu}(t) +\phi(0)+\beta_t^{1/2})\sigma_{t-1}(\xbm_t)   + c_{\alpha} \beta_t^{1/2}\sigma_{t-1}([\xbm]^*_t)+\frac{1}{t^2}$, where $c_{\mu}(t)= \log^{1/2}\left(\frac{t-1+\sigma^2}{2\pi \phi^2(0) c_{\sigma}^2\sigma^2}\right)$. 

Combining the bounds in both cases leads to~\eqref{eqn:post-mean-instregret-1} in Lemma~\ref{lem:post-mean-instregret}.

Next, we make two observations in the following remarks, aiming to provide further clearance for Lemma~\ref{lem:post-mean-instregret} and its proof.
\begin{rem}[Rate of $c_{\mu}(t)$.]\label{remark:bpmi-inst-regret-param}
   The parameter $c_{\mu}(t)$ is obtained from an upper bound of $\mu_{t-1}(\xbm_t)-\xi_{t-1}^+$ and is $\mathcal{O}(\log^{1/2}(t))$, which is the same order as $\beta_t^{1/2}$ (\ref{def:beta}); hence, $c_{\mu}(t)$ does not affect the order of $r_t$.
\end{rem}
\begin{rem}[Positive lower bound for $EI_{t-1}(\xbm_t)$.]\label{remark:bpmi-inst-regret}
   A key intermediate step towards the upper bound on $r_t$ is establishing a positive lower bound for $EI_{t-1}(\xbm_t)$, which reduces at $\mathcal{O}\left(\frac{1}{\sqrt{t}}\right)$ for BPMI. Such a lower bound is possible for BPMI because of two reasons. First, its maximum exploitation part $\max_{\xbm\in C} \xi_{t-1}^+-\mu_{t-1}(\xbm)$ is bounded below by $0$. 
   In particular, since BPMI uses the posterior mean distribution of the entire domain to generate $\mu_{t-1}^+$, we can always obtain $\mu_{t-1}^+-\mu_{t-1}(\xbm^+_t)=0$. We note that this maximum exploitation lower bound holds for BSPMI as well but not BOI. 
   Second, the exploration part is subjected to the global lower bound of $\sigma_{t-1}(\xbm)$ (Lemma~\ref{thm:gp-sigma-bound}), which holds for all incumbents and acquisition functions. 
\end{rem}
%We conclude this section by pointing out that a different instantaneous regret bound for $r_t$ with a modified EI and BPMI in the frequentist setting is shown in~\cite{wang2014theoreybo}.

%%%%%%%%%%%%%%%%%%%%%%%%%%%%%%%%%%%%%%%%%%%%%%%%%%%%%%%%%%%%%%%%%%%%%%%%%%%%%%%%%%%%%%%
\subsubsection{BSPMI instantaneous regret}\label{se:sampledpostmeanreg-prep}
Recall that BSPMI chooses the best value of the GP posterior mean over the previously sampled points $\xbm_i, i=1,\dots,t$ as incumbent. We define the events $E^s(t)$ and $E^r(t)$ as follows. 
\begin{defi}\label{def:eft-bspmi}
  We use~\eqref{eqn:eft-1} as the definition of event $E^f(t)$. 
  The event $E^{s}_1(t)$ is defined as the inequality in~\eqref{eqn:eft-2} with $\ubm_t=\xbm_t$, 
  while the event $E^{s}_2(t)$ is defined as the inequality in~\eqref{eqn:eft-2} with $\ubm_{t}=\xbm_{t-1}$, 
  The event $E^{r}(t)$ is the intersection of $E^f(t)$, $E^s_1(t)$, and $E^s_2(t)$. 
  That is, $E^r(t)$ is true if and only if the three inequalities~\eqref{eqn:eft-1},~\eqref{eqn:eft-2} with $\ubm_t=\xbm_t$, and~\eqref{eqn:eft-2} with $\ubm_t=\xbm_{t-1}$ hold.
\end{defi}
\begin{rem}[Additional event for BSPMI.]\label{remark:bspmi-event}
     We note that the event $E^r(t)$ includes more inequalities than the BPMI counterpart. 
   However, only $E^s_1(t)$ will be used in the proof of Lemma~\ref{lem:sampled-post-mean-instregret}. Event $E^s_2(t)$ is only necessary in the cumulative regret analysis as the summation of $r_t$ leads to more terms for BSMPI and, thus, requires the additional sequence $E_2^s(t)$ where $\ubm_t=\xbm_{t-1}$.
\end{rem}
The instantaneous regret bound for GP-EI with BSPMI is given below. 
\begin{lem}\label{lem:sampled-post-mean-instregret}
  When $E^r(t)$ is true,  the following bound holds for GP-EI with BSPMI:
  \begin{equation} \label{eqn:sampled-post-mean-instregret-1}
  \centering
  \begin{aligned}
          r_t \leq& (c_{\mu}(t)  +\phi(0)  +2\beta_t^{1/2})\sigma_{t-1}(\xbm_{t}) + c_{\alpha}\beta_t^{1/2} \sigma_{t-1}([\xbm^*]_t)+\frac{1}{t^2},
  \end{aligned}
  \end{equation}
   where $c_{\mu }(t)=\log\left(\frac{t-1+\sigma^2}{2\pi \phi^2(0) c_{\sigma}^2\sigma^2}  \right)$ and $c_{\alpha}=1.328$.
\end{lem}
The proof of the above lemma is similar to the BPMI counterpart. We similarly consider two cases: $\mu_{t-1}^m-f(\xbm_t)\geq 0$ and $\mu^m_{t-1}-f(\xbm_t)< 0$. In both cases, 
we use the confidence interval $|f(\xbm_t)-\mu_{t-1}(\xbm_t)|\leq \beta_t^{1/2}\sigma_{t-1}(\xbm_t)$ (namely, $E^s(t)$)  one more time to obtain a  bound on $EI_{t-1}(\xbm_t)$ and, in turn, $r_t$.

Compared to BPMI, BSPMI has a bound for $r_t$ of similar form, but is $\beta_t^{1/2}\sigma_{t-1}(\xbm_t)$ larger. 
The exploitation part $\mu_{t-1}^m-\mu_{t-1}(\xbm)$ is no longer guaranteed to be negative, as is the case for BPMI, resulting in the additional $\beta_t^{1/2}\sigma_{t-1}(\xbm_t)$ term (first in the proof \eqref{eqn:sampled-post-mean-instregret-pf-2} and then in~\eqref{eqn:sampled-post-mean-instregret-1}).

\subsubsection{BOI instantaneous regret}\label{se:bestobs-prep}
Recall that BOI chooses the incumbent to be the best observation over the already sampled points $\xbm_i, i=1, \dots, t$. We define events $E^f(t)$, $E^s(t)$, and $E^r(t)$ similar to BPMI and BSPMI below; 
further, we define an additional event $E^{y}(t)$, related to simple regret, because $E^r(t)$ alone does not guarantee an upper bound on $r_t$ for BOI in our analysis (see challenge 5 in Section~\ref{se:challenge} and Remark~\ref{remark:boi-inst-regret} below). As a consequence, the BOI analysis is different from BPMI and BSPMI and our final BOI result (Theorem~\ref{thm:boi-cumulative-regret}) characterizes convergence in terms of both cumulative and a noisy simple regret.
\begin{defi}\label{def:eft-boi}
  The event $E^f(t)$ is defined via~\eqref{eqn:eft-1}. 
  The event $E^{s}_1(t)$ is defined as the inequality in~\eqref{eqn:eft-2} with $\ubm_t=\xbm_t$ and the event $E^{s}_2(t)$ is defined as the inequality in~\eqref{eqn:eft-2} with $\ubm_t=\xbm^*$. 
  The event $E^{r}(t)$ is the intersection of $E^f(t)$, $E^s_1(t)$, and $E^s_2(t)$. 
  That is, $E^r(t)$ is true if and only if~\eqref{eqn:eft-1},~\eqref{eqn:eft-2} with $\ubm_t=\xbm_t$, and~\eqref{eqn:eft-2} with $\ubm_t=\xbm^*$ all hold.
  Finally, the event $E^y(t)$ is defined via the noisy simple regret $r_t^s$ as 
   \begin{equation} \label{eqn:eft-y}
  \centering
  \begin{aligned}
       r_t^s := y^+_{t-1}-f(\xbm^*) \geq \beta_t^{1/2} \sigma_{t-1}(\xbm_t).
   \end{aligned}
  \end{equation}
\end{defi}
\begin{rem}[Noisy simple regret of EI.]\label{remark:boi-event}
   We point out that $r_t^s$ in~\eqref{eqn:eft-y} can be viewed as the noisy simple regret. In~\cite{bull2011convergence}, the simple regret of noise-free GP-EI is defined as $f^+_{t-1}-f(\xbm^*)$. Since $y^+_{t-1}$ is directly used in GP-EI with BOI, $r_t^s$ is a natural extension of simple regret in the noisy case. A converging upper bound for $r_t^s$ means that $y_t^+-f(\xbm^*) \leq 0$ as $t\to\infty$. 
\end{rem}
The instantaneous regret bound for GP-EI with BOI is given below. 
\begin{lem}\label{lem:best-observation-instregret}
  The following inequality holds for GP-EI with BOI when $E^r(t)$ and $E^y(t)$ are true:
  \begin{equation} \label{eqn:best-observation-instregret-1}
  \centering
  \begin{aligned}
          r_t \leq&  (c_y(t)+2\beta_t^{1/2}+\phi(0))\sigma_{t-1}(\xbm_{t})+c_{\alpha}\beta_t^{1/2} \sigma_{t-1}([\xbm^*]_t)+\frac{1}{t^2},
  \end{aligned}
  \end{equation}
  where $c_y(t)= \max\{\log^{\frac{1}{2}}(\frac{t-1+\sigma^2}{2\pi \phi^2(0)c_{\sigma}^2 \sigma^2}),3\}$ and $c_{\alpha}=1.328$.
\end{lem}

\begin{rem}[Technical necessity of $E^y(t)$.]\label{remark:boi-inst-regret}
  A positive lower bound on $EI_{t-1}(\xbm_t)$, as we noted in Remark~\ref{remark:bpmi-inst-regret}, is not guaranteed in the case of BOI. In particular, the maximum of exploitation  $y^+_{t-1}-\mu_{t-1}(\xbm)$ is not guaranteed to have a desirable lower bound, as opposed to the case for BPMI or BSPMI.
   Thus, at least within our analysis, BOI does not generate a straightforward bound on $r_t$ in terms of $\sigma_{t-1}(\xbm_t)$ and $\sigma_{t-1}([\xbm)]^*_t$.
   
   We use the additional inequality~\eqref{eqn:eft-y} and event $E^y(t)$ to present a reasonable convergence result.
   When $E^y(t)$ holds true, we can use techniques similar to those for BPMI to obtain an upper bound on $r_t$, by proving a positive lower bound for $EI_{t-1}(\xbm^*)$. When $E^y(t)$ is not true, the noisy simple regret $r_t^s$ is by definition bounded above by $\beta_t^{1/2}\sigma_{t-1}(\xbm_t)$, which converges to $0$ at a rate $\mathcal{O}(T^{-1/2}\log^{1/2}(T)\sqrt{\gamma_T})$ (see Theorem~\ref{thm:boi-cumulative-regret}). %While this ``mixed'' convergence result for BOI may be considered unusual, %or even unsatisfactory by optimizers used to no-regret algorithms
   %it provides theoretical convergence guarantees that, to the best our knowledge, were not previously established. 
   More importantly, GP-EI with BOI may well not possess no-regret properties as BPMI and BSPMI do. We further provide some intuition of this point in Remark~\ref{remark:boi-inst-regret-example}. 
   %The reduction in upper bound of noisy simple regret could be what the user desires. 
\end{rem}
As Remark~\ref{remark:boi-inst-regret} explains, the lack of a lower bound for $y^+_{t-1}-\mu_{t-1}(\xbm_t)$  in terms of $\sigma_{t-1}(\xbm_t)$ is the main technical obstacle to an upper bound for $r_t$. 
   In~\cite{nguyen17a}, the authors used a stopping criterion to provide a lower bound for $EI_{t-1}(\xbm_t)$ with BOI, and therefore a lower bound for  $y^+_{t-1}-\mu_{t-1}(\xbm_t)$. However, it is not entirely clear whether GP-EI has reached a satisfactory solution when the stopping criterion is met. Further, a finite exit $T$ can complicate the cumulative regret bound analysis.

In the following remark, we attempt to provide some intuition on why the behavior of GP-EI with BOI might not be no-regret when $E^y(t)$ is \textit{not} true.
\begin{rem}[Brittleness of BOI]\label{remark:boi-inst-regret-example}
 % \coscom{This remark will definitely be read in detail by the reviewers. Make sure not only it is correct and crystal clear but also that has absolutely no typos, both grammar and math.}
   
   We provide a possible explanation of why GP-EI with BOI might not display no-regret property when $E^r(t)$ is true but $E^y(t)$ is not true using the following example. 
   We consider SE and Matérn kernels. Suppose the noisy simple regret $r_n^s=y^+_{n-1}-f(\xbm^*)=-0.01$ for some $n\in\Nbb$, a possible value due to the noise. That is, the best observation is smaller than $f(\xbm^*)$. Since $y^+_t$ is monotonic, we have $y^+_t-f(\xbm^*)\leq -0.01$ for all $t\geq n$. Then, given $E^r(t)$, we can write that 
    \begin{equation*} \label{eqn:remark-boi-1}
  \centering
  \begin{aligned}
       y^+_{t-1} - \mu_{t-1}(\xbm_t) =& y^+_{t-1} -f(\xbm_t) +f(\xbm_t)-\mu_{t-1}(\xbm_t) \leq y^+_{t-1} -f(\xbm^*) +\beta_t^{1/2}\sigma_{t-1}(\xbm_t) \\ \leq & -0.01 +\beta_t^{1/2}\sigma_{t-1}(\xbm_t).\\
  \end{aligned}
  \end{equation*}
   Since $\beta_t^{1/2}\sigma_{t-1}(\xbm_t)\to 0$ for SE and Matérn  kernels, for a large enough $t$, we have 
    \begin{equation} \label{eqn:remark-boi-2}
  \centering
  \begin{aligned}
       y^+_{t-1} - \mu_{t-1}(\xbm_t) < -0.005.
  \end{aligned}
  \end{equation}
  The number $-0.005$ is chosen for convenience. 
  Given that $\sigma_{t-1}(\xbm_t)\to 0$, by~\eqref{eqn:remark-boi-2}, we have $z_{t-1}(\xbm_t)\to -\infty$, which represents the ratio between exploitation and exploration at $\xbm_t$.

  We consider the derivative of $EI_{t-1}(\xbm)$ with respect to the exploitation $y^+_{t-1} - \mu_{t-1}(\xbm)$, the exploration $\sigma_{t-1}(\xbm)$, and their ratio at $\xbm_t$.  
   From Lemma~\ref{lem:EI-ms}, the derivatives are  $\Phi(z_{t-1}(\xbm_t))$ and $\phi(z_{t-1}(\xbm_t))$, respectively.  
    As $t\to\infty$ and $z_{t-1}(\xbm_t)\to -\infty$, by Lemma~\ref{lem:EI-ms}, $\frac{\Phi(z_{t-1}(\xbm_t))}{\phi(z_{t-1}(\xbm_t))}=\mathcal{O}(-\frac{1}{z_{t-1}(\xbm_t)})=\mathcal{O}(\sigma_{t-1}(\xbm_t))$. 
  Thus, the relative importance of exploitation of EI reduces at the rate $\mathcal{O}(\sigma_{t-1}(\xbm_t))$.
 We conjecture that this rate could be decreasing to $0$ too quickly for GP-EI with BOI to be no-regret.
  We base our conjecture on the ratios between the derivatives of exploitation and exploration of two other well-understood algorithms. 
  First, this ratio of a pure exploration algorithm is $0$, which has linear cumulative regret~\citep{pmlr-v65-scarlett17a,lattimore2020bandit} and is not no-regret. 
  %\coscom{here you need to be carefull since you need to show proof of it. cite baby cite}
  Second, for the no-regret UCB, this ratio is $\mathcal{O}\left(\frac{1}{\beta_t^{1/2}}\right)$, which is $\mathcal{O}\left(\frac{1}{\log^{1/2}(t)}\right)$ in the Bayesian setting, and is therefore orders larger than $\mathcal{O}(\sigma_{t-1}(\xbm_t))$ for SE and Matérn kernels 
  \footnote{One can see this point by estimating the order of $\sum_{t=1}^T \frac{1}{\log^{1/2}(t)}$ using the logarithmic integral function  and compare it to the order of $\sum_{t=1}^T \sigma_{t-1}(\xbm_t) \leq \sqrt{C_{\gamma}T\gamma_T}$ using $\gamma_T$ for the kernels.}.
  Thus, we posit that BOI's trade-off between exploitation and exploration might be moving towards pure exploration too fast to attain a sublinear cumulative regret for GP-EI, when $y_{t-1}^+-f(\xbm^*)$ becomes small and negative.

  In some literature, BOI is called brittle~\citep{wang2014theoreybo,tran2022regret}.
   Further, BOI is found to perform relatively poorly in the noisy case in~\cite{picheny2013benchmark}.
   This can be explained by the example above where GP-EI turns quickly to exploration if a relatively negative noise (large in absolute value) is encountered at some $t$. In addition, once such a noise is observed, the behavior of GP-EI does not reverse itself due to its monotonicity of $_T^+$. 
\end{rem}

%% file: Sections/Pmregret.tex
\section{Instantaneous regret bound transformation}\label{se:instant-reg}
%In this section, we transform the instantaneous regret bounds in Lemmas~\ref{lem:post-mean-instregret},~\ref{lem:sampled-post-mean-instregret}, and~\ref{lem:best-observation-instregret}. The reason for the transformation is that unlike $\sum_{t=1}^T \sigma_{t-1}(\xbm_t)$, the upper bound of  $\sum_{t=1}^T \sigma_{t-1}([\xbm^*]_t)$ is not unclear and likely algorithm-dependent. 
In this section, we transform the instantaneous regret bounds established in Lemmas~\ref{lem:post-mean-instregret}, \ref{lem:sampled-post-mean-instregret}, and \ref{lem:best-observation-instregret}. This transformation is necessary because, unlike $\sum_{t=1}^T \sigma_{t-1}(\xbm_t)$, the upper bound of $\sum_{t=1}^T \sigma_{t-1}([\xbm^*]_t)$ is generally unclear and likely dependent on the specific algorithm.
Accordingly, the changes to the upper bounds are threefold.  

\begin{itemize}
    \item First, we replace $\sigma_{t-1}([\xbm^*]_t)$ with $\sigma_{t-1}(\xbm_t)$, where the parameters multiplying $\sigma_{t-1}(\xbm_t)$ become larger than those in
Lemma~\ref{lem:post-mean-instregret},~\ref{lem:sampled-post-mean-instregret}, and~\ref{lem:best-observation-instregret}. 
    \item Second, the exploitation term $\max\{\xi_{t-1}^+-\mu_{t-1}(\xbm_t),0\}$ appears in the transformed upper bound.

   \item Third, given any filtration $\mathcal{F}'_{t-1}$ that $E^r(t)$ is true, we present the transformed bound on $r_t$ with a probability ($\geq 1-\frac{1}{2t^{\frac{\alpha}{2}}}$, $\alpha\in(0,1]$), while the bound in Lemma~\ref{lem:post-mean-instregret},~\ref{lem:sampled-post-mean-instregret}, and~\ref{lem:best-observation-instregret} does not have one. 
\end{itemize}
We explain more on the three points above. First, after the transformation, the parameters multiplying $\sigma_{t-1}(\xbm_t)$ are larger as they incorporate the effect from terms involving $\sigma_{t-1}([\xbm]^*_t)$. Clearly, the new parameters need to be  controlled so that a sublinear cumulative regret bound is possible. 
Second, the exploitation term $\max\{\xi_{t-1}^+-\mu_{t-1}(\xbm_t),0\}$ appears in the transformed upper bound, as a result of the trade-off between exploitation and exploration of EI. We bound $\sum_{t=1}^T \max\{\xi_{t-1}^+-\mu_{t-1}(\xbm_t),0\}$ for each incumbent with different techniques in Section~\ref{se:regret}.
This point corresponds to challenge 4 in Section~\ref{se:challenge}. Finally, the probability given $\mathcal{F}'_{t-1}$ ($\geq 1-\frac{1}{2t^{\frac{\alpha}{2}}}$) is attentively designed so that the instantaneous regret bound holds with a high-enough probability, but also not too large that the corresponding parameters multiplying $\sigma_{t-1}(\xbm_t)$ become too large to generate sublinear cumulative regret. This is our solution to challenge 3 mentioned in Section~\ref{se:challenge}. 

%The transformation (Theorem~\ref{theorem:EI-sigma-star}) is important in producing a viable cumulative regret bound since the upper bound of $\sum_{t-1}^T\sigma_{t-1}([\xbm]^*_t)$ is algorithm-dependent and, therefore, unknown. We achieve this transformation by revealing new exploration and exploitation properties of the non-convex EI functions and comparing the $EI_{t-1}$ values at $\xbm_t$ and $[\xbm^*]_t$. 
{
The transformation introduced in Theorem~\ref{theorem:EI-sigma-star} plays a critical role in deriving a tractable cumulative regret bound, as the upper bound of the quantity $\sum_{t-1}^T\sigma_{t-1}([\xbm]^*_t)$ is inherently algorithm-dependent and thus not explicitly known. The key to achieve this transformation is to reveal new exploration and exploitation properties of the non-convex EI functions and comparing the $EI_{t-1}$ values at $\xbm_t$ and $[\xbm^*]_t$. }

A similar analysis of EI properties was first presented in ~\citep{wang2025convergence} with time-independent parameters. Here, we significantly evolve the proof by introducing time-varying parameters and probabilities. More details can be found in the proof sketch and proof for Theorem~\ref{theorem:EI-sigma-star}.

We define the following parameters: 
 \begin{equation} \label{def:eta}
  \centering
  \begin{aligned}
   \alpha_t = \alpha \log(t), \ 
 \zeta_t^{1/2} =\frac{\sqrt{2\pi}\phi(0)}{\sqrt{\alpha}} t^{\frac{\alpha}{2}}, \  \eta_t^{1/2} = \zeta_t^{1/2} \beta_t^{1/2}, 
   \end{aligned}
  \end{equation} 
where $\alpha\in(0,1]$. 
Recall that $\beta_t$ is defined in~\eqref{def:beta}.
 We emphasize that all these parameters are used as auxiliaries in the analysis, and are not involved in the GP-EI algorithm.

We explain the parameters in~\eqref{def:eta} next. 
First, we consider $\alpha_t$. We state a lemma similar to Lemma~\ref{lem:fmu}, which is the well-known confidence interval on $|f(\xbm)-\mu_{t-1}(\xbm)|$, but with parameters $\alpha$ and  $\alpha_t$ for ease of reference. 
\begin{lem}\label{lem:ei-r-1}
  For any given $\mathcal{F}'_{t-1}$, one has that
  \begin{equation} \label{eqn:ft-bound-RKHS-1}
  \centering
  \begin{aligned}
        |f(\xbm)-\mu_{t-1}(\xbm)|\leq \alpha_t^{1/2} \sigma_{t-1}(\xbm),
   \end{aligned}
  \end{equation} 
  with probability $\geq 1-\frac{1}{t^{\frac{\alpha}{2}}}$ for given $\xbm\in C$, where $\alpha\in(0,1]$, $\alpha_t = \alpha \log(t)$.
\end{lem}

     From Lemma~\ref{lem:ei-r-1}, it is clear that $\alpha_t$ serves the same purpose as $\beta_t$ since it is used in the confidence interval $|f(\xbm)-\mu_{t-1}(\xbm)|$. However, unlike $\beta_t$ which leads to $|f(\xbm)-\mu_{t-1}(\xbm)|\leq \beta_t^{1/2}\sigma_{t-1}(\xbm)$ for all $\xbm\in\Cbb_t$ and $t\in\Nbb$, $\alpha_t$ bounds $|f(\xbm)-\mu_{t-1}(\xbm)|\leq \alpha_t^{1/2}\sigma_{t-1}(\xbm)$ at given $\xbm$ with $\mathcal{F}_{t-1}'$. 
The choice of $\alpha\in(0,1]$ allows us to use a significantly smaller $\alpha_t$ compared to $\beta_t$, the latter having an equivalent $\alpha$ of $4d+4$. The reason we need a smaller $\alpha_t$ is to in turn generate small enough $\zeta_t$ and $\eta_t$, as we explain below.

    The parameters $\zeta_t$ and $\eta_t$ in~\eqref{def:eta} are chosen to ensure $\eta_t^{1/2}\Phi(-\alpha_t^{1/2})>\phi(0)$, which is necessary in the exploration-exploitation trade-off analysis of EI and used in the proof of Theorem~\ref{theorem:EI-sigma-star}.    
The relationship between the parameters $\alpha_t$, $\beta_t$, $\zeta_t$, and $\eta_t$ is formalized in the next lemma.
\begin{lem}\label{lem:EI-regret-parameters}
    The parameters $\alpha_t$, $\beta_t$, $\zeta_t$, and $\eta_t$ satisfy
    \begin{equation} \label{eqn:EI-regret-parameters}
  \centering
  \begin{aligned}
     \eta_t^{1/2} \Phi(-\alpha_t^{1/2})=\zeta_t^{1/2}\beta_t^{1/2} \Phi(-\alpha_t^{1/2}) > \phi(0). 
   \end{aligned}
  \end{equation}
\end{lem}
\begin{rem}[Choice of $\zeta_t$ and $\eta_t$.]\label{remark:parameter-relation}
    The parameters $\zeta_t$ and $\eta_t$ are chosen to be as small as they can be while satisfying~\eqref{eqn:EI-regret-parameters}. That is, $\alpha_t$ and $\beta_t$ essentially determine $\zeta_t$ and $\eta_t$. 
    We add that if the standard parameter $\beta_t$ is used in place of $\alpha_t$, 
    the corresponding $\zeta_t$ would be $\mathcal{O}(t^{2d+2})$, which leads to $c_1(t)$ and $c_2(t)$ in Theorem~\ref{theorem:EI-sigma-star} that are too large to achieve sublinear cumulative regret upper bounds. 
\end{rem}
To resolve the smaller probability of $|f(\xbm)-\mu_{t-1}(\xbm)|\leq \alpha_t^{1/2}\sigma_{t-1}(\xbm)$ holding from requiring $\alpha \in (0,1]$, we use the super-martingale theories, similar to the GP-TS analysis in~\cite{chowdhury2017kernelized} instead of union bound. This point is further discussed in Remark~\ref{remark:small-probability} and repeatedly utilized in the proof for Section~\ref{se:regret}.

For simplicity of presentation and without losing generality, we assume throughout the remainder of the paper that  
 \begin{equation} \label{eqn:EI-sigma-star-assp}
  \centering
  \begin{aligned}
   c_{\xi}(t)+\phi(0)+(c_{\alpha}+2)\beta_t^{1/2} \leq (4+\zeta_t^{1/2})\beta_t^{1/2}, 
   \end{aligned}
  \end{equation}
 where $c_{\alpha}=1.328$, $c_{\xi}(t)=c_{\mu}(t)=\log^{\frac{1}{2}}(\frac{t-1+\sigma^2}{2\pi \phi^2(0)c_{\sigma}^2 \sigma^2})$ for both BPMI and BSPMI, and $c_{\xi}(t)=c_y(t)= \max\{c_{\mu}(t),3\}$ for BOI.
 The inequality~\eqref{eqn:EI-sigma-star-assp} is always true for $t$ large enough since the left-hand side is $\mathcal{O}(log^{1/2}(t))$ and the right-hand side is $\mathcal{O}({t^{\frac{\alpha}{2}}}\log^{1/2}(t))$.
Hence,~\eqref{eqn:EI-sigma-star-assp} does not lose generality. This relationship simplifies the upper bound in Theorem~\ref{theorem:EI-sigma-star} and is used towards the end of its proof.

The transformed instantaneous regret bound is now given in the following theorem. 
\begin{thm}\label{theorem:EI-sigma-star}
   From the parameters in~\eqref{def:eta}, let $c_1(t)= \frac{\eta_t^{1/2}}{\phi(0)}$ and $c_{2}(t)= (4+\zeta^{1/2}_t)\beta_t^{1/2} $.
   For any filtration $\mathcal{F}_{t-1}'$ with BPMI or BSPMI, if the corresponding $E^r(t)$ is true,
   then we have with probability $\geq 1- \frac{1}{2t^{\frac{\alpha}{2}}}$ that
      \begin{equation} \label{eqn:EI-sigma-star-1}
  \centering
  \begin{aligned}
      r_t  \leq c_1(t) \max\{\xi^+_{t-1}-\mu_{t-1}(\xbm_{t}),0\} + c_2(t) \sigma_{t-1}(\xbm_{t})+\frac{1}{t^2}. 
   \end{aligned}
  \end{equation}
  For BOI and any filtration $\mathcal{F}_{t-1}'$, if $E^r(t)$ is true with probability $\geq 1- \frac{1}{2t^{\frac{\alpha}{2}}}$, then either~\eqref{eqn:EI-sigma-star-1} is true or $E^y(t)$ is false.  
\end{thm} 

\paragraph{Proof sketch for Theorem~\ref{theorem:EI-sigma-star}}\label{remark:explorevexploit}
Starting from the upper bound of $r_t$ in Lemmas~\ref{lem:post-mean-instregret},~\ref{lem:sampled-post-mean-instregret}, and~\ref{lem:best-observation-instregret}, we consider two scenarios based on the value of $f(\xbm_{t})-f([\xbm^*]_t)$. 
Scenario A is when $f(\xbm_{t})-f([\xbm^*]_t) \leq  c_1(t)\max\{\xi^+_{t-1}-\mu_{t-1}(\xbm_{t}),0\}+c_2(t)\sigma_{t-1}(\xbm_{t})$, which implies~\eqref{eqn:EI-sigma-star-1}. 

Scenario B is when $f(\xbm_{t})-f([\xbm^*]_t) > c_1(t)\max\{\xi^+_{t-1}-\mu_{t-1}(\xbm_{t}),0\}+c_2(t)\sigma_{t-1}(\xbm_{t})$.
In this scenario, we can derive a lower bound for $\xi_{t-1}^+-\mu_{t-1}([\xbm]_t^*) - \xi_{t-1}^+-\mu_{t-1}(\xbm_t)$, which represents the exploitation of $EI_{t-1}([\xbm]^*_t)$ subtracted by the exploitation of $EI_{t-1}(\xbm_t)$, with probability $\geq 1-\frac{1}{2t^{\alpha/2}}$. 
Since  $EI_{t-1}(\xbm_t) \geq EI_{t-1}([\xbm]^*_t)$ and $EI_{t-1}(\xbm)$ is monotonic with respect to both its exploration $\sigma_{t-1}(\xbm)$ and exploitation $\xi^+_{t-1}-\mu_{t-1}(\xbm)$, we prove by contradiction that the exploration of $EI_{t-1}([\xbm]^*_t)$ and $EI_{t-1}(\xbm_t)$ satisfies $\sigma_{t-1}([\xbm]^*_t) \leq \sigma_{t-1}(\xbm_t)$ with probability $\geq 1-\frac{1}{2t^{\alpha/2}}$. 
The contradiction stems from our analysis that under Scenario B, $EI_{t-1}(\xbm_t) < EI_{t-1}([\xbm]^*_t)$ with probability $\geq 1-\frac{1}{2t^{\alpha/2}}$ if $\sigma_{t-1}([\xbm]^*_t) > \sigma_{t-1}(\xbm_t)$, which we elaborate next.

Suppose $\sigma_{t-1}([\xbm]^*_t) > \sigma_{t-1}(\xbm_t)$. We consider two cases distinguished by the relationship between $\sigma_{t-1}([\xbm]^*_t)$ and $\sigma_{t-1}(\xbm_t)$ under Scenario B, represented by~\eqref{eqn:EI-sigma-star-pm-pf-14} and~\eqref{eqn:EI-sigma-star-pm-pf-16}, respectively. While the first case (marked Case 1 in the proof) is rather straightforward, the second case (marked Case 2) requires more complex analysis. 
We further divide the proof into Case 2.1 where $z(\xbm_t)<0$ and Case 2.2 where $z(\xbm_t)\geq 0$. Due to the non-convexity of EI, to prove that $EI_{t-1}(\xbm_t) < EI_{t-1}([\xbm]^*_t)$ with probability $\geq 1-\frac{1}{2t^{\alpha/2}}$, we distinguish Case 2.1 into another two layers of cases (\textit{e.g.}, Case 2.1.1.1), as each of them requires its own analysis. For Case 2.2, we adopt existing results from Case 2.1 through our choice of $c_1(t)$ and similarly prove $EI_{t-1}(\xbm_t) < EI_{t-1}([\xbm]^*_t)$ with probability $\geq 1-\frac{1}{2t^{\alpha/2}}$.

Combining all the Cases and Scenarios, we can then replace $\sigma_{t-1}([\xbm^*_t])$ with $\sigma_{t-1}(\xbm_t)$ in ~\eqref{eqn:post-mean-instregret-1},~\eqref{eqn:sampled-post-mean-instregret-1}, and~\eqref{eqn:best-observation-instregret-1} from Lemma~\ref{lem:post-mean-instregret},~\ref{lem:sampled-post-mean-instregret}, and~\ref{lem:best-observation-instregret}. 

\begin{rem}[Reduced probability of inequalities]\label{remark:small-probability}
As we mentioned above in this section, as well as in challenge 3 of Section~\ref{se:challenge}, $\alpha$ and $\alpha_t$ are critical analytic choices leading to sublinear cumulative regret bounds.
From Remark~\ref{remark:parameter-relation}, $\alpha_t$ essentially determines $c_1(t)$ and $c_2(t)$. The larger $\alpha_t$ is, the larger $c_1(t)$ and $c_2(t)$ need to be. 
Hence, for a smaller upper bound on $r_t$, $\alpha_t$ should be smaller, $\alpha$ should be smaller, and the probability of~\eqref{eqn:EI-sigma-star-1} holding is also smaller. 

On the other hand, we need to consider the probability of~\eqref{eqn:EI-sigma-star-1} \textbf{not} holding when we compute the cumulative regret bound and its probability. To reach sublinear cumulative regret with an arbitrary probability ($1-\delta$), the probability of~\eqref{eqn:EI-sigma-star-1} not holding needs to be small enough. Hence, $\alpha$ should be large enough. 

Clearly, $\alpha$ and $\alpha_t$ need to be carefully chosen to balance these two desirable but contradictory properties. 
We do not choose the values of $\alpha$ until Section~\ref{se:regret} for specific kernels in order to  minimize the rates of cumulative regret bounds. The value of $\alpha$ is dependent on the kernel and its $\gamma_t$. 

In BO literature, the probability of an event or inequality dependent on $t$ is often chosen to be $1-\frac{1}{t^2}$ so that the summation of the probabilities of these events being \textbf{not} true, \textit{i.e.}, $\sum_{t=1}^T \frac{1}{t^2}$, is finite~\citep{chowdhury2017kernelized}.
However, this is in fact not necessary for an algorithm to be no-regret, as the sum is only required to be sublinear $o(T)$. 
Here, given $E^r(t)$, we choose the probability of~\eqref{eqn:EI-sigma-star-1} \textbf{not} holding to be $ \frac{1}{t^{\frac{\alpha}{2}}}$, whose sum from $1$ to $T$ is $\mathcal{O}(T^{1-\frac{\alpha}{2}})$, to control the total expectation of $r_t$ when~\eqref{eqn:EI-sigma-star-1} does not hold. 
In other words, Theorem~\ref{theorem:EI-sigma-star} does \textbf{not} hold with a small probability, but not too small ($\frac{1}{t^2}$ would be too small) that would require the parameters $c_1(t)$ and $c_2(t)$ to be too large to generate sublinear cumulative regret bounds. 
For instance, by introducing and controlling $\alpha_t$, which is different and smaller than $\beta_t$, we allow $c_2(t)$ to be $\mathcal{O}(t^{\frac{\alpha}{2}}\log^{1/2}(t))$.

Finally, we point out that the usual $\frac{1}{t^2}$ rate is equivalent to $\alpha=4$ (which makes $c_2(t)$ too large) while we let $\alpha\in (0,1]$.
\end{rem}

\section{Cumulative regret bounds for GP-EI}\label{se:regret}
With the transformed instantaneous regret upper bound obtained in Section~\ref{se:instant-reg}, we can now bound the cumulative regret for GP-EI in this section. %we present the cumulative regret bounds for the GP-EI algorithms, following a similar analysis framework as GP-TS~\citep{chowdhury2017kernelized}. 
We first state the probability of $E^r(t)$ for each incumbent with the choice of $\beta_t$ in~\eqref{def:beta} below.
\begin{lem}\label{lem:eftprob}
     Given parameters $\beta_t$, $E^f(t)$ is true for all $t\in\Nbb$ with probability  $\geq 1-\delta/8$ and $E^s(t)$ is true with probability $\geq 1-\delta/8$ for each sequence of $\ubm_t$.
     Therefore, for BPMI, $E^r(t)$ is true with probability $\geq 1-\delta/4$.
     For BSPMI, $E^r(t)$ is true with probability $\geq 1-3\delta/8$.
     For BOI, $E^r(t)$ is true with probability $\geq 1-3\delta/8$.
\end{lem}

\subsection{Cumulative regret bounds for BPMI} \label{se:postmeanreg-reg}
We establish the cumulative regret bound for BPMI in this section.
We note that the exploitation satisfies $\mu^+_t-\mu_{t-1}(\xbm_t)\leq 0$ in Theorem~\ref{theorem:EI-sigma-star}. Thus, the instantaneous regret bound can be further simplified (as stated in Corollary~\ref{cor:EI-sigma-star-pm}) to lead to the cumulative regret bound in the following lemma.
\begin{lem}\label{lemma:ei-pm-regret-1}
    Given $\delta\in(0,1)$, GP-EI with BPMI has with probability $\geq 1-\delta$ that
     \begin{equation} \label{eqn:ei-pm-regret-1}
  \centering
  \begin{aligned}
     R_T  \leq&  (\eta_T^{1/2}+4\beta_T^{1/2})\sum_{t=1}^T \sigma_{t-1}(\xbm_t) + \frac{\pi^2}{6}  + 2B T^{1-\frac{\alpha}{2}} + c_{B}(\delta) \sqrt{T}\eta_T^{1/2},\\
    \end{aligned}
  \end{equation}
   where $c_{B}(\delta) = 3(B+2)\sqrt{2\log(\frac{4}{3\delta})}$.
\end{lem}
 
The rate of cumulative regret bound for GP-EI with BPMI is given below.
\begin{thm}\label{thm:pm-cumu-regret}
    The GP-EI algorithm with BPMI leads to the cumulative regret upper bound  
    \begin{equation} \label{eqn:ei-pm-cumu-1}
  \centering
  \begin{aligned}
       R_T = \mathcal{O}( \max\{T^{\frac{1}{2}+\frac{\alpha}{2}} \log^{\frac{1}{2}}(T)\sqrt{\gamma_T},T^{1-\frac{\alpha}{2}} \}),
    \end{aligned}
  \end{equation}
  with probability $\geq 1-\delta$. 
   For SE kernels, choose $\alpha=\frac{1}{2}$ for the smallest $R_T$. Then, $R_T=\mathcal{O}(T^{\frac{3}{4}} \log^{\frac{d+2}{2}}(T))$.
   For Mat\'{e}rn kernels with $\nu>\frac{1}{2}$,
  choose $\alpha=\frac{\nu}{2\nu+d}$ for smallest $R_T$. Then, $R_T=\mathcal{O}(T^{\frac{3\nu+2d}{4\nu+2d}}\log^{\frac{4\nu+d}{4\nu+2d}}(T))$.
\end{thm}
\begin{rem}[Choice of $\alpha$.]\label{remark:bpmi-regret}
   The GP-EI algorithm with BPMI is a no-regret algorithm for SE and Matérn kernels. 
   %In particular, thanks to the parameter $\alpha$, for Matérn kernel, we have shown that the cumulative regret bound is sublinear whatever the kernel parameter $\nu$ is. 
   For Mat\'{e}rn kernels, Theorem~\ref{thm:pm-cumu-regret} also shows that the cumulative regret bound is sublinear for any $\nu$ due to the Bayesian setting and flexibility given by  $\alpha$. We choose $\alpha$ such that the orders for both operands in the $\max$ operator in~\eqref{eqn:ei-pm-cumu-1} are the same and to minimize $\mathcal{O}(R_T)$. In this sense, $\alpha$ has an optimal value for each kernel; however, we emphasize again that $\alpha$ is an analytic tool and is not part of the GP-EI algorithm.
   %For both kernels, we choose $\alpha$ such that the orders in~\eqref{eqn:ei-pm-cumu-1} are the same for both terms to minimize $\mathcal{O}(R_T)$. In this sense, $\alpha$ has an optimal value for each kernel. We emphasize that $\alpha$ is an analytic tool and has no actual effect on $R_T$.
\end{rem}
\begin{rem}[Regret bound comparisons.]
Compared to the cumulative regret bound of GP-UCB and GP-TS using the same information gain $\gamma_T$ in the Bayesian setting~\citep{vakili2021information}, the regret bounds for GP-EI with BPMI are larger for both SE and Mat\'{e}rn kernels. 
On the other hand, it is well-known that GP-EI can be easily modified to achieve near optimal convergence rate in terms of simple regret bound~\citep{bull2011convergence,wang2025convergence}.  Since there is no algorithmic parameter in GP-EI, the regret bounds are maintained due to the intrinsic properties of EI. 
To obtain the same rate for cumulative regret bound as GP-UCB, one might simply add a parameterized and increasing exploration term, \textit{e.g.}, $c_{\alpha}\beta_t^{1/2}\sigma_{t-1}(\xbm)$, to the EI function, similar to the $\beta_t^{1/2}\sigma_{t-1}(\xbm)$ term in GP-UCB.
\end{rem}

%% file: Sections/SPMregret.tex
\subsection{Cumulative regret bounds for BSPMI}\label{se:bspm-reg}
%In this section, we present the cumulative regret bound for BSPMI. 
We next  establish the cumulative regret upper bound for BSPMI.
\begin{lem}\label{lemma:bspm-inst-regret}
    Given $\delta\in(0,1)$, GP-EI with BSPMI has with probability $\geq 1-\delta$ that
     \begin{equation} \label{eqn:bspm-inst-regret-1}
  \centering
  \begin{aligned}
     R_T  \leq& 5\eta_T^{1/2}(B+\beta_T^{1/2})+ 10\beta_T^{1/2}\eta_T^{1/2}\sum_{t=1}^T\sigma_{t-1}(\xbm_t)  + \frac{\pi^2}{6}  + 2B T^{1-\frac{\alpha}{2}} 
       + c_B(\delta)\eta_T^{1/2} \beta_T^{1/2}\sqrt{T},\\
    \end{aligned}
  \end{equation}
   where $c_B(\delta)= \sqrt{2\log(\frac{8}{5\delta})}(5B+8)c_B/\phi(0)$ and $c_B=5B+8$. 
\end{lem}

\begin{rem}[Bounding the exploitation term.]\label{remark:bspm-regret-1}
   Unlike BPMI, here, the summation of the exploitation term $\max\{\mu_{t-1}^m-\mu_{t-1}(\xbm_t),0\}$ needs to be bounded. For this, we construct a subsequence $\xbm_{t_i}, i=1,2,\dots$ of $\xbm_t$ that contains all the indices $t$ where $\mu_{t-1}^m-\mu_{t-1}(\xbm_t)>0$. 
   Then, using the inequalities from $E^r(t)$, we can write $\mu_{t_i-1}^m-\mu_{t_i-1}(\xbm_{t_i}) \leq f(\xbm_{t_i-1}) -f(\xbm_{t_i})+\beta_{t_i}^{1/2}\sigma_{t_i-1}(\xbm_{t_i-1}) + \beta_{t_i}^{1/2}\sigma_{t_i-1}(\xbm_{t_i})$. 
   Using this technique and summing up all $t_i$ lead to the bound on  $\sum_{t=1}^T \max\{\mu_{t_i-1}^m-\mu_{t_i-1}(\xbm_{t_i}),0\}$, which consists of $2B$ and  $\sum_{t=1}^T\beta_{t_i}^{1/2}\sigma_{t_i-1}(\xbm_{t_i})$ (see~\eqref{eqn:bspm-inst-regret-pf-5} for details). 
   Thus, the sum of the exploitation terms $\sum_{t=1}^T \max\{\mu_{t-1}^m-\mu_{t-1}(\xbm_{t}),0\}$ is sublinear. 
\end{rem}
 Compared to BPMI, the coefficient $10\beta_T^{1/2}\eta_T^{1/2}$ is larger than the counterpart $\eta_T^{1/2}$ in Lemma~\ref{lemma:ei-pm-regret-1}, which leads to a larger cumulative regret bound for BSMPI. This is consistent with its larger instantaneous regret bound.
 The rate of cumulative regret bound for GP-EI with BSPMI is given below.
\begin{thm}\label{thm:bspmi-cumu-regret}
    The GP-EI algorithm with BSPMI leads to the cumulative regret bound  
    \begin{equation} \label{eqn:ei-bspm-cumu-1}
  \centering
  \begin{aligned}
       R_T = \mathcal{O}( \max\{T^{\frac{1}{2}+\frac{\alpha}{2}} \log(T)\sqrt{\gamma_T},T^{1-\frac{\alpha}{2}} \}),
    \end{aligned}
  \end{equation}
  with probability $\geq 1-\delta$.
   For SE kernels, we choose $\alpha=\frac{1}{2}$. Then, the best rate is $R_T=\mathcal{O}(T^{\frac{3}{4}} \log^{\frac{d+3}{2}}(T))$.
   For Mat\'{e}rn kernels with $\nu>\frac{1}{2}$, we 
  choose $\alpha=\frac{\nu}{2\nu+d}$ to get the best rate $R_T=\mathcal{O}(T^{\frac{3\nu+2d}{4\nu+2d}}\log^{\frac{3\nu+d}{2\nu+d}}(T))$.
\end{thm}
\begin{rem}[Comparison to BPMI.]\label{remark:bspmi-regret}
   The GP-EI algorithm with BSPMI is also a no-regret algorithm. 
   The cumulative regret bound is $\log^{1/2}(T)$ larger than BPMI, intuitively because BSPMI is the minimum of the posterior mean only over the previous samples, while BPMI is the minimum of posterior mean over the entire domain $C$.
   %In~\cite{scott2011correlated}, it is often recommended that the algorithm be simplified from using best posterior mean to the best sampled posterior mean. Our results for GP-EI here shows that there could be theoretical support behind such choices.
   
\end{rem}

%% file: Sections/BOregret.tex
\subsection{Cumulative regret bounds for BOI}\label{se:bestobs-reg}
The regret bound of this section is different from before, namely, we show that BOI generates either a sublinear cumulative regret bound or a fast reducing noisy simple regret bound ($\mathcal{O}(\log^{1/2}(T) T^{-\frac{1}{2}}\sqrt{\gamma_T})$).
After $T$ iterations, let $n_y(T)$ denote how many times $E^y(t)$ is \textit{not} true for $t=\{1,2,\ldots, T\}$. The following lemma provides an upper bound for $R_T$.
\begin{lem}\label{lemma:boi-inst-regret}
      Given $\delta\in(0,1)$, GP-EI with BOI has with probability $\geq 1-\delta$ that
     \begin{equation} \label{eqn:boi-inst-regret-1}
  \centering
  \begin{aligned}
     R_T \leq&
      c_{B}(\delta) \eta_T^{1/2} \sqrt{T}\sigma+ c_3\eta_T^{1/2} c_{\beta}(T)\sum_{t=1}^T \sigma_{t-1}(\xbm_t) + \frac{\pi^2}{6}  + 2B T^{1-\frac{\alpha}{2}} \\
       &+ \eta_T^{1/2}c_{\beta}(T)c_{B\sigma}(\delta)\sqrt{T} +2B n_y(T),\\
    \end{aligned}
  \end{equation}
   where $c_{B}(\delta)=2.5(2B+3)\sqrt{2\log(8/\delta)}$, $c_3=10$, $c_{B\sigma}(\delta)=(12.5B+20+2.5\sigma)\sqrt{2\log(4/\delta)}$, and $c_{\beta}(t)=\max\{\beta_t^{1/2},c_y(t)\}$. 
\end{lem}
The order of each term is $\mathcal{O}(T^{\frac{\alpha}{2}+\frac{1}{2}}\log^{1/2}(T))$, $\mathcal{O}(T^{\frac{\alpha}{2}+\frac{1}{2}}\log(T)\sqrt{\gamma_t})$, $T^{1-\frac{\alpha}{2}}$, $\mathcal{O}(T^{\frac{\alpha}{2}+\frac{1}{2}}\log(T))$, and $n_y(T)$ in~\eqref{eqn:boi-inst-regret-1}. 
The BOI cumulative regret bound is characterized next. We note that the two scenarios below are exclusive, as the opposite of $n_y(T)=\mathcal{O}(T)$ is what scenario 2 describes. 
\begin{thm}\label{thm:boi-cumulative-regret}
The GP-EI algorithm with BOI leads to the cumulative regret bound
\begin{equation} \label{eqn:ts-regret-1}
\centering
\begin{aligned}
     R_T = \mathcal{O}( \max\{T^{\frac{1}{2}+\frac{\alpha}{2}} \log(T)\sqrt{\gamma_T},T^{1-\frac{\alpha}{2}}, n_y(T) \}),
\end{aligned}
\end{equation}
  with probability $\geq 1-\delta$. 
  
Furthermore, one of the following two (exclusive) scenarios occur:\\
  (1) If $n_y(T)=o(T)$, then GP-EI with BOI is a no-regret algorithm for SE and  Mat\'{e}rn kernels.  
   For SE kernel, we
  choose $\alpha=\frac{1}{2}$ to obtain $R_T=\mathcal{O}(\max\{T^{\frac{3}{4}} \log^{\frac{d+3}{2}}(T),n_y(T)\})$.
   For Mat\'{e}rn kernels with $\nu>\frac{1}{2}$, we
  choose $\alpha=\frac{\nu}{2\nu+d}$ to obtain $R_T=\mathcal{O}(\max\{T^{\frac{3\nu+2d}{4\nu+2d}}\log^{\frac{3\nu+d}{2\nu+d}}(T)),n_y(T)\}$.\\
  (2) Otherwise, there exists a constant $c_n>0$ such that $n_y(T)\geq c_n T$. In this case, the upper bound for noisy simple regret $r_t^s$ reduces at rate $\mathcal{O}(\log^{1/2}(T) T^{-\frac{1}{2}}\sqrt{\gamma_T})$.
\end{thm}

\begin{rem}\label{remark:boi-regret}
   Unlike GP-EI with BPMI and BSPMI, GP-EI with BOI is not guaranteed to be no-regret because of the presence of $n_y(T)$ in the cumulative upper bound~\eqref{eqn:ts-regret-1}.
   %as far as our analysis is concerned. 
   As we discussed in Remark~\ref{remark:boi-inst-regret}, if at some $t\in\Nbb$ we have $y^+_{n-1}-f(\xbm^*)\leq 0$, then $E^y(t)$ is false for all $t\geq n$. 
   This means that $n_y(T)$ is not sublinear in $T\geq t$. 
   Consequently, the cumulative regret bound is not necessarily sublinear. However, this case is the second scenario of Theorem~\ref{thm:boi-cumulative-regret}, for which GP-EI with BOI converges in the sense that the noisy simple regret is reduced quickly.
   
%We also point out that for an optimizer that prefers to find one optimal solution quickly, GP-EI with BOI could be a judicious and effective choice since its simple regret upper bound is converging fast. \coscom{This argument is a bit loose/weak: assuming an optimizer prefers to find one optimal solution quickly, it looks like he/you can't tell which scenario will occur' so the convergence may not be fast in the first scenario. Or is it fast in that scenario as well for simple regret? It is. But need to think about framing.}
\end{rem}

\subsection{Implications of our results}\label{se:rkhs}
In the context of bandit optimization, where the objective is to minimize cumulative regret across all iterations, we establish for the first time that both BPMI and BSPMI achieve desirable sublinear cumulative regret bounds for the noisy GP-EI algorithm. In practice, however, BPMI tends to be more computationally demanding, as it requires computing the (absolute) maximum of the posterior mean over the entire domain $C$. This introduces an additional global optimization subproblem that must be solved prior to maximizing the acquisition function, which can become particularly costly in high-dimensional design spaces. In contrast, although BSPMI incurs a slightly looser cumulative regret upper bound—on the order of  ($\log^{1/2}(T)$)—it only involves evaluating the posterior mean at previously sampled points, making it significantly more efficient from a computational standpoint.

On the other hand, GP-EI with BOI remains one of the most widely adopted acquisition strategies in practice (Frazier, 2018). From a theoretical perspective, however, it does not inherently guarantee sublinear cumulative regret. Furthermore, under high observation noise (i.e., large $\sigma$), the performance of GP-EI with BOI is more susceptible to the effects of individual noisy evaluations, especially when compared to BPMI and BSPMI (see Remark~\ref{remark:boi-inst-regret-example}). Nevertheless, when the primary goal is to identify a single best solution rather than minimizing cumulative regret, the rapid reduction in noisy simple regret offered by BOI may still be satisfactory in practice—particularly when the noise level is relatively low.

Ultimately, the choice of incumbent selection strategy should be guided by the user's priorities, such as the whether a no-regret property is important, the level of observation noise, computational constraints, and ease of implementation. 
In a common scenario where achieving sublinear cumulative regret is desirable and the computational cost of identifying the global maximum of the posterior mean is prohibitive, we recommend employing GP-EI with BSPMI. This recommendation is theoretically justified by this work and aligns well with the empirical benchmark results in~\cite{picheny2013benchmark}.

Our theory is presented under the Bayesian setting, \textit{i.e.}, by assuming that the objective function $f(x)$ is a sample from a GP. Extending the results to the frequentist setting is challenging, mainly due to the lack of tight confidence interval on $|f(\xbm)-\mu_{t-1}(\xbm)|$ at given $t$. This problem is recognized in~\cite{vakili2021open}. For instance, the probability $1-(\frac{1}{2t})^{\frac{\alpha}{2}}$ at $t$ given $\mathcal{F}'_{t-1}$ in Theorem~\ref{theorem:EI-sigma-star} can no longer be easily obtained. In the frequentist setting, we have the uniform bound $|f(\xbm)-\mu_{t-1}(\xbm)|\leq \beta_t^{1/2}\sigma_{t-1}(\xbm)$ for all $\xbm\in C$ and $t\in\Nbb$ and the state-of-the-art $\beta_t$ is of the order $\gamma_t$~\citep{chowdhury2017kernelized}. Therefore, the corresponding $c_1(t)$ and $c_2(t)$ would be much larger and cannot produce sublinear cumulative regret bound. 
We point out that in literature~\citep{srinivas2009gaussian,vakili2021information}, Bayesian setting produces better cumulative regret bounds than the frequentist setting, and hence this challenge is expected~\citep{vakili2021open}.
One potential approach to resolve this is to develop the instantaneous regret in tandem with the confidence interval at a given $t$. The goal is to potentially have a smaller $\beta_t$ with a slightly worse probability, given $\mathcal{F}_{t-1}'$. 

%This is a topic for future research.

%% file: Sections/Conclusions.tex
\section{Numerical experiments}\label{se:exp}

In this section, we evaluate the GP-EI algorithms with three different incumbents to validate our theoretical results. 

We consider six commonly used synthetic test functions: Branin 2D, Schwefel 2D, Styblinski-Tang 2D,  Camel 2D, Rosenbrock 4D and Hartmann 6D. The mathematical expressions of these functions are provided in Table~\ref{tab:funcs}:
\begin{itemize}
\item The \textbf{Branin 2D} function has a smooth, regular landscape with three well-separated global minima. Its simplicity makes it suitable for evaluating global search capabilities in a controlled setting. (\textbf{Simple})
\item The \textbf{Styblinski–Tang 2D} function is multimodal with a relatively smooth, separable landscape. It contains multiple local minima and presents a moderate level of difficulty for optimization algorithms. (\textbf{Moderate}) 
\item  The \textbf{Six-Hump Camel 2D} function features a symmetric landscape with two global minima, multiple local optima, and saddle points. This makes it a moderately challenging test for distinguishing between closely spaced optima. (\textbf{Moderate}) 
\item  The \textbf{Schwefel 2D} function presents a highly deceptive and rugged landscape with many local minima and a global minimum far from the origin, making it a difficult benchmark for most optimization methods. (\textbf{Complex})
\item  The \textbf{Rosenbrock 4D} function, although unimodal, contains a narrow, curved valley leading to the global minimum. Its ill-conditioned structure poses a significant challenge, especially for gradient-based methods in higher dimensions. (\textbf{Complex})
\item The \textbf{Hartmann 6D} function is a high-dimensional, multimodal benchmark with steep, narrow basins of attraction. Its deceptive structure makes it a difficult test case for global optimization strategies. (\textbf{Complex})
\end{itemize}
 We consider three noise levels: $\sigma = 0.001$, $\sigma = 0.01$, and $\sigma = 0.1$. In all experiments, we use a Gaussian process prior with zero mean ($\mu \equiv 0$) and a Matérn ($\nu =1.5$)  kernel (Similar results are observed when using a squared exponential (SE) kernel, as reported in the Appendix~\ref{appdx:experiments}). Kernel parameters are estimated via maximum likelihood.
The total evaluation budget for each function is $N = 500 + n_0$, where the number of initial design points is set to $n_0 = 10d$, with $d$ denoting the dimension of the function. For each test function, we conduct $R = 100$ independent trials and report the mean cumulative regret. The results are presented in Figures~\ref{fig:results1} and Figures~\ref{fig:results2}.
The solid lines represent the average cumulative regret/ the number of iteration ($R_t/t$) across trials, and the shaded regions indicate the 95\% confidence intervals.

\begin{table}[!b]
\centering
\resizebox{\textwidth}{!}{  
\begin{tabular}{llc}
	\hline
	\hline
	Functions &  $d$ & Equation  \\
	\hline
        Branin 2D &2&
	$\begin{gathered} 
		f(\mathbf{x})= \frac{1}{51.95}\left(( x_2 - \frac{5.1}{4\pi^2}x_1^2 + \frac{5}{\pi}x_1 - 6 )^2 + 10(1 - \frac{1}{8\pi})\cos(x_1) - 44.81\right) \\
		 -5\leq x_1 \leq 10, \quad 0\leq x_2 \leq 15 \\
		\mathbf{x}^* = (-\pi,12.275), (\pi, 2.275), (9.42478, 2.475), f^* = -1.05
	\end{gathered}$\\
	\hline
    Styblinski-Tang 2D &2& 
	$\begin{gathered} 
		f(\mathbf{x}) = \frac{1}{45.17}\left(\frac{1}{2}\sum_{i=1}^{2} \left( x_i^{4} - 16x_i^{2} + 5x_i \right) + 8.72\right)\\
		x_i \in [-5, 5], i= 1, 2\\
		\mathbf{x}^* = (-2.9034,-2.9035), f^* = -1.54
	\end{gathered}$\\
	\hline
    Camel 2D&2& 
	$\begin{gathered} 
f(\mathbf{x}) = \frac{1}{26.28}\left(
\bigl(4 - 2.1\,x_{1}^{2} + \tfrac{x_{1}^{4}}{3}\bigr)\,x_{1}^{2}
+ x_{1}x_{2}
+ \bigl(-4 + 4x_{2}^{2}\bigr)\,x_{2}^{2}  -20.12 \right)\\
		x_1 \in [-3, 3], x_1 \in [-2, 2] \\
		\mathbf{x}^* = (0.0898, -0.7126),  (-0.0898, 0.7126) , f^* = -0.8049
	\end{gathered}$\\
	\hline
	Schwefel 2D &2 &  
	$\begin{gathered} 
		f(\mathbf{x})= \frac{1}{274.3}\left( 418.9829 *2 -\sum_{i=1}^{2} w_{i} \sin \left(\sqrt{\left|w_{i}\right|}\right) - 838.57\right)\\  
		w_{i}=500x_i, i = 1, 2\\
		x_i \in [-1, 1],  i= 1, 2\\
		\mathbf{x}^* = (0.8419,0.8419), f^* = -3.057
	\end{gathered}$\\
	\hline
	Rosenbrock 4D&4& 
	$\begin{gathered} 
    f(\mathbf{x}) = \frac{1}{372997} \left(\sum_{i=1}^{3} \left(\, 100\bigl(x_{i+1} - x_{i}^{2}\bigr)^{2} \;+\; \bigl(x_{i} - 1\bigr)^{2} \right) - 383434 \right)
    \\
		x_i \in [-5, 10], i= 1, 2,3,4\\
		\mathbf{x}^* = (1,1,1,1), f^* = -1.0280
	\end{gathered}$\\
	\hline
	Hartmann 6D&6& 
	$\begin{gathered}
		f(\mathbf{x})=\frac{1}{0.38}\left(-\sum_{i=1}^{4} \alpha_{i} \exp \left(-\sum_{j=1}^{6} A_{i j}\left(x_{j}-P_{i j}\right)^{2}\right) + 0.26\right)\\
		\alpha=(1.0,1.2,3.0,3.2)^{T}\\
		\mathbf{A}=\left(\begin{array}{cccccc}
			10 & 3 & 17 & 3.50 & 1.7 & 8 \\
			0.05 & 10 & 17 & 0.1 & 8 & 14 \\
			3 & 3.5 & 1.7 & 10 & 17 & 8 \\
			17 & 8 & 0.05 & 10 & 0.1 & 14
		\end{array}\right)\\
		\mathbf{P}=10^{-4}\left(\begin{array}{cccccc}
			1312 & 1696 & 5569 & 124 & 8283 & 5886 \\
			2329 & 4135 & 8307 & 3736 & 1004 & 9991 \\
			2348 & 1451 & 3522 & 2883 & 3047 & 6650 \\
			4047 & 8828 & 8732 & 5743 & 1091 & 381
		\end{array}\right)\\
		x_i \in [0, 1], i = 1, \cdots, 6\\
		\mathbf{x}^* = (0.20169,0.150011,0.476874,0.275332,0.311652,0.6573), f^* = -8.059
	\end{gathered}$ \\
	\hline
	\hline
\end{tabular}%
}
\caption{\label{tab:funcs} List of test functions. All test functions are standardized such that their output values have zero mean and unit standard deviation across the design space.}
\end{table}%

\begin{figure}[!t]
\centering
\begin{subfigure}{0.33\linewidth}
\includegraphics[width=\linewidth]{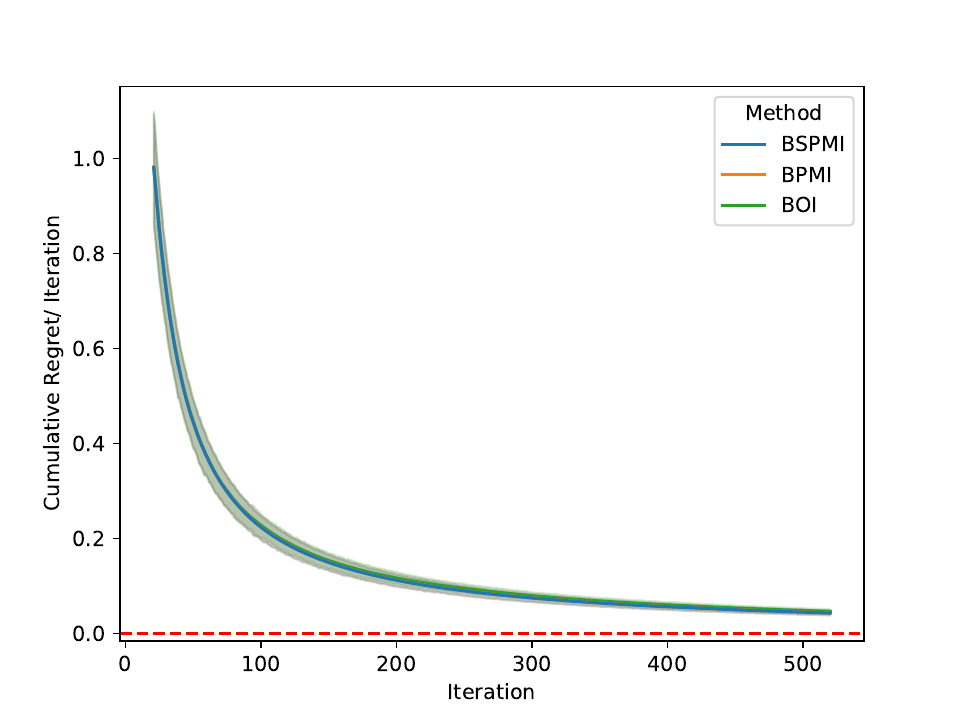} 
\caption{Branin 2D, $\sigma = 0.001$}
\label{fig:1a}
\end{subfigure}\hfill
\begin{subfigure}{0.33\linewidth}
\includegraphics[width=\linewidth]{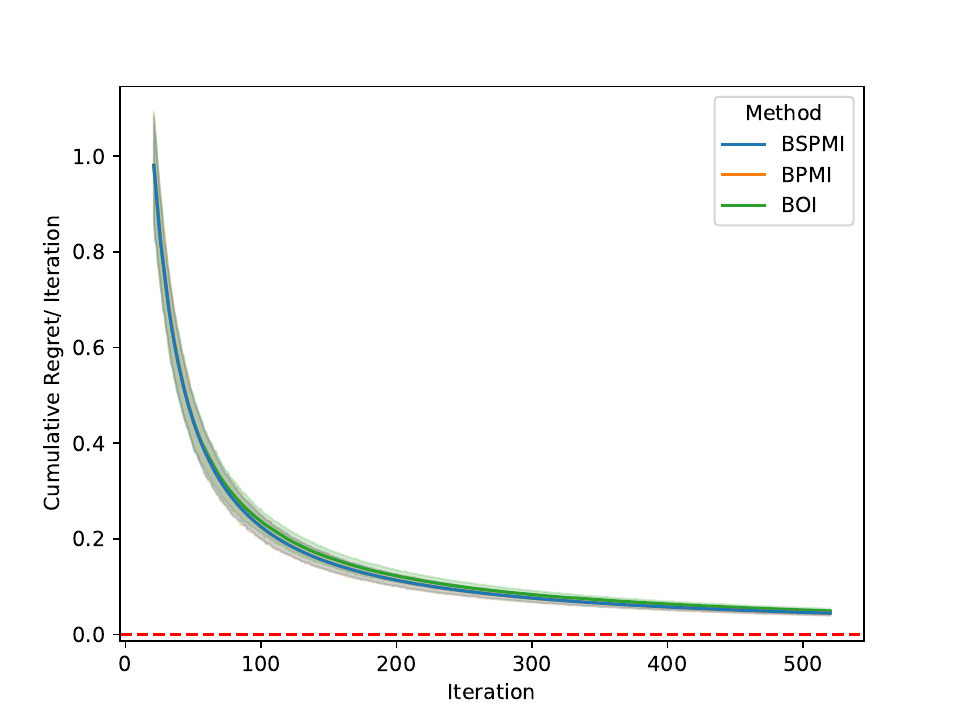}
\caption{Branin 2D, $\sigma = 0.01$}
\label{fig:1b}
\end{subfigure}\hfill
\begin{subfigure}{0.33\linewidth}
\includegraphics[width=\linewidth]{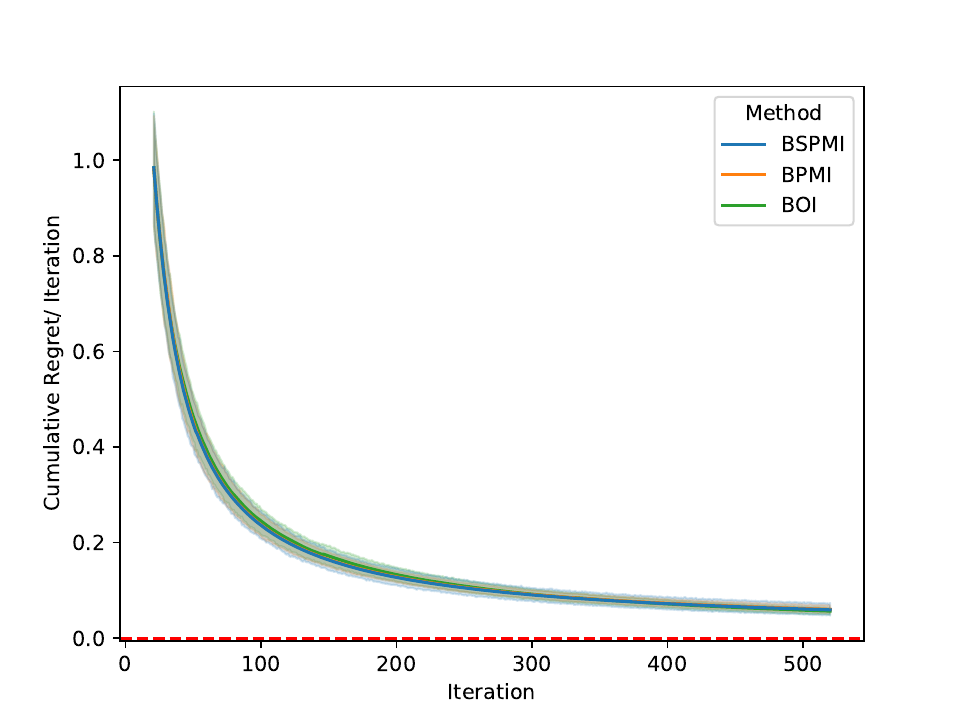}
\caption{Branin 2D, $\sigma = 0.1$}
\label{fig:1c}
\end{subfigure}

\begin{subfigure}{0.33\linewidth}
\includegraphics[width=\linewidth]{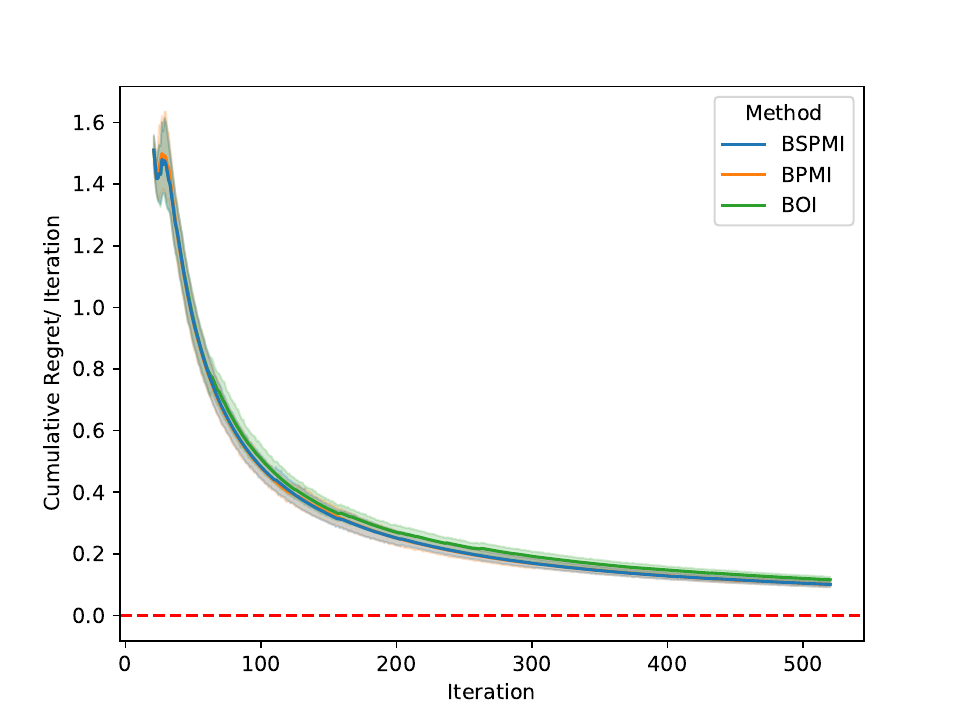} 
\caption{Styblinski-Tang 2D, $\sigma = 0.001$}
\label{fig:2a}
\end{subfigure}\hfill
\begin{subfigure}{0.33\linewidth}
\includegraphics[width=\linewidth]{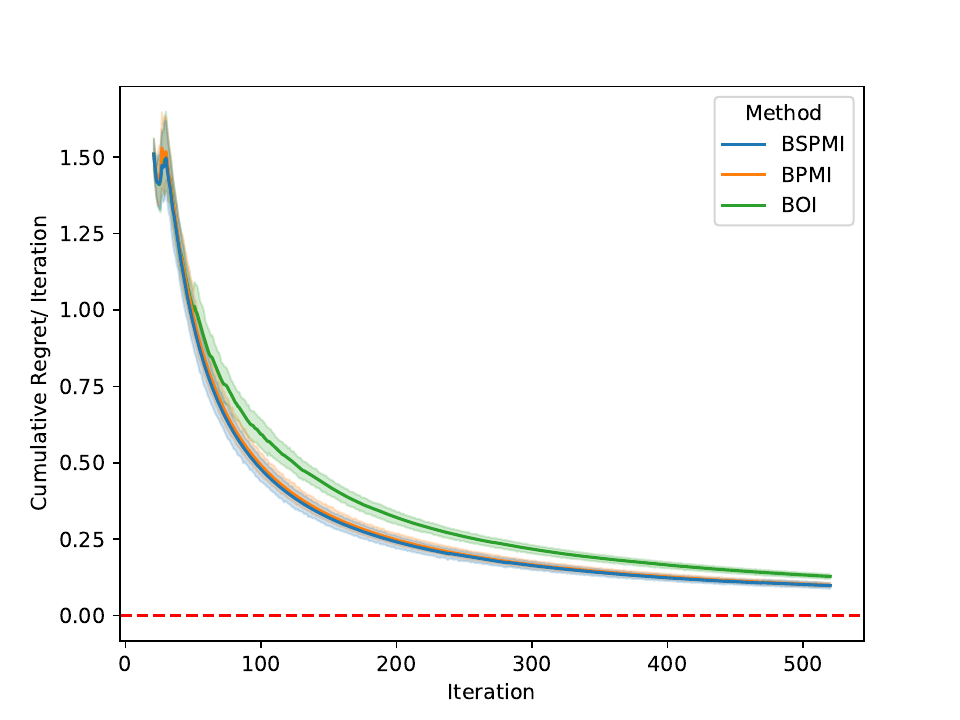}
\caption{Styblinski-Tang 2D, $\sigma = 0.01$}
\label{fig:2b}
\end{subfigure}
\begin{subfigure}{0.33\linewidth}
\includegraphics[width=\linewidth]{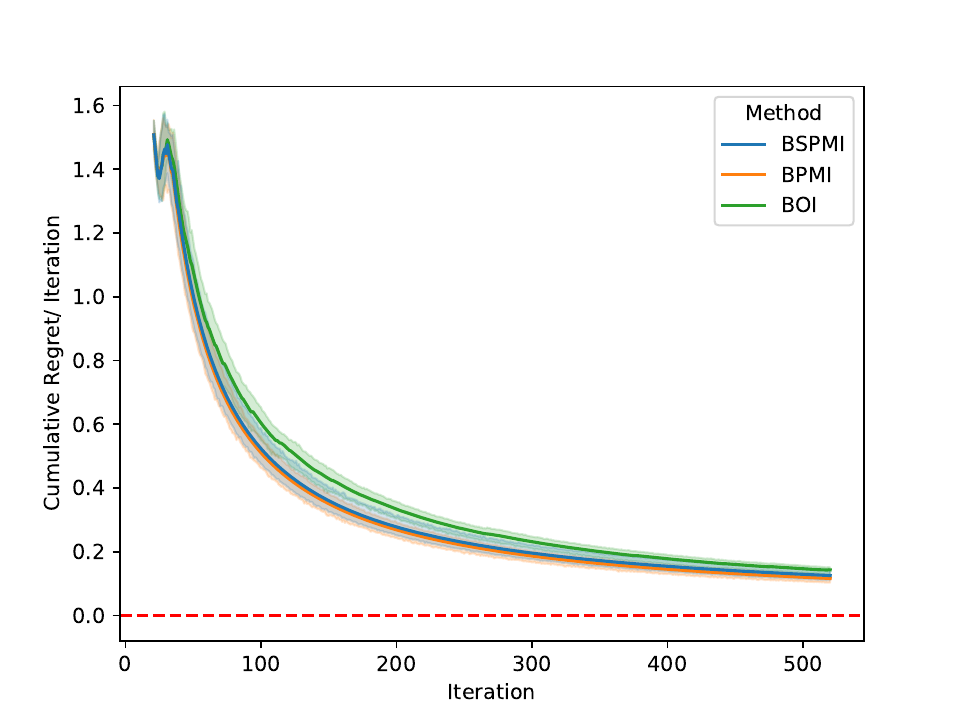}
\caption{Styblinski-Tang 2D, $\sigma = 0.1$}
\label{fig:2c}
\end{subfigure}

\begin{subfigure}{0.33\linewidth}
\includegraphics[width=\linewidth]{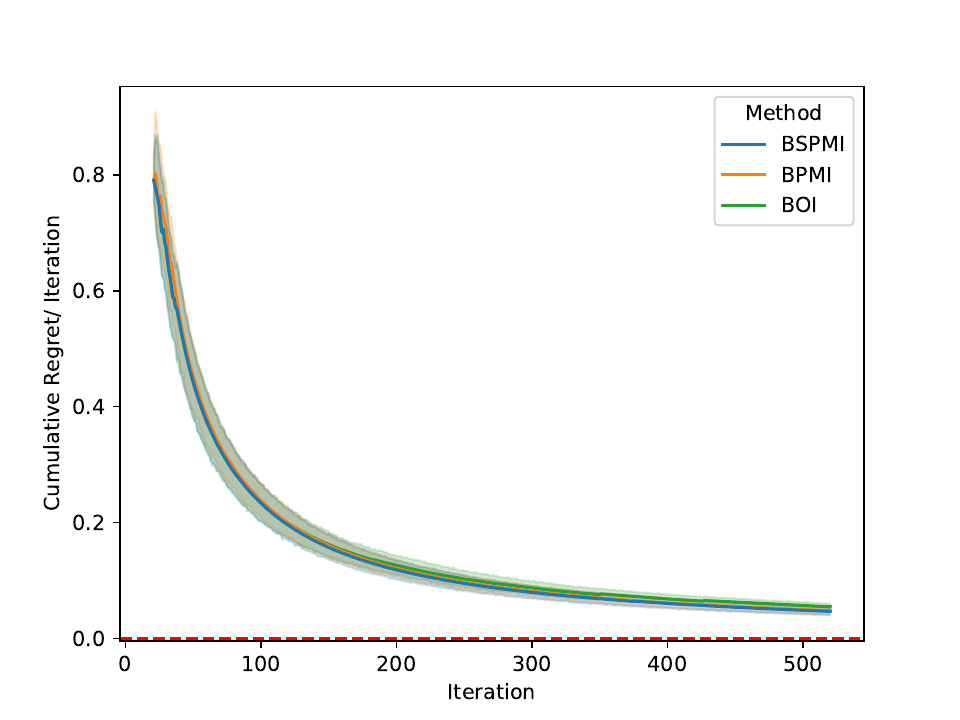} 
\caption{Camel 2D, $\sigma = 0.001$}
\label{fig:2a}
\end{subfigure}\hfill
\begin{subfigure}{0.33\linewidth}
\includegraphics[width=\linewidth]{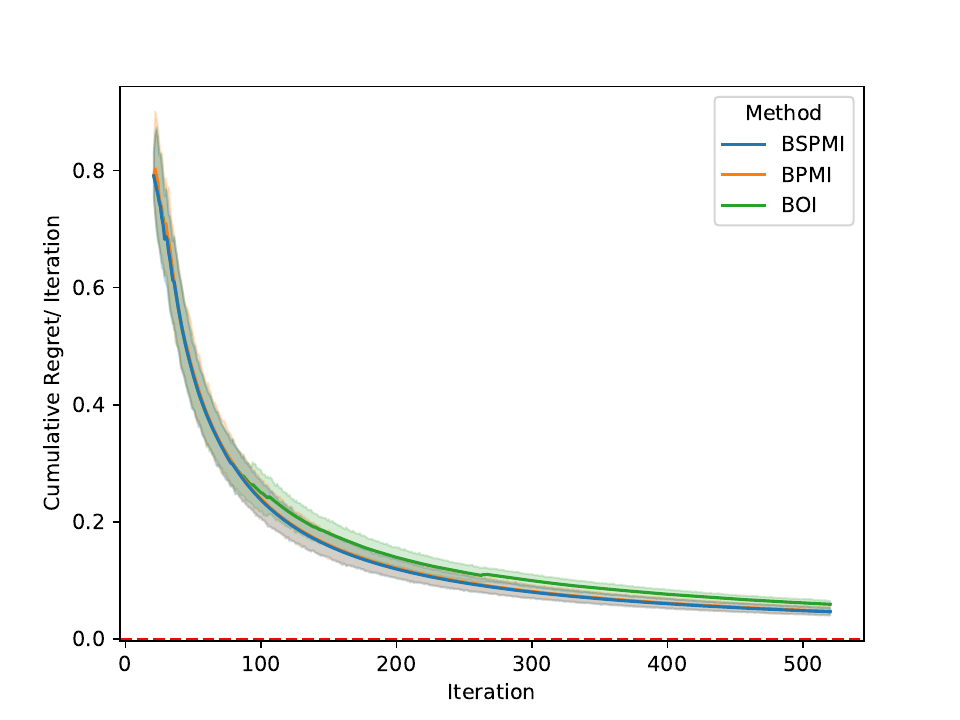}
\caption{Camel 2D, $\sigma = 0.01$}
\label{fig:2b}
\end{subfigure}
\begin{subfigure}{0.33\linewidth}
\includegraphics[width=\linewidth]{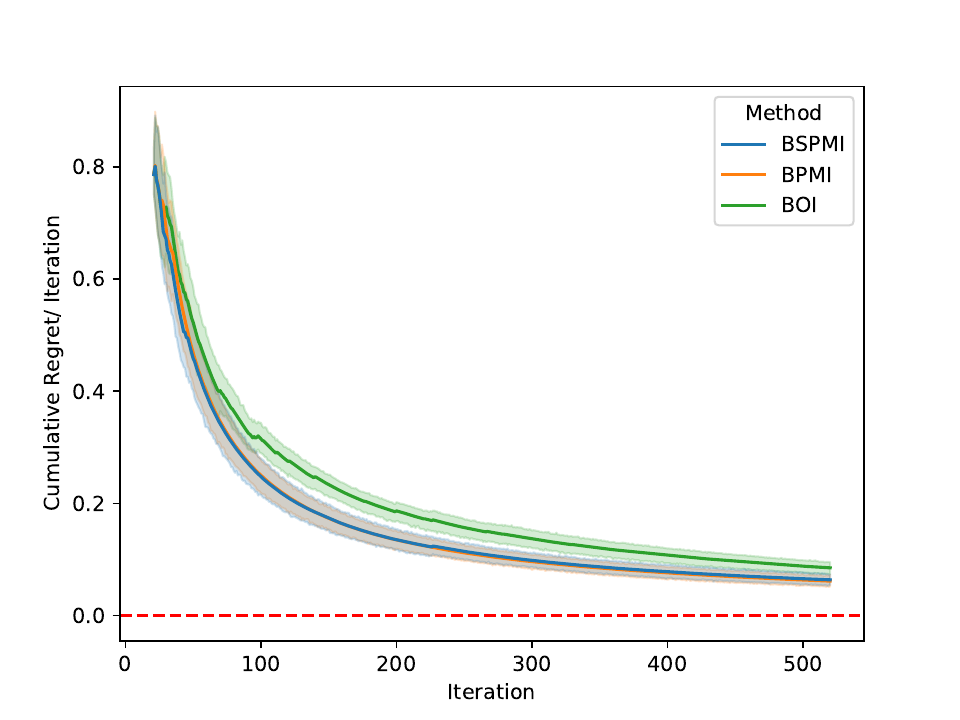}
\caption{Camel 2D, $\sigma = 0.1$}
\label{fig:2c}
\end{subfigure}

\begin{subfigure}{0.33\linewidth}
\includegraphics[width=\linewidth]{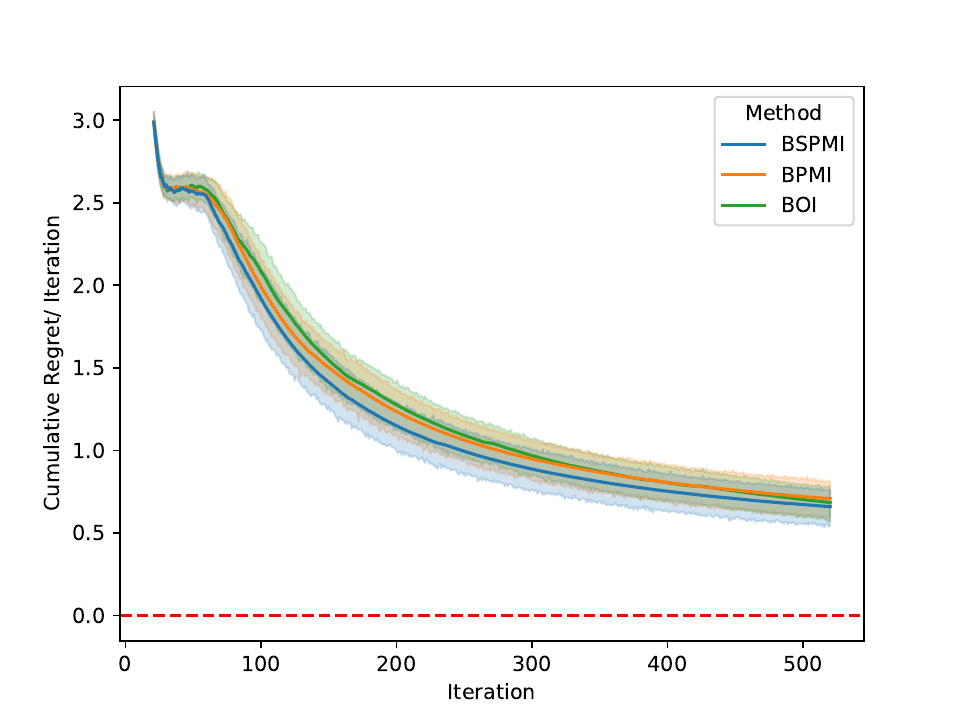} 
\caption{Schwefel 2D, $\sigma = 0.001$}
\label{fig:2a}
\end{subfigure}\hfill
\begin{subfigure}{0.33\linewidth}
\includegraphics[width=\linewidth]{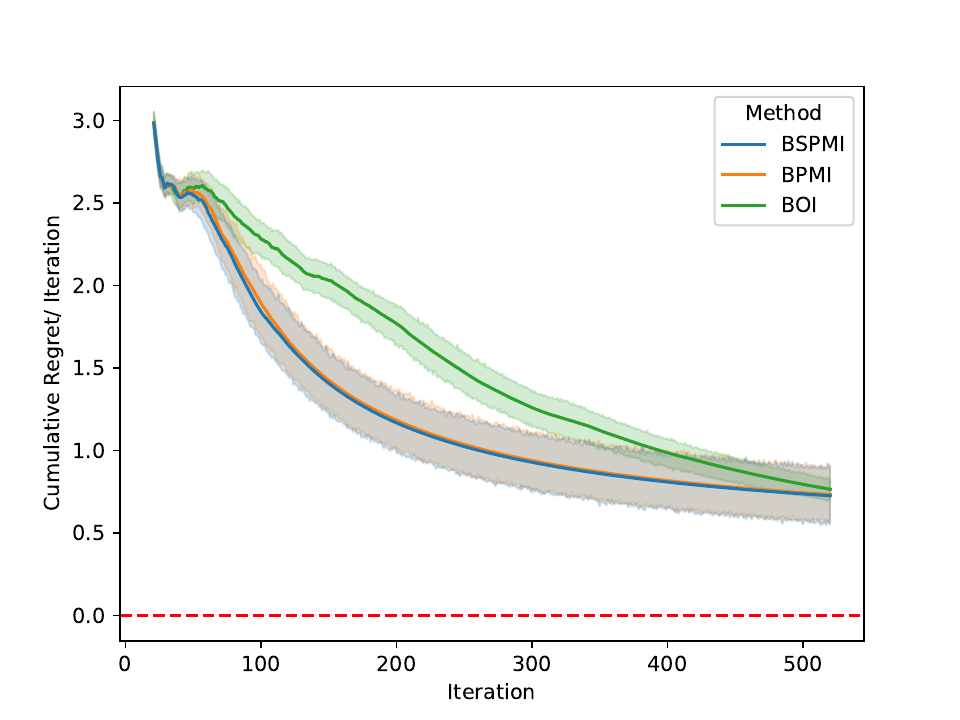}
\caption{Schwefel 2D, $\sigma = 0.01$}
\label{fig:2b}
\end{subfigure}
\begin{subfigure}{0.33\linewidth}
\includegraphics[width=\linewidth]{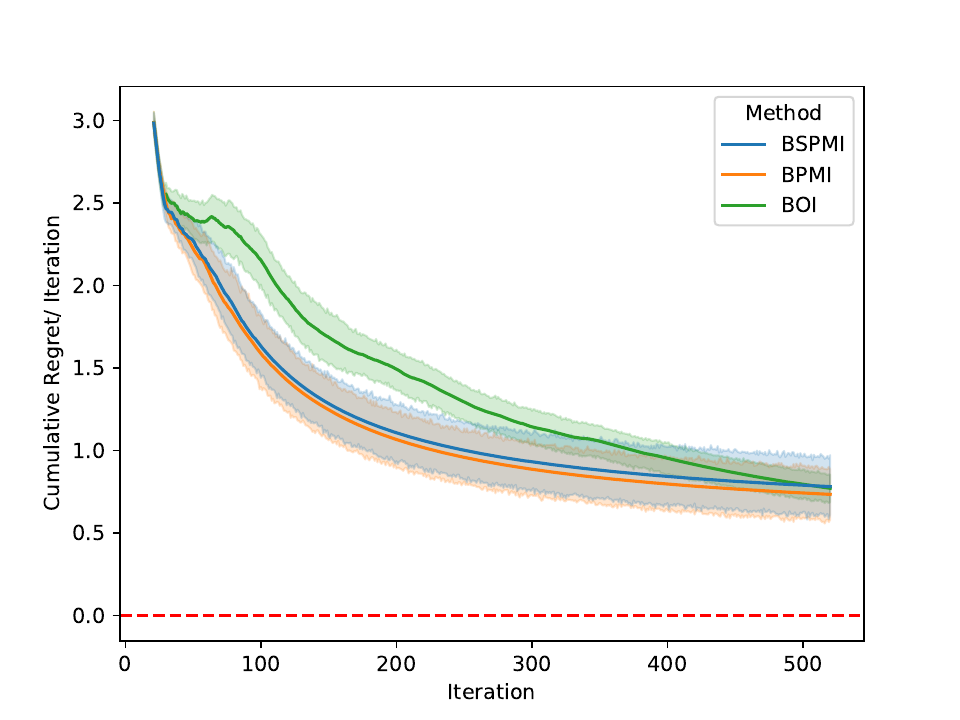}
\caption{Schwefel 2D, $\sigma = 0.1$}
\label{fig:2c}
\end{subfigure}

\caption{Cumulative regret/the number of iterations for different GP-EI algorithms with different incumbents on four 2D test functions (Branin 2D, Styblinski-Tang 2D, Camel 2D and Schwefel 2D) with the Matérn ($\nu =1.5$) kernel.}
\label{fig:results1}
\end{figure}

\begin{figure}[!t]
\centering
\begin{subfigure}{0.33\linewidth}
\includegraphics[width=\linewidth]{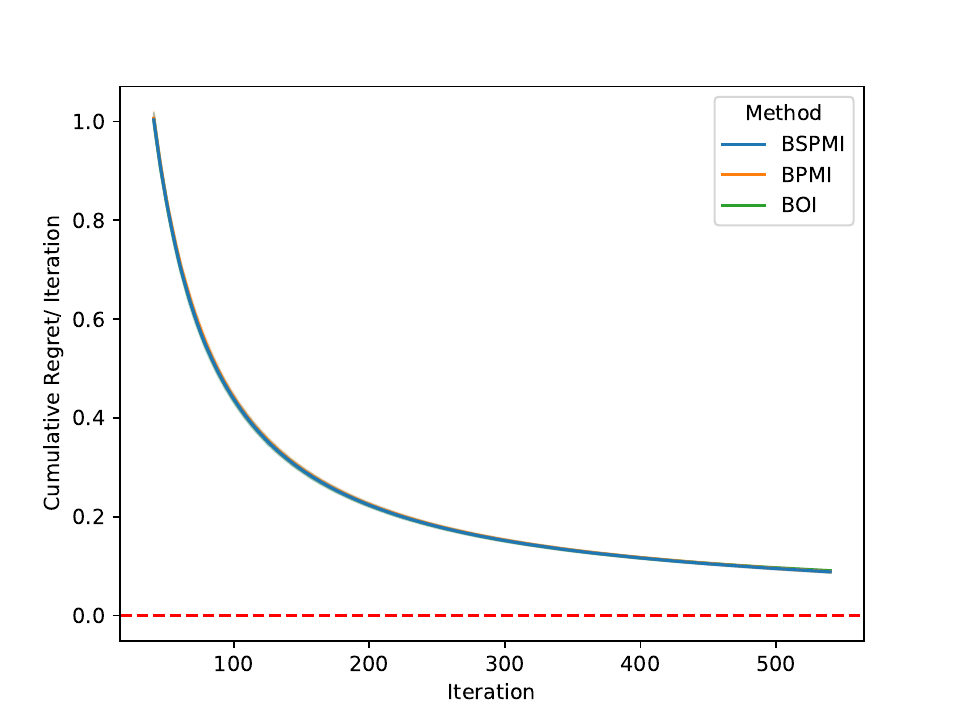} 
\caption{Rosenbrock 4D, $\sigma = 0.001$}
\label{fig:2a}
\end{subfigure}\hfill
\begin{subfigure}{0.33\linewidth}
\includegraphics[width=\linewidth]{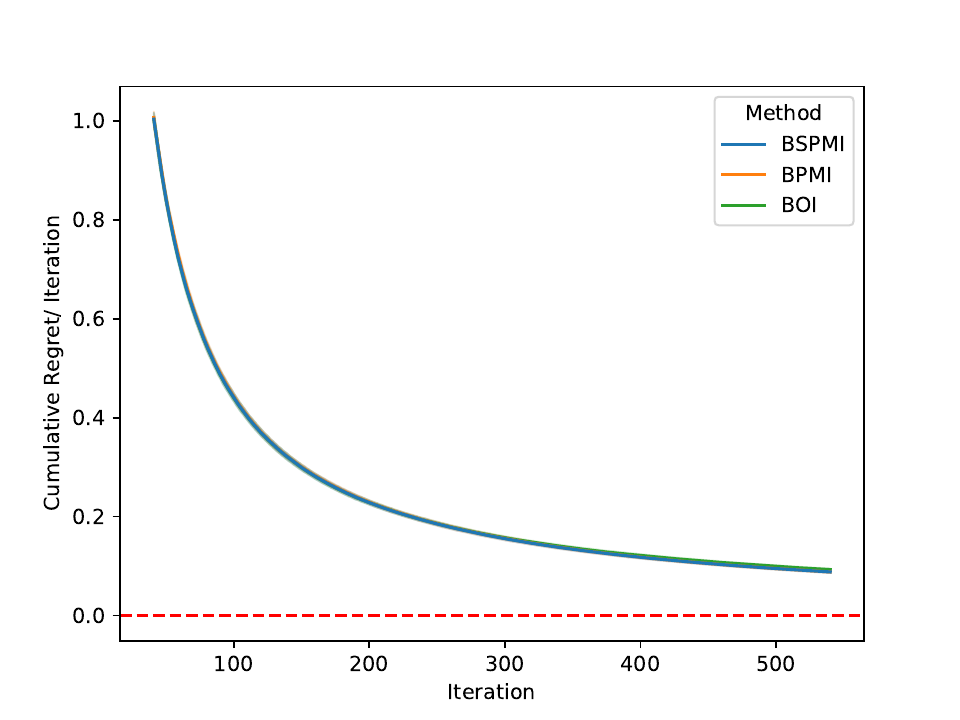}
\caption{Rosenbrock 4D, $\sigma = 0.01$}
\label{fig:2b}
\end{subfigure}
\begin{subfigure}{0.33\linewidth}
\includegraphics[width=\linewidth]{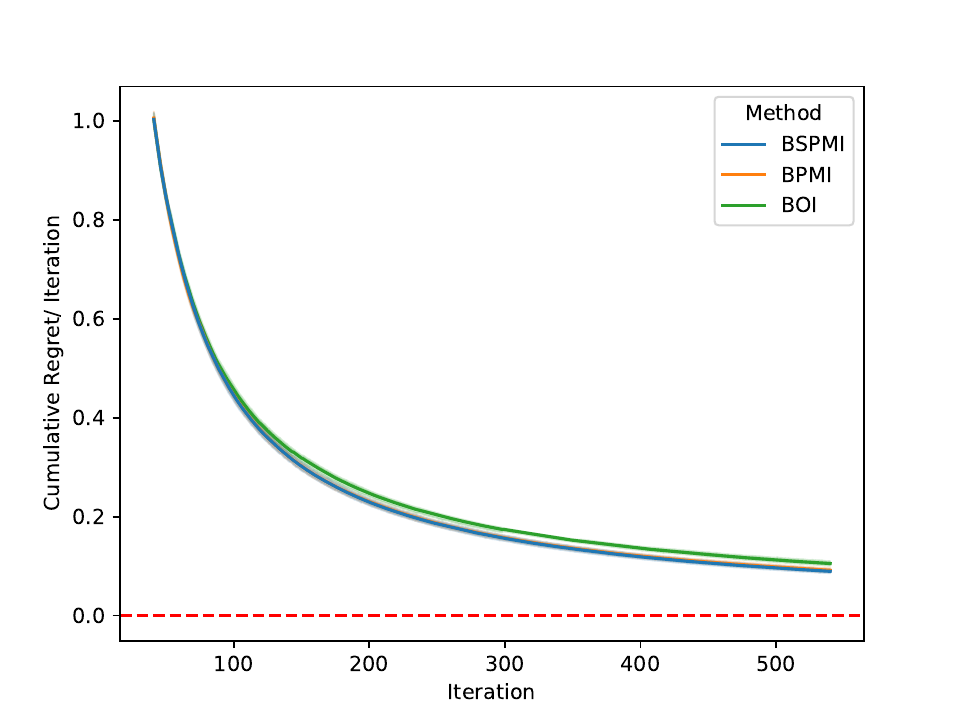}
\caption{Rosenbrock 4D, $\sigma = 0.1$}
\label{fig:2c}
\end{subfigure}

\begin{subfigure}{0.33\linewidth}
\includegraphics[width=\linewidth]{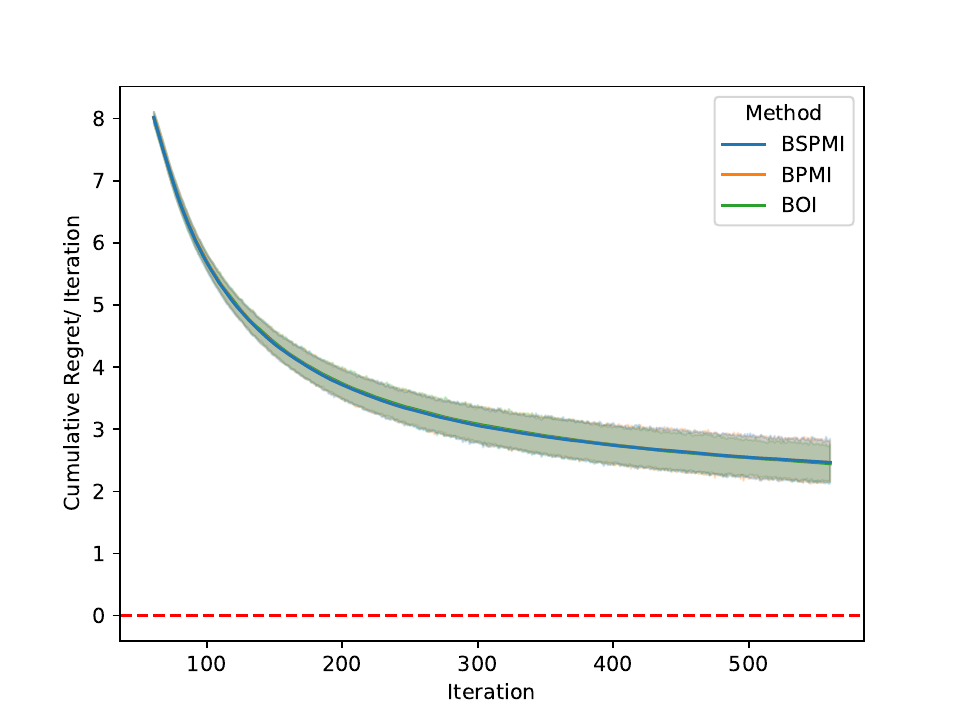} 
\caption{Hartmann 6D, $\sigma = 0.001$}
\label{fig:2a}
\end{subfigure}\hfill
\begin{subfigure}{0.33\linewidth}
\includegraphics[width=\linewidth]{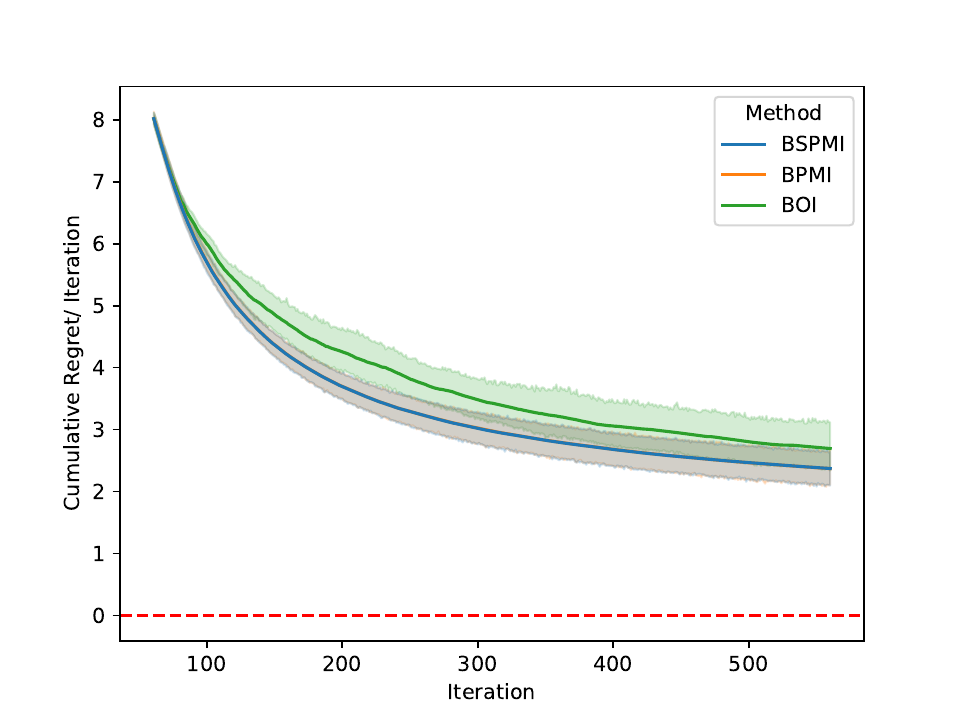}
\caption{Hartmann 6D, $\sigma = 0.01$}
\label{fig:2b}
\end{subfigure}
\begin{subfigure}{0.33\linewidth}
\includegraphics[width=\linewidth]{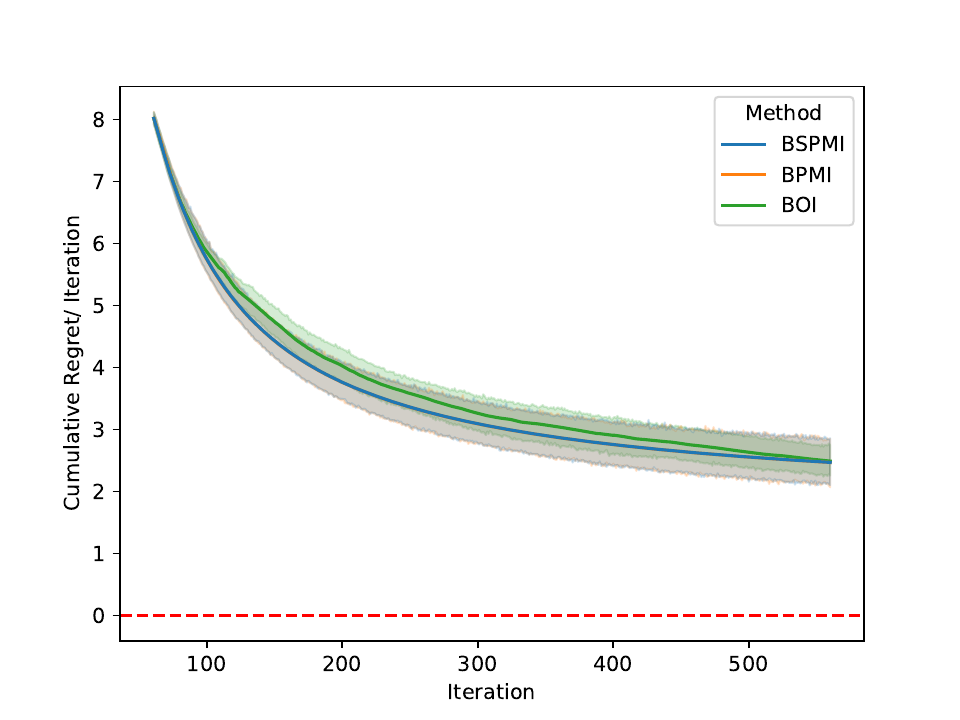}
\caption{Hartmann 6D, $\sigma = 0.1$}
\label{fig:2c}
\end{subfigure}

\caption{Cumulative regret/the number of iterations for different GP-EI algorithms with different incumbents on two higher dimension test functions (Rosenbrock 4D and Hartmann 6D) with the  Matérn ($\nu =1.5$) kernel.}
\label{fig:results2}
\end{figure}

It is evident that BPMI and BSPMI display no-regret behavior, i.e. $R_t/t$ converges to 0 as $t$ increases for most of the functions. For BOI, when $\sigma$ is small (0.001), its performance is similar to that of BPMI and BSPMI. However, as the noise level increases, and, thus, $y^+_{t-1}<f(\xbm^*)$ more likely to occur, BOI performs noticeably worse than BPMI and BSPMI, especially for more complex and higher-dimensional problems. This is consistent with our theoretical findings, including our discussion in Remark~\ref{remark:boi-inst-regret} and~\ref{remark:boi-regret}, as well as previous benchmark tests in literature~\citep{picheny2013benchmark}.
To further validate this point, we also plot $y_{t-1}^+-f(\xbm^*)$ for BOI for the Rosenbrock 4D function for $\sigma = 0.1$ in Figure~\ref{fig:noisysimpleregret}. It is clear that the $y_{t-1}^+-f(\xbm^*)$ decreases below $0$ quickly (just with the initial design of 40 points, the value already falls below 0). 
 
\begin{figure}[!t]
    \centering
    \includegraphics[width=0.5\linewidth]{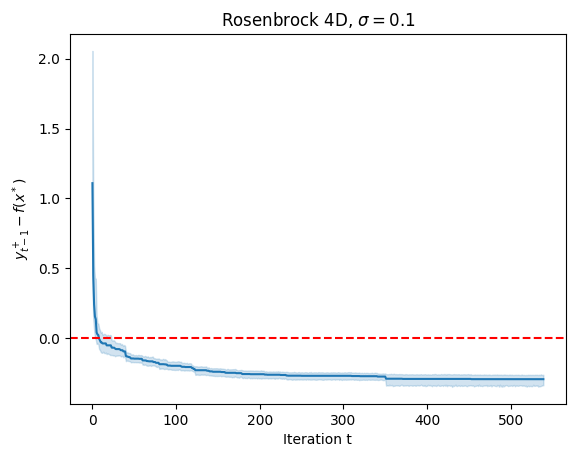}
    \caption{Values of $y_{t-1}^+-f(\xbm^*)$ for Rosenbrock 4D function for $\sigma = 0.1$. Blue solid lines represents the mean and and the shaded regions indicate the 95\% confidence intervals.}
    \label{fig:noisysimpleregret}
\end{figure}

\section{Conclusions}\label{se:conclusion}
We prove the cumulative regret upper bounds for classic GP-EI algorithms in this paper. We show that GP-EI with both BPMI and BSPMI are no-regret algorithms for SE and Mat\'{e}rn kernels. We further show that with BOI, GP-EI either has a sublinear cumulative regret bound or has a sublinear noisy simple regret bound. With the established regret bounds, practitioners have theoretical guidance on the choice of the incumbent.

%% file: Sections/Appx-Background.tex
\section{Background information}\label{appdx:back}
We first define two equivalent forms of EI~\eqref{eqn:EI-1}.
We distinguish between its exploration and exploitation parts and define the \textit{trade-off} form $EI(a,b):\Rbb\times\Rbb\to\Rbb$ as
\begin{equation} \label{eqn:EI-ab}
 \centering
  \begin{aligned}
       EI(a,b) = a \Phi\left(\frac{a}{b}\right)+b \phi\left(\frac{a}{b}\right), 
  \end{aligned}
\end{equation}
where $b\in(0,1]$. One can view $a$ and $b$ as two independent variables.
For a given $\xbm$, if $a_t=\xi^+_{t}-\mu_{t}(\xbm)$ and $b_t=\sigma_{t}(\xbm)\in (0,1]$, $EI(a_t,b_t)=EI_t(\xbm)$. 

Another commonly used function in the analysis of EI is the $\tau:\Rbb\to\Rbb$ function thanks to its reduced dimensionality. It is defined as 
\begin{equation} \label{def:tau}
 \centering
  \begin{aligned}
   \tau(z) = z\Phi(z) + \phi(z). 
  \end{aligned}
\end{equation}
Thus, the $\tau$ form of EI can be written as $EI_t(\xbm) = \sigma_t(\xbm) \tau(z_t(\xbm))$.
Union bound is given in the next lemma.
\begin{lem}\label{lem:unionbound}
  For a countable set of events $A_1,A_2,\dots$, we have 
 \begin{equation*} \label{eqn:union-bound-1}
  \centering
  \begin{aligned}
    \Pbb(\bigcup_{i=1}^{\infty} A_i) \leq \sum_{i=1}^{\infty} \Pbb(A_i).
  \end{aligned}
\end{equation*}
\end{lem}
  
The following Lemmas are well-established results from~\cite{srinivas2009gaussian} on the information gain and variances.
\begin{lem}\label{lem:variancebound}
 The sum of posterior standard deviation $\sigma_{t}(\xbm)$ satisfies
 \begin{equation} \label{eqn:var-1}
  \centering
  \begin{aligned}
    \sum_{t=1}^T  \sigma_{t-1}(\xbm) \leq \sqrt{C_{\gamma} T \gamma_T}, 
  \end{aligned}
\end{equation}
 where $C_{\gamma} = 2/log(1+\sigma^{-2})$.
\end{lem}
The state-of-the-art rates of $\gamma_t$ for two commonly-used kernels are summarized below~\citep{vakili2021information}.
\begin{lem}\label{lem:gammarate}
  For a GP with $t$ samples, the information gain for the SE kernel is $\gamma_t=\mathcal{O}\left(\log^{d+1}(t)\right)$, while for 
  the Mat\'{e}rn kernel with  parameter $\nu>\frac{1}{2}$ is $\gamma_t=\mathcal{O}\left(t^{\frac{d}{2\nu+d}}\log^{\frac{2\nu}{2\nu+d}} (t)\right)$.
\end{lem}
The formal definition of RKHS is given below.
\begin{defi}\label{def:rkhs}
   Consider a positive definite kernel $k: \mathcal{X}\times \mathcal{X}\to\Rbb $ with respect to a finite Borel measure supported on $\mathcal{X}$. A Hilbert space $H_k$ of functions on $\mathcal{X}$ with an inner product $\langle \cdot,\cdot \rangle_{H_k}$ is called a RKHS with kernel $k$ if $k(\cdot,\xbm)\in H_k$ for all $\xbm\in \mathcal{X}$, and $\langle f,k(\cdot,\xbm)\rangle_{H_k}=f(\xbm)$ for all $\xbm\in \mathcal{X}, f\in H_k$. The induced RKHS norm $\norm{f}_{H_k}=\sqrt{\langle f,f\rangle_{H_k}}$ measures the smoothness of $f$ with respect to $k$.
\end{defi}
The super-martingale series is defined next.
\begin{defi}\label{def:supmart}
   A sequence of random variables $\{Y_t\},t=0,1,\dots$ is called a super-martingale of a filtration $\mathcal{F}_t$ if for all $t$, $Y_t$ is $\mathcal{F}_t$-measurable, and for $t\geq 1$,
    \begin{equation} \label{eqn:supmart-1}
  \centering
  \begin{aligned}
     \Ebb[Y_t|\mathcal{F}_{t-1}] \leq Y_{t-1}.
    \end{aligned}
  \end{equation} 
\end{defi}
The Azuma-Hoeffding inequality is given in the next lemma, which provides an important bound for a super-martingale~\citep{azuma1967weighted}.
\begin{lem}\label{lem:supmart}
    If a super-martingale $\{Y_t\}$ of $\mathcal{F}_t$ satisfies $|Y_t-Y_{t-1}|\leq c_t$ for some constant $c_t$, $t=1,\dots,T$, then for $\forall \delta\geq 0$,
    \begin{equation} \label{eqn:supmart-2}
  \centering
  \begin{aligned}
       \Pbb\left[Y_T-Y_{0}\leq \sqrt{2\log(1/\delta)\sum_{t=1}^T c_t^2}\right]\geq 1-\delta.
    \end{aligned}
  \end{equation} 
\end{lem}
Before concluding the section, we state the straightforward bound of $f$ on $C$ as a Lemma for easy reference. 
\begin{lem}\label{lem:f-bound}
  There exists $B>0$, such that $|f(\xbm)|\leq B$ for all $\xbm\in C$.
\end{lem}
The Lipschitz continuity of $f$ extends to $I_t(\xbm)=\max\{\xi_t^+-f(\xbm),0\}$ in the next lemma.
\begin{lem}\label{lem:lipschitz-cplus}
    If a function $f:\Rbb^d\to\Rbb$ is Lipschitz continuous on $O\subset\Rbb^d$ with Lipschitz constant $L$, then its improvement function is also Lipschitz, \textit{i.e.}
    \begin{equation} \label{eqn:lip-plus1}
  \centering
  \begin{aligned}
     |I_t(\xbm) - I_t(\xbm')| \leq L \norm{\xbm-\xbm'}_1,
  \end{aligned}
\end{equation}
where  $\xbm,\xbm'\in O$.
\end{lem} 
\begin{proof}
  From the definition of Lipschitz functions, we know
     \begin{equation} \label{eqn:lip-plus-pf-1}
  \centering
  \begin{aligned}
     |f(\xbm) - f(\xbm')| \leq L \norm{\xbm-\xbm'}_1.
  \end{aligned}
\end{equation}
   If $I_t(\xbm)>0$ and $I_t(\xbm')>0$, then 
     \begin{equation} \label{eqn:lip-plus-pf-2}
  \centering
  \begin{aligned}
     |I_t(\xbm) - I_t(\xbm')| = |f(\xbm') - f(\xbm)| \leq L \norm{\xbm-\xbm'}_1.
  \end{aligned}
\end{equation}
   If $I_t(\xbm)>0$ and $I_t(\xbm')= 0 $, then  
     \begin{equation} \label{eqn:lip-plus-pf-3}
  \centering
  \begin{aligned}
     |I_t(\xbm) - I_t(\xbm')| = \xi_t^+-f(\xbm) \leq \xi_t^+-f(\xbm) -\xi_t^+ + f(\xbm')\leq  |f(\xbm) - f(\xbm')| \leq L \norm{\xbm-\xbm'}_1.
  \end{aligned}
\end{equation}
   Similarly, if $I_t(\xbm)= 0$ and $I_t(\xbm')> 0 $, then  
     \begin{equation} \label{eqn:lip-plus-pf-4}
  \centering
  \begin{aligned}
     |I_t(\xbm) - I_t(\xbm')| = \xi^+_t-f(\xbm') \leq -\xi^+_t+f(\xbm')+ \xi^+_t-f(\xbm')\leq  |f(\xbm) - f(\xbm')| \leq L \norm{\xbm-\xbm'}_1.
  \end{aligned}
\end{equation}
   Finally if $I_t(\xbm) = 0$ and $I_t(\xbm')= 0 $, then  
      \begin{equation} \label{eqn:lip-plus-pf-5}
  \centering
  \begin{aligned}
     |I_t(\xbm) - I_t(\xbm')| = 0  \leq L \norm{\xbm-\xbm'}_1.
  \end{aligned}
\end{equation}
\end{proof}

%% file: Sections/Appx-Prep.tex
\section{Preliminary lemmas and Proofs}\label{se:prep-proof}  

In order to facilitate the analysis of the regret bounds, we apply and establish several existing and novel results (Lemmas~\ref{lem:tauvsPhi},~\ref{lem:EI}, and~\ref{lem:EI-ms}) to characterize EI in Section~\ref{se:EIproperty-proof}. 
 Then, we establish the preliminary bounds for GP-EI in Section~\ref{se:proofGP} under the Bayesian setting. Similar results were established in ~\citep{wang2025convergence}. 
 Finally, in Section~\ref{se:proofofsigmabound}, we present novel proofs for the global lower bound of posterior standard deviation, an important step in the regret analysis.
\subsection{Properties of EI}\label{se:EIproperty-proof}
First, we state some straightforward properties of $\phi$, $\Phi$ and $\tau$ as a lemma.
\begin{lem}\label{lem:phi}
The PDF and CDF of standard normal distribution satisfy $0< \phi(x)\leq \phi(0)<0.4, \Phi(x)\in(0,1)$, 
for any $x\in\Rbb$. 
 Given a random variable $r$ that follows the standard normal distribution: $r\sim\mathcal{N}(0,1)$, the probability of $r>c, c>0$, satisfies $\Pbb\{r>c\}\leq \frac{1}{2}e^{-c^2/2}$. 
 Similarly, for $c<0$, $\Pbb\{r<c\}\leq \frac{1}{2}e^{-c^2/2}$. 
\end{lem}
The last statement in Lemma~\ref{lem:phi} is a well-known result for standard normal distribution (\textit{e.g.}, see proof of Lemma 5.1 in~\cite{srinivas2009gaussian}).
The property of $\tau(\cdot)$ in~\eqref{def:tau} is given below.
\begin{lem}\label{lem:tau}
  The function $\tau(z)$ is monotonically increasing in $z$ and $\tau(z)>0$ for $\forall z\in \Rbb$.
  The derivative of $\tau(z)$ is $\Phi(z)$.
\end{lem}
\begin{proof}
   From the definition of $\tau(z)$, we can write
  \begin{equation} \label{eqn:tau-1}
  \centering
   \begin{aligned}
   \tau(z)=z\Phi(z)+\phi(z) > \int_{-\infty}^{z} u\phi(u) du + \phi(z) = -\phi(u)|_{-\infty}^z+\phi(z) =0.
  \end{aligned}
  \end{equation}
  Given the definition of $\phi(u)$, 
  \begin{equation} \label{eqn:tau-2}
 \centering
  \begin{aligned}
      \frac{d \phi(u)}{d u} = \frac{1}{\sqrt{2\pi}} e^{-\frac{u^2}{2}} (-u) = -\phi(u)u.
   \end{aligned}
\end{equation}
   Thus, the derivative of $\tau$ is 
   \begin{equation} \label{eqn:tau-3}
 \centering
  \begin{aligned}
      \frac{d\tau(z)}{d z} = \Phi(z)+z\phi(z) -\phi(z)z = \Phi(z)>0.
   \end{aligned}
\end{equation}
\end{proof}
Another lemma comparing $\tau$ with $\Phi$ is given below.
\begin{lem}\label{lem:tauvsPhi}
  For any $z>0$, $\Phi(-z)>\tau(-z)$. 
\end{lem}
\begin{proof}
   Define $q(z) = \Phi(-z)-\tau(-z)$.
   Using integration by parts, we have
   \begin{equation} \label{eqn:tvp-pf-1}
 \centering
  \begin{aligned}
   \Phi(z) = \int_{-\infty}^z \phi(u)du > \int_{-\infty}^{z}\phi(u)\left(1-\frac{3}{u^4}\right)du=-\frac{\phi(z)}{z}+\frac{\phi(z)}{z^3}. 
   \end{aligned}
\end{equation}
  Replacing $z$ with $-z$ in~\eqref{eqn:tvp-pf-1},
   \begin{equation} \label{eqn:tvp-pf-1.5}
 \centering
  \begin{aligned}
   \phi(-z)\left(\frac{1}{z}-\frac{1}{z^3}\right)<\Phi(-z).
   \end{aligned}
\end{equation}
   Multiplying both sides in~\eqref{eqn:tvp-pf-1.5} by $1+z$,
\begin{equation} \label{eqn:tvp-pf-2}
 \centering
  \begin{aligned}
   (1+z)\Phi(-z)> \phi(-z)\frac{z^2-1}{z^3}(1+z)= \phi(-z)\left(1+\frac{z^2-z-1}{z^3}\right). 
   \end{aligned}
\end{equation}
   Thus, if $z > \frac{1+\sqrt{5}}{2}$, then the right-hand-side of~\eqref{eqn:tvp-pf-2} $> \phi(-z)$ and  
 \begin{equation} \label{eqn:tvp-pf-3}
 \centering
  \begin{aligned}
   q(z):=\Phi(-z)-\tau(-z) = (1+z)\Phi(-z) - \phi(-z)>0.
   \end{aligned}
\end{equation}
   Therefore, in the following, we focus on $z\in(0,\frac{1+\sqrt{5}}{2}]$. We analyze $q(z)$ using its derivatives.
   Taking the derivative of $q$, by Lemma~\ref{lem:tau},  
   \begin{equation} \label{eqn:tvp-pf-4}
 \centering
  \begin{aligned}
      \frac{d q(z)}{d z} = -\phi(-z)+\Phi(-z) := q'(z).
   \end{aligned}
\end{equation}
  Further, the derivative of $q'(z)$ is 
  \begin{equation} \label{eqn:tvp-pf-5}
 \centering
  \begin{aligned}
      \frac{d^2 q(z)}{d z^2} = \frac{d q'(z)}{d z} = -\phi(-z)+\phi(-z) z=\phi(z)(z-1).
   \end{aligned}
\end{equation}
  For $z>1$, $\frac{d^2 q(z)}{d z^2}>0$. For $0<z<1$, $\frac{d^2 q(z)}{d z^2}<0$. 
  Thus, $q'(z)$ is monotonically decreasing for $0<z<1$ and then monotonically increasing for $z>1$ .
  We first consider $q'(z)$ for $0<z<1$.
  We know that by simple algebra, $q'(0)=\Phi(0)-\phi(0)>0$ and $q'(1)=\Phi(-1)-\phi(-1)<0$.  
   Thus, there exists a $0<\bar{z}<1$ so that $q'(\bar{z})=0$.
  Next, for $z>1$, from Lemma~\ref{lem:tau}, we can write  
    \begin{equation} \label{eqn:tvp-pf-6}
 \centering
  \begin{aligned}
    q'(z) =  \Phi(-z)-\phi(-z) < z\Phi(-z) - \phi(-z) = - \tau(-z) < 0.
   \end{aligned}
\end{equation}
   Therefore, $0<\bar{z}<1<\frac{1+\sqrt{5}}{2}$ is a unique stationary point such that $q'(\bar{z})=0$.
   Thus, $q'(z)>0 $ for $0 <z< \bar{z}< \frac{1+\sqrt{5}}{2}$ and $q'(z)<0$ for $z> \bar{z}$.
   This means that for $0<z<\bar{z}$, $q(z)$ is monotonically increasing. For $ \bar{z}<z<\frac{1+\sqrt{5}}{2}$, $q(z)$ is monotonically decreasing. Therefore, $q(z) >\min\{q(0), q(\frac{1+\sqrt{5}}{2})\}$ for $z\in (0,\frac{1+\sqrt{5}}{2})$. Since $q(0)>0$ and $q(\frac{1+\sqrt{5}}{2})>0$, $q(z)>0$ for $z\in (0,\frac{1+\sqrt{5}}{2})$. 
   Combined with~\eqref{eqn:tvp-pf-3}, the proof is complete.
\end{proof}
Next, we state a basic inequality for $\Phi(\cdot)$ without proof next.
\begin{lem}\label{lem:Philowbound}
  Given $z>0$, $\Phi(-z)=1-\Phi(z) > \frac{z}{1+z^2}\frac{1}{\sqrt{2\pi}}e^{-z^2/2}$. 
\end{lem}
The next lemma establishes the upper and lower bounds for $EI_{t-1}(\xbm)$. This lemma is used to bound $EI_{t-1}(\xbm_t)$ with $\sigma_{t-1}(\xbm_t)$.
\begin{lem}\label{lem:EI}
 $EI_t(\xbm)$ satisfies 
    $EI_t(\xbm) \geq 0$ and $EI_t(\xbm) \geq \xi_{t}^+- \mu_{t}(\xbm)$.
Moreover, 
\begin{equation} \label{eqn:EI-property-2}
 \centering
  \begin{aligned}
   z_{t}(\xbm)\leq  \frac{EI_{t}(\xbm)}{\sigma_{t}(\xbm) }=\tau(z_t(\xbm)) < \begin{cases}  \phi(z_{t}(\xbm)), \ &z_{t}(\xbm)<0\\
             z_{t}(\xbm) +\phi(z_{t}(\xbm)), \ &z_{t}(\xbm)\geq 0.
      \end{cases}
  \end{aligned}
\end{equation}

\end{lem}
\begin{proof}
     From the definition of $I_t$ and $EI_t$, the first statement follows immediately.
     By~\eqref{eqn:EI-1},  
\begin{equation} \label{eqn:EI-property-pf-1}
 \centering
  \begin{aligned}
     \frac{EI_t(\xbm)}{\sigma_{t}(\xbm) } = z_{t}(\xbm) \Phi(z_{t}(\xbm)) + \phi(z_{t}(\xbm)). 
  \end{aligned}
\end{equation}
If $z_{t}(\xbm)<0$, or equivalently $\xi_{t}^+ -\mu_{t}(\xbm)<0$,~\eqref{eqn:EI-property-pf-1} leads to
     $\frac{EI_t(\xbm)}{\sigma_{t}(\xbm)} < \phi(z_{t}(\xbm))$. 
If $z_{t}(\xbm)\geq 0$, $\Phi(\cdot)<1$ gives us
     $\frac{EI_t(\xbm)}{\sigma_{t}(\xbm)} < z_{t}(\xbm) +\phi(z_{t}(\xbm))$. 
 The left inequality in~\eqref{eqn:EI-property-2} is an immediate result of $EI_t(\xbm) \geq \xi^+_{t}- \mu_{t}(\xbm)$.
\end{proof}
The next lemma puts a bound on $\xi^+_{t-1}-\mu_{t-1}(\xbm)$ if $\xi^+_{t-1}-\mu_{t-1}(\xbm)<0$ and $EI_{t-1}(\xbm)$ is bounded below by a positive sequence denoted as $\kappa_t$. It is also previously shown in~\cite{nguyen17a}.
\begin{lem}\label{lem:mu-bounded-EI}
   If $EI_{t-1}(\xbm)\geq \kappa_t$ for some $\frac{1}{\sqrt{2\pi}}>\kappa_t>0$ and $\xi^+_{t-1}-\mu_{t-1}(\xbm)<0$, then we have 
  \begin{equation} \label{eqn:mu-bounded-EI-1}
  \centering
  \begin{aligned}
           \xi_{t-1}^+-\mu_{t-1}(\xbm) \geq -  \sqrt{2\log\left(\frac{1}{\sqrt{2\pi}\kappa_t}\right)}\sigma_{t-1}(\xbm).
   \end{aligned}
\end{equation}
\end{lem}
\begin{proof}
     By definition of $EI_{t-1}$,
  \begin{equation} \label{eqn:mu-bounded-EI-pf-1}
  \centering
  \begin{aligned}
       \kappa_t\leq& (\xi_{t-1}^+-\mu_{t-1}(\xbm))\Phi(z_{t-1}(\xbm))+\sigma_{t-1}(\xbm)\phi(z_{t-1}(\xbm)) < \sigma_{t-1}(\xbm)\phi(z_{t-1}(\xbm))\\
    =& \sigma_{t-1}(\xbm) \frac{1}{\sqrt{2\pi}} e^{-\frac{1}{2}z_{t-1}^2(\xbm)}\leq \frac{1}{\sqrt{2\pi}} e^{-\frac{1}{2}z_{t-1}^2(\xbm)}.
   \end{aligned}
\end{equation}
   Rearranging and taking the logarithm of~\eqref{eqn:mu-bounded-EI-pf-1}, we have 
  \begin{equation} \label{eqn:mu-bounded-EI-pf-2}
  \centering
  \begin{aligned}
       2\log(\frac{1}{\sqrt{2\pi}\kappa_t}) >  z_{t-1}^2(\xbm) = \left(\frac{\xi_{t-1}^+-\mu_{t-1}(\xbm)}{\sigma_{t-1}(\xbm)}\right)^2.
   \end{aligned}
\end{equation}
   Given that $\xi^+_{t-1}-\mu_{t-1}(\xbm)<0$, we recover~\eqref{eqn:mu-bounded-EI-1}.  
\end{proof}
 In the next lemma, $EI_{t-1}(\xbm)$ is shown to increase monotonically with respect to both its exploitation part $\xi_{t-1}^+-\mu_{t-1}(\xbm)$ and exploration part $\sigma_{t-1}(\xbm)$. Furthermore, the ratio of the partial derivatives with respect to both exploitation and exploration is bounded when $z_{t-1}(\xbm)<0$. This lemma plays an important role in analyzing the exploration and exploitation trade-off of EI and is used in Section~\ref{se:instant-reg}
\begin{lem}\label{lem:EI-ms}
  $EI(a,b)$ is monotonically increasing in both $a$ and $b$ for $b\in(0,1]$. 
  Further, let $z_{ab}=\frac{a}{b}$. If $z_{ab}<0$, 
  \begin{equation} \label{eqn:EI-ms-1}
 \centering
  \begin{aligned}
  \frac{\frac{\partial EI(a,b)}{\partial a} }{\frac{\partial EI(a,b)}{\partial b}}=\frac{\Phi(z_{ab})}{\phi(z_{ab})}<-\frac{1}{z_{ab}}+\frac{1}{z_{ab}^3}-\frac{3}{z_{ab}^5}.
  \end{aligned}
\end{equation}
  Consequently, the derivative ratio on the left-hand side of~\eqref{eqn:EI-ms-1} $\to 0$ if $z_{ab}\to -\infty$.
\end{lem}
\begin{proof}
   We prove the lemma by taking the derivative of $EI(a,b)$ with respect to both variables. First,  
 \begin{equation} \label{eqn:EI-ms-pf-1}
 \centering
  \begin{aligned}
   \frac{\partial EI(a,b)}{\partial a} = \Phi\left(\frac{a}{b}\right) + a\phi\left(\frac{a}{b}\right) \frac{1}{b} + b\frac{\partial \phi\left(\frac{a}{b}\right)}{\partial a}.
   \end{aligned}
\end{equation}
   From~\eqref{eqn:tau-2},~\eqref{eqn:EI-ms-pf-1} is  
  \begin{equation} \label{eqn:EI-ms-pf-2}
 \centering
  \begin{aligned}
   \frac{\partial EI(a,b)}{\partial a} = \Phi\left(\frac{a}{b}\right) + \phi\left(\frac{a}{b}\right) \frac{a}{b}  -\phi\left(\frac{a}{b}\right)\frac{a}{b}  = \Phi\left(\frac{a}{b}\right)>0.
   \end{aligned}
\end{equation}
  Similarly, 
 \begin{equation} \label{eqn:EI-ms-pf-3}
 \centering
  \begin{aligned}
   \frac{\partial EI(a,b)}{\partial b} =&  -a\phi\left(\frac{a}{b}\right)\frac{a}{b^2} + \phi\left(\frac{a}{b}\right)-b\phi\left(\frac{a}{b}\right)\frac{a}{b}(-\frac{a}{b^2})
     = \phi\left(\frac{a}{b}\right)>0.
   \end{aligned}
\end{equation}
  Thus, 
   \begin{equation} \label{eqn:EI-ms-pf=4}
 \centering
  \begin{aligned}
  \frac{\frac{\partial EI(a,b)}{\partial a} }{\frac{\partial EI(a,b)}{\partial b}}=\frac{\Phi(z_{ab})}{\phi(z_{ab})}.
  \end{aligned}
\end{equation}
  Using integration by parts, we can write for $z<0$ that 
\begin{equation*} \label{eqn:EI-ms-pf=5}
 \centering
  \begin{aligned}
  \Phi(z)=\int_{-\infty}^z \phi(u)du =  \left(-\frac{1}{z}+\frac{1}{z^3}-\frac{3}{z^5}\right) \phi(z) -\int_{-\infty}^{z}\phi(u)\frac{15}{u^6} du.
  \end{aligned}
\end{equation*}
Thus, 
\begin{equation} \label{eqn:EI-ms-pf=5}
 \centering
  \begin{aligned}
\Phi(z)< \left(-\frac{1}{z}+\frac{1}{z^3}-\frac{3}{z^5}\right) \phi(z).
\end{aligned}
\end{equation}
 Using $z_{ab}$ in place of $z$ in~\eqref{eqn:EI-ms-pf=5} completes the proof.
\end{proof}

\subsection{Preliminary bounds of GP-EI}\label{se:proofGP}
The basic confidence interval on $|f(\xbm) - \mu_{t}(\xbm)  |$ in the Bayesian setting is presented below.
\begin{lem}\label{lem:fmu}
Given $\delta\in(0,1)$, let  $\beta = 2\log(1/\delta)$. For any given $\xbm\in C$ and $t\in \Nbb 1$,   
\begin{equation} \label{eqn:fmu-1}
 \centering
  \begin{aligned}
     |f(\xbm) - \mu_{t}(\xbm)  | \leq \beta^{1/2} \sigma_{t}(\xbm), 
  \end{aligned}
\end{equation}
holds with probability $\geq 1-\delta$.
Moreover, the one-sided inequalities hold with probability $\geq 1-\delta/2$, \textit{i.e.},    with probability $\geq 1-\delta/2$,
     $f(\xbm) - \mu_{t}(\xbm)   \leq \beta^{1/2} \sigma_{t}(\xbm)$, 
and with probability $\geq 1-\delta/2$,
     $f(\xbm) - \mu_{t}(\xbm)   \geq -\beta^{1/2} \sigma_{t}(\xbm)$.
\end{lem}
\begin{proof}
 Under Assumption~\ref{assp:gp}, $f(\xbm)\sim \mathcal{N}(\mu_{t}(\xbm),\sigma_{t}^2(\xbm))$. By Lemma~\ref{lem:phi}, 
 \begin{equation} \label{eqn:fmu-a-pf-1}
 \centering
  \begin{aligned}
     \Pbb\left\{ f(\xbm)-\mu_t(\xbm) > \beta^{1/2} \sigma_t(\xbm)\right\} = 1-\Phi\left(\beta^{1/2} \right) \leq \frac{1}{2} e^{-\beta/2}.
  \end{aligned}
 \end{equation}
  Similarly, 
 \begin{equation} \label{eqn:fmu-a-pf-2}
 \centering
  \begin{aligned}
     \Pbb\left\{ f(\xbm)-\mu_t(\xbm) < -\beta^{1/2} \sigma_t(\xbm)\right\} \leq \frac{1}{2} e^{-\beta /2}.
  \end{aligned}
 \end{equation}
  Thus, 
   \begin{equation*} \label{eqn:fmu-a-pf-3}
 \centering
  \begin{aligned}
     \Pbb\left\{ |f(\xbm)-\mu_t(\xbm)| < \beta^{1/2} \sigma_t(\xbm)\right\} \geq 1- e^{-\beta /2}.
  \end{aligned}
 \end{equation*}
  Let 
     $e^{-\beta/2} = \delta$
  and~\eqref{eqn:fmu-1} is proven. Similarly, from~\eqref{eqn:fmu-a-pf-1},~\eqref{eqn:fmu-a-pf-2}, we obtain the one-sided inequalities.
\end{proof}
Lemma~\ref{lem:fmu} is extended to a discrete set $\Cbb$ and all $t\in\Nbb$ next.
\begin{lem}[Lemma 5.1 in~\cite{srinivas2009gaussian}]\label{lem:discrete-fmu}
  Given $\delta\in(0,1)$ and a discrete set $\Cbb\subseteq C$, let  $\beta_t = 2\log(|\Cbb|\pi_t /\delta)$, where $\pi_t$ satisfies $\sum_{t=1}^{\infty} \frac{1}{\pi_t}=1,\pi_t>0$. Then,  for all $\xbm\in \Cbb$, $t\in \Nbb$,   
\begin{equation} \label{eqn:discrete-fmu-1}
 \centering
  \begin{aligned}
     |f(\xbm) - \mu_{t-1}(\xbm)  | \leq \beta_t^{1/2} \sigma_{t-1}(\xbm), 
  \end{aligned}
\end{equation}
holds with probability $\geq 1-\delta$.
\end{lem}
Using union bound on $t\in\Nbb$ and $\xbm_t$, we extend Lemma~\ref{lem:fmu} to the following result on the sequence $\xbm_t$. 
    \begin{lem}[Lemma 5.5 in~\cite{srinivas2009gaussian}]\label{lem:fmu-t}
Given $\delta\in(0,1)$, let  $\beta_t = 2\log(\pi_t /\delta)$, where $\pi_t$ satisfies $\sum_{t=1}^{\infty} \frac{1}{\pi_t}=1,\pi_t>0$. Then,  for $\forall t\geq 1$,   
\begin{equation} \label{eqn:fmu-t-1}
 \centering
  \begin{aligned}
     |f(\xbm_{t}) - \mu_{t-1}(\xbm_{t})  | \leq \beta_t^{1/2} \sigma_{t-1}(\xbm_{t}), 
  \end{aligned}
\end{equation}
holds with probability $\geq 1-\delta$.
\end{lem}

We derive the bounds on $|I_{t-1}(\xbm)-EI_{t-1}(\xbm)|$ at given $t\in\Nbb$ and $\xbm$ next.
\begin{lem}\label{lem:fandnu}
For any given $\xbm\in C$, $t\in\Nbb$, and a scalar $w>0$, one can write that
\begin{equation} \label{eqn:fandmu-1}
 \centering
  \begin{aligned}
    I_{t-1}(\xbm) - EI_{t-1}(\xbm)  < \begin{cases} \sigma_{t-1}(\xbm) w, \  & \textnormal{ if }\xi^+_{t-1}- f(\xbm)\leq 0\\
                     -f(\xbm) +\mu_{t-1}(\xbm),  \ &\textnormal{ if } \xi^+_{t-1}-f(\xbm)>0.
         \end{cases}
  \end{aligned} 
\end{equation}
\end{lem}
The scalar $w$ in Lemma~\ref{lem:fandnu} can be viewed as a general constant and can replaced with any positive parameters such as $\beta$. 
\begin{proof}
      If $\xi^+_{t-1}-f(\xbm)\leq 0$, we have by deinition~\eqref{eqn:improvement} and Lemma~\ref{lem:EI}, 
\begin{equation} \label{eqn:fandmu-pf-1}
 \centering
  \begin{aligned}
     I_{t-1}(\xbm)-EI_{t-1}(\xbm) =  - EI_{t-1}(\xbm)  < 0< \sigma_{t-1}(\xbm) w. 
  \end{aligned}
\end{equation}
For $\xi^+_{t-1}- f(\xbm)>0$, we can write via Lemma~\ref{lem:EI}, 
\begin{equation} \label{eqn:fandmu-pf-2}
 \centering
  \begin{aligned}
   I_{t-1}(\xbm)-EI_{t-1}(\xbm) =& \xi^+_{t-1}- f(\xbm) - EI_{t-1}(\xbm)\leq  \xi^+_{t-1}-f(\xbm)-\xi^+_{t-1}+\mu_{t-1}(\xbm)  \\
   <& -f(\xbm)+\mu_{t-1}(\xbm). 
  \end{aligned}
\end{equation}
\end{proof}

A lower bound of $I_{t-1}(\xbm) - EI_{t-1}(\xbm)$ is given below.
\begin{lem}\label{lem:fnu-lower}
Let $w>1$, $c_1=0.4$ and $c_2=0.2$. For any given $\xbm\in C$, if $f(\xbm) - \mu_{t-1}(\xbm) < \sigma_{t-1}(\xbm) w$, then we can write 
\begin{equation} \label{eqn:fnu-lower-1}
 \centering
  \begin{aligned}
     I_{t-1}(\xbm) - EI_{t-1}(\xbm)   >  -\sigma_{t-1}(\xbm) w - \sigma_t(\xbm) c_1 e^{- c_2 w^2}. 
  \end{aligned}
\end{equation}
\end{lem}
\begin{proof}
    We consider two cases based on the sign of $EI_{t-1}(\xbm)-\sigma_{t-1}(\xbm)w$. 
First, $EI_{t-1}(\xbm)-\sigma_{t-1}(\xbm) w<0$. Given that $I_{t-1}(\xbm)\geq 0$, 
\begin{equation} \label{eqn:fnu-lower-pf-1}
 \centering
  \begin{aligned}
     I_{t-1}(\xbm) - EI_{t-1}(\xbm)   >  -\sigma_{t-1}(\xbm) w > -\sigma_{t-1}(\xbm) w- \sigma_{t-1}(\xbm) c_1 e^{- c_2 w^2},
  \end{aligned}
\end{equation}
  for any $c_1,c_2>0$.
  Second, $EI_{t-1}(\xbm)-\sigma_{t-1}(\xbm) w \geq 0$. We claim that $\xi^+_{t-1}-\mu_{t-1}(\xbm)\geq 0$ in this case. 
  Suppose on the contrary, $\xi^+_{t-1}-\mu_{t-1}(\xbm)<0$, then $z_{t-1}(\xbm)<0$. From~\eqref{eqn:EI-property-2}, 
\begin{equation} \label{eqn:fnu-lower-pf-2}
 \centering
  \begin{aligned}
      EI_{t-1} (\xbm)   \leq  \sigma_{t-1}(\xbm)\phi(z_{t-1}(\xbm))  < \sigma_{t-1}(\xbm) \phi(0) < \sigma_{t-1}(\xbm) w,
  \end{aligned}
\end{equation}
where the second inequality comes from the monotonicity of $\phi$ when $z_{t-1}(\xbm)\leq 0$ and the last inequality is due to $\phi(0)<1<w$. This contradicts the premise of the second case.
Therefore, $\xi^+_{t-1}-\mu_{t-1}(\xbm)\geq 0$ and  
\begin{equation} \label{eqn:fnu-lower-pf-3}
 \centering
  \begin{aligned}
     \frac{EI_{t-1}(\xbm)-\xi^+_{t-1}+\mu_{t-1}(\xbm)}{\sigma_{t-1}(\xbm)} =& z_{t-1}(\xbm) \Phi(z_{t-1}(\xbm)) + \phi(z_{t-1}(\xbm)) - z_{t-1}(\xbm) \\
          =& z_{t-1}(\xbm)[\Phi(z_{t-1}(\xbm))-1] + \phi(z_{t-1}(\xbm)) <\phi(z_{t-1}(\xbm)).
  \end{aligned}
\end{equation}
 Since $EI_{t-1}(\xbm)-\sigma_{t-1}(\xbm) w \geq 0$, 
\begin{equation} \label{eqn:fnu-lower-pf-4}
 \centering
  \begin{aligned}
    w\leq \frac{EI_{t-1}(\xbm)}{\sigma_{t-1}(\xbm)} = z_{t-1}(\xbm) \Phi(z_{t-1}(\xbm)) + \phi(z_{t-1}(\xbm))  < z_{t-1}(\xbm) + \phi(z_{t-1}(\xbm))\leq z_{t-1}(\xbm)+\phi(0).
  \end{aligned}
\end{equation}
  Using $w>1$ in~\eqref{eqn:fnu-lower-pf-4} , we know $z_{t-1}(\xbm) > 1-\phi(0)=0.601$. As $z_{t-1}(\xbm)$ increases, $\phi(z_t(\xbm))$ decreases. Thus, we have
\begin{equation} \label{eqn:fnu-lower-pf-5}
 \centering
  \begin{aligned}
    \frac{z_{t-1}(\xbm)}{\phi(z_{t-1}(\xbm))} > \frac{0.601}{\phi(0.601)} = 1.804.
  \end{aligned}
\end{equation}
  Applying~\eqref{eqn:fnu-lower-pf-5} to~\eqref{eqn:fnu-lower-pf-4}, we obtain
 \begin{equation} \label{eqn:fnu-lower-pf-6}
 \centering
  \begin{aligned}
     z_{t-1}(\xbm) > 0.643 w,  \ \phi(z_{t-1}(\xbm)) < \phi(0.643 w).
  \end{aligned}
\end{equation}
   Using~\eqref{eqn:fnu-lower-pf-6} in~\eqref{eqn:fnu-lower-pf-3}, we have
 \begin{equation} \label{eqn:fnu-lower-pf-7}
 \centering
  \begin{aligned}
     \frac{EI_{t-1}(\xbm)-\xi^+_{t-1}+\mu_{t-1}(\xbm)}{\sigma_{t-1}(\xbm)} < \phi(0.643 w) = \frac{1}{\sqrt{2\pi}}  e^{-\frac{1}{2}(0.643w)^2}< 0.4 e^{-0.2 w^2}.
  \end{aligned}
\end{equation}
Rearranging~\eqref{eqn:fnu-lower-pf-7} leads to 
 \begin{equation} \label{eqn:fnu-lower-pf-8}
 \centering
  \begin{aligned}
    -EI_{t-1}(\xbm) > -\xi^+_{t-1}+\mu_{t-1}(\xbm) - \sigma_{t-1}(\xbm) 0.4 e^{-0.2 w^2}.
  \end{aligned}
\end{equation}
  Therefore, from the condition of the lemma and~\eqref{eqn:fnu-lower-pf-8}, 
 \begin{equation} \label{eqn:fnu-lower-pf-9}
 \centering
  \begin{aligned}
   I_{t-1}(\xbm) -EI_{t-1}(\xbm) >& \xi^+_{t-1}-f(\xbm) -\xi^+_{t-1}+\mu_{t-1}(\xbm) - \sigma_{t-1}(\xbm) 0.4 e^{-0.2 w^2}\\
   >&-\sigma_t(\xbm) w-\sigma_t(\xbm) 0.4 e^{-0.2 w^2}.
  \end{aligned}
\end{equation}
 Let $c_1 = 0.4$ and $c_2 = 0.2$ and the proof is complete.
\end{proof}

From here on, we use the definition of events (Definition~\ref{def:eft}) in our analysis for ease of reference.
\begin{lem}\label{lem:IEI-bound}
 If $E^f(t)$ is true, then 
 \begin{equation} \label{eqn:IEI-bound-1}
  \centering
  \begin{aligned}
     |I_{t-1}(\xbm) - EI_{t-1}(\xbm)|\leq c_{\alpha}\beta_{t}^{1/2} \sigma_{t-1}(\xbm), \forall \xbm\in \Cbb_t, \forall t\in\Nbb,
  \end{aligned}
\end{equation}
  where $c_{\alpha}=1.328$. 
\end{lem}
\begin{proof}
Since $E^f(t)$ is true, $|f(\xbm) - \mu_{t-1}(\xbm)| \leq \beta_{t}^{1/2} \sigma_{t-1}(\xbm)$.
    Then, using $w=\beta_{t}^{1/2}$ and $t-1$ in Lemma~\ref{lem:fandnu}, we have 
 \begin{equation} \label{eqn:rkhs-bound-pf-1}
  \centering
  \begin{aligned}
     I_{t-1}(\xbm) - EI_{t-1}(\xbm) \leq \beta_{t}^{1/2} \sigma_{t-1}(\xbm) \leq c_{\alpha}\beta_{t}^{1/2} \sigma_{t-1}(\xbm) .
  \end{aligned}
\end{equation}
     Using $w=\beta_{t}^{1/2}$ in  Lemma~\ref{lem:fnu-lower} leads to 
  \begin{equation} \label{eqn:rkhs-bound-pf-2}
  \centering
  \begin{aligned}
     I_{t-1}(\xbm) - EI_{t-1}(\xbm) \geq - \beta_{t}^{1/2} \sigma_{t-1}(\xbm) - 0.4 e^{-0.2 \beta_{t} } \sigma_{t-1}(\xbm) .
  \end{aligned}
\end{equation}
   As $\beta_{t}^{1/2}$ increases, $e^{-0.2 \beta_{t} }$ is exponentially decreasing.
   Thus, $\frac{\beta_t^{1/2}}{0.4 e^{-0.2 \beta_{t} }}>\frac{1}{0.4e^{-0.2}} $ for $\beta_{t}> 1$. Hence, 
  \begin{equation} \label{eqn:rkhs-bound-pf-3}
  \centering
  \begin{aligned}
     I_{t-1}(\xbm) - EI_{t-1}(\xbm)  \geq -1.328 \beta_{t}^{1/2} \sigma_{t-1}(\xbm) = -c_{\alpha}\beta_{t}^{1/2} \sigma_{t-1}(\xbm).
  \end{aligned}
\end{equation}
\end{proof}
The constant parameter $c_{\alpha}>1$ obtained in the proof of Lemma~\ref{lem:IEI-bound} leads to a larger bound for $|I_{t-1}(\xbm)-EI_{t-1}(\xbm)|$ than $|f(\xbm)-\mu_{t-1}(\xbm)|$, since the former is derived from the latter.

\begin{lem}\label{lem:compact-bound}
 If $E^s(t)$ is true, then
\begin{equation} \label{eqn:compactbound-1}
 \centering
  \begin{aligned}
     |I_{t-1}(\ubm_{t}) - EI_{t-1}(\ubm_{t})  | \leq c_{\alpha}\beta_t^{1/2} \sigma_{t-1}(\ubm_{t}), 
  \end{aligned}
\end{equation}
  where $c_{\alpha}=1.328$.
\end{lem}
Lemma~\ref{lem:compact-bound} can be proved in a similar manner as Lemma~\ref{lem:IEI-bound}. 
Consider now $\xbm^*$ and its closest point to $\Cbb_t$, $[\xbm^*]_t$.
\begin{lem}\label{lem:compact-IL}
 If $E^f(t)$ is true, then for $\forall t\in\Nbb$, 
   \begin{equation} \label{eqn:compact-IL-1}
  \centering
  \begin{aligned}
          |I_{t-1}(\xbm^*) - EI_{t-1}([\xbm^*]_t)| \leq c_{\alpha}\beta_t^{1/2} \sigma_{t-1}([\xbm^*]_t) + \frac{1}{t^2}.
  \end{aligned}
  \end{equation}
\end{lem}
 \begin{proof}
  From Assumption~\ref{assp:gp} and Lemma~\ref{lem:lipschitz-cplus}, we know 
  \begin{equation} \label{eqn:compact-IL-pf-1}
  \centering
  \begin{aligned}
     | I_{t-1}(\xbm) - I_{t-1}(\xbm') | < L \norm{\xbm-\xbm'}_1.
  \end{aligned}
\end{equation}
   Applying the discretization $\Cbb_t$ and~\eqref{eqn:compactdisc-1} to~\eqref{eqn:compact-IL-pf-1}, for $\forall \xbm\in C$, we have 
    \begin{equation} \label{eqn:compact-IL-pf-2}
  \centering
  \begin{aligned}
     | I_{t-1}(\xbm) - I_{t-1}([\xbm]_t) | < L h_t.
  \end{aligned}
\end{equation}
   Since $L h_t  =\frac{1}{t^2}$ and $|\Cbb_t| = ( r d t^2 L)^d$,
   using~\eqref{eqn:compact-IL-pf-2} and Lemma~\ref{lem:IEI-bound}, we have  
   \begin{equation} \label{eqn:compact-IL-pf-4}
  \centering
  \begin{aligned}
          |I_{t-1}(\xbm^*) - EI_{t-1}([\xbm^*]_t)| \leq& |I_{t-1}(\xbm^*) - I_{t-1}([\xbm^*]_t)| + |I_{t-1}([\xbm^*]_t) - EI_{t-1}([\xbm^*]_t)|\\ 
        \leq& c_{\alpha}\beta_t^{1/2} \sigma_{t-1}([\xbm^*]_t) + \frac{1}{t^2}.
  \end{aligned}
\end{equation}
\end{proof}

\subsection{Global lower bound for posterior standard deviation}\label{se:proofofsigmabound}
In this section, we present another important lemma that provides a global lower bound of $\mathcal{O}\left(\frac{1}{\sqrt{t}}\right)$ on the posterior standard deviation $\sigma_t(\xbm)$ for $\forall \xbm\in C$.  This lower bound ensures a positive lower bound on $EI_{t-1}(\xbm_t)$, which combined with Lemma~\ref{lem:mu-bounded-EI} leads to the upper bound on $r_t$.
Such a lower bound was previously proved by invoking the fact that $\sigma_{t}(\xbm)$ is the smallest at $\xbm$ when all previous $t$ samples are at $\xbm$ (see for example~\cite{wang2014theoreybo,hu2025}). 
Here, we provide an alternate novel proof of this by first establishing 
 Lemma~\ref{lem:bound-h-sigma} and using techniques from numerical linear algebra~\citep{demmel1997applied}. 
This second proof approach is rigorous and applies to any kernel with a positive lower bound on $C$, \textit{e.g.}, SE and Matérn kernels. 
To facilitate Lemma~\ref{lem:bound-h-sigma}, we define part of the quantity in~\eqref{eqn:GP-post} by 
\begin{equation} \label{def:GP-h}
  \centering
  \begin{aligned}
   \hbm_t(\xbm)=(\Kbm_t+\sigma^2\Ibm)^{-1}\kbm_t(\xbm). 
  \end{aligned}
\end{equation}
The relationship between $\hbm_t$ and $\sigma_t(\xbm)$ satisfies the following lemma. 
\begin{lem}\label{lem:bound-h-sigma}
   The posterior standard deviation in~\eqref{eqn:GP-post} satisfies
\begin{equation} \label{eqn:bound-h-sigma-1}
  \centering
  \begin{aligned}
   \sigma_{t}(\xbm) \geq \sigma\norm{\hbm_t(\xbm)}.
  \end{aligned}
\end{equation}
   at given $\xbm\in C$ and $t\in\Nbb$.
\end{lem}
\begin{proof}
   We shall repeatedly use the fact that the covariance matrix $\Kbm_t$ is symmetric and positive semi-definite. Specifically, given $t$ samples and $\xbm\in C$, the covariance matrix of $\fbm_{1:t}$ and $f(\xbm)$  with zero noise 
\begin{equation} \label{eqn:bound-h-sigma-pf-1}
  \centering
  \begin{aligned}
    \begin{bmatrix*}
               \Kbm_t &\kbm_t\\
               \kbm_t^T &k(\xbm,\xbm)
      \end{bmatrix*}  
  \end{aligned}
\end{equation}
   is symmetric and positive semi-definite. Recall that by Assumption~\ref{assp:gp}, $k(\xbm,\xbm)=1$. On the other hand, the covariance matrix with a Gaussian noise $\mathcal{N}(0,\sigma^2)$ prior leads to  
\begin{equation} \label{eqn:bound-h-sigma-pf-2}
  \centering
  \begin{aligned}
    \begin{bmatrix}
               \Kbm_t +\sigma^2\Ibm & \kbm_t\\
               \kbm_t^T & 1
      \end{bmatrix},  
  \end{aligned}
\end{equation}
  which is positive-definite for $\forall \sigma>0$. Since $\Kbm_t$ is symmetric and positive-definite, we have the spectral decomposition~\citep{demmel1997applied} 
\begin{equation} \label{eqn:bound-h-sigma-pf-3}
  \centering
  \begin{aligned}
               \Kbm_t = \Qbm^T_t \Dbm_t^K \Qbm_t, 
  \end{aligned}
\end{equation}
  where $\Qbm_t$ is an orthogonal $t\times t$ matrix and $\Dbm_t^K$ is a diagonal matrix whose entries are the eigenvalues of $\Kbm_t$. 
  For simplicity, denote the eigenvalues of $\Kbm_t$ as $\lambda_i,i=1,\dots,t$. 
  Without losing generality, let $(\Dbm_t^K)_i = \lambda_i > 0$.

  Using the spectral decomposition~\eqref{eqn:bound-h-sigma-pf-3}, we have
\begin{equation} \label{eqn:bound-h-sigma-pf-4}
  \centering
  \begin{aligned}
   \Kbm_t+\sigma^2\Ibm= \Qbm^T_t \Dbm_t^{K1} \Qbm_t,  \ (\Kbm_t+\sigma^2\Ibm)^{-1}= \Qbm^T_t (\Dbm_t^{K1})^{-1} \Qbm_t, 
  \end{aligned}
  \end{equation}
where 
\begin{equation} \label{eqn:bound-h-sigma-pf-5}
  \centering
  \begin{aligned}
    \Dbm_t^{K1} =  (\Dbm_t^K+\sigma^2\Ibm), \ (\Dbm_t^{K1})_i = \lambda_i +\sigma^2,
  \end{aligned}
  \end{equation} 
  for $ i=1,\dots,t$. 
  Next, consider the matrix $(\sigma^2 (\Kbm_t+\sigma^2 I)^{-1}+\Ibm)^{-1}(\Kbm_t+\sigma^2 I)$.
Applying~\eqref{eqn:bound-h-sigma-pf-4}, we can write 
\begin{equation} \label{eqn:bound-h-sigma-pf-6}
  \centering
  \begin{aligned}
(\sigma^2 (\Kbm_t+&\sigma^2 I)^{-1}+\Ibm)^{-1}(\Kbm_t+\sigma^2 I)= (\sigma^2 \Qbm_t^T (\Dbm_t^{K1})^{-1}\Qbm_t +\Ibm)^{-1} \Qbm_t^T\Dbm^{K1}_t\Qbm_t\\
       =& (\Qbm^T_t \Dbm^{K2}_t \Qbm_t)^{-1} \Qbm^T_t\Dbm^{K1}_t\Qbm_t = \Qbm^T_t (\Dbm^{K2}_t)^{-1}\Dbm^{K1}_t\Qbm_t =\Qbm^T_t \Dbm^{K3}_t\Qbm_t, 
  \end{aligned}
\end{equation}
  where 
\begin{equation} \label{eqn:bound-h-sigma-pf-7}
  \centering
  \begin{aligned}
        \Dbm^{K2}_t = \sigma^2(\Dbm_t^{K1})^{-1}+\Ibm, \Dbm_t^{K3} = (\Dbm_t^{K2})^{-1}\Dbm_t^{K1}.
  \end{aligned}
\end{equation}
  Thus, $(\sigma^2 (\Kbm_t+\sigma^2 I)^{-1}+\Ibm)^{-1}(\Kbm_t+\sigma^2 I)$ is symmetric and positive-definite.
   The components of the $\Dbm_t^{K2}$ and $\Dbm_t^{K3}$ are
\begin{equation} \label{eqn:bound-h-sigma-pf-8}
  \centering
  \begin{aligned}
   (D_t^{K2})_i = \frac{\sigma^2}{\lambda_i+\sigma^2}+1, \ (D_t^{K3})_i = \frac{(\lambda_i+\sigma^2)^2}{\lambda_i+2\sigma^2}.
  \end{aligned}
\end{equation}
Thus,  
\begin{equation} \label{eqn:bound-h-sigma-pf-9}
  \centering
  \begin{aligned}
  (D_t^{K3})_i = \frac{ \lambda_i^2+2\lambda_i\sigma^2+\sigma^4}{\lambda_i+2\sigma^2}=\lambda_i+\frac{\sigma^4}{\lambda_i+2\sigma^2}> \lambda_i.
  \end{aligned}
\end{equation}
 From~\eqref{eqn:bound-h-sigma-pf-3},~\eqref{eqn:bound-h-sigma-pf-6} and~\eqref{eqn:bound-h-sigma-pf-9},  for any vector $ \vbm \in\Rbb^t$ and $\vbm\neq \zerobold$
\begin{equation} \label{eqn:bound-h-sigma-pf-10}
  \centering
  \begin{aligned}
          \vbm^T \Kbm_t \vbm < \vbm^T (\sigma^2 (\Kbm_t+\sigma^2 I)^{-1}+\Ibm)^{-1}(\Kbm_t+\sigma^2 I)\vbm.
  \end{aligned}
\end{equation}
  Since~\eqref{eqn:bound-h-sigma-pf-1} is symmetric and positive semi-definite, we can construct vector $[\vbm,u]\in \Rbb^{t+1}$, for $\forall \vbm\in\Rbb^t$, $u\in\Rbb$ and $u\vbm\neq \zerobold$. Then, 
\begin{equation} \label{eqn:bound-h-sigma-pf-11}
  \centering
  \begin{aligned}
  \begin{bmatrix*}
               \vbm\\
               u
      \end{bmatrix*}^T 
    \begin{bmatrix*}
               \Kbm_t &\kbm_t\\
               \kbm_t^T &1
      \end{bmatrix*} 
 \begin{bmatrix*}
               \vbm\\
               u 
      \end{bmatrix*}= \vbm^T \Kbm_t\vbm + 2u\vbm^T\kbm+u^2 \geq 0.
  \end{aligned}
\end{equation}
Replacing $\Kbm_t$ in~\eqref{eqn:bound-h-sigma-pf-11}, by~\eqref{eqn:bound-h-sigma-pf-10} 
\begin{equation} \label{eqn:bound-h-sigma-pf-12}
  \centering
  \begin{aligned}
  &\begin{bmatrix*}
               \vbm\\
               u
      \end{bmatrix*}^T 
    \begin{bmatrix*}
              (\sigma^2 (\Kbm_t+\sigma^2 I)^{-1}+\Ibm)^{-1}(\Kbm_t+\sigma^2 I) &\kbm_t\\
               \kbm_t^T & 1
      \end{bmatrix*} 
 \begin{bmatrix*}
               \vbm\\
               u 
      \end{bmatrix*} \\
       &= \vbm^T (\sigma^2 (\Kbm_t+\sigma^2 I)^{-1}+\Ibm)^{-1}(\Kbm_t+\sigma^2 I) \vbm + 2u\vbm^T\kbm+u^2\\ 
       &> \vbm^T  \Kbm_t\vbm + 2u\vbm^T\kbm+u^2 \geq 0.
  \end{aligned}
\end{equation}
Therefore, the matrix in~\eqref{eqn:bound-h-sigma-pf-12} is symmetric and positive-definite, and thus its Schur complement is positive, \textit{i.e.},  
\begin{equation} \label{eqn:bound-h-sigma-pf-13}
  \centering
  \begin{aligned}
    1- \kbm_t^T (\Kbm_t+\sigma^2 I)^{-1} (\sigma^2 (\Kbm_t+\sigma^2 I)^{-1}+\Ibm)\kbm_t> 0.
  \end{aligned}
\end{equation}
Simple rearrangement of~\eqref{eqn:bound-h-sigma-pf-13} leads to 
\begin{equation} \label{eqn:bound-h-sigma-pf-14}
  \centering
  \begin{aligned}
     \kbm_t^T (\Kbm_t+\sigma^2 I)^{-1}\kbm_t + \kbm_t^T \sigma^2 (\Kbm_t+\sigma^2 I)^{-1} (\Kbm_t+\sigma^2 I)^{-1}\kbm_t < 1.
  \end{aligned}
\end{equation}
Or equivalently, 
\begin{equation} \label{eqn:bound-h-sigma-pf-15}
  \centering
  \begin{aligned}
      \norm{\hbm_t(\xbm)} \leq \frac{1}{\sigma} \sqrt{1- \kbm_t^T (\Kbm_t+\sigma^2 I)^{-1}\kbm_t}=\frac{1}{\sigma}\sigma_t(\xbm).
  \end{aligned}
\end{equation}
\end{proof}
The aforementioned global lower bound of $\sigma_t(\xbm)$ is given below.
\begin{lem}\label{thm:gp-sigma-bound}
    There exists $c_{\sigma}>0$ such that for $\forall \xbm\in C$ the posterior standard deviation has the lower bound  
 \begin{equation} \label{eqn:gp-sigma-bound-1}
 \centering
  \begin{aligned}
     \sigma_t(\xbm) \geq c_{\sigma} \sigma \sqrt{\frac{1}{t+\sigma^2}}.
   \end{aligned}
\end{equation}
  If there exists a lower bound for $k^2(\cdot,\cdot)$ on $C$, \textit{i.e.}, $k^2(\xbm,\xbm')\geq c_k$ for $\forall \xbm,\xbm'\in C$, then $c_{\sigma}=\frac{c_k}{1+\sigma^2}$. 
\end{lem}
We now present the two proofs of Lemma~\ref{thm:gp-sigma-bound}. 
\begin{proof}
   We start with the first proof. 
 We invoke the fact that the minimum posterior standard deviation at $\xbm$ is obtained if the previous $t$ samples are all $\xbm$. Suppose all $t$ samples are $\xbm$. Then, all entries of $\kbm_t$ and $\Kbm_t$ are ones. It is easy to verify that  
   \begin{equation} \label{eqn:gp-sigma-bound-pf-1}
 \centering
  \begin{aligned}
      (\Kbm_t+\sigma^2\Ibm)^{-1} = -\frac{1}{(t+\sigma^2)\sigma^2}\Pbm +\frac{1}{\sigma^2} \Ibm,
  \end{aligned}
\end{equation}
where $\Pbm$ is a $t\times t$ matrix with all entries $1$.
Thus, by~\eqref{eqn:GP-post}, we have 
   \begin{equation} \label{eqn:gp-sigma-bound-pf-2}
 \centering
  \begin{aligned}
    \sigma_t^2(\xbm) \geq  1- \pbm^T  \left[-\frac{1}{(t+\sigma^2)\sigma^2}\Pbm +\frac{1}{\sigma^2} \Ibm\right] \pbm = \frac{\sigma^2}{t+\sigma^2},
  \end{aligned}
\end{equation}
   where $\pbm$ is the $t$-dimensional vector with all $1$ entries. The lemma is proven with $c_{\sigma}=1$.
   
   We now present the second and novel proof. 
   From the Gershgorin theorem~\citep{gershgorin1931uber}, we have the eigenvalues $\lambda_i,i=1,\dots,t$ of $\Kbm_t$ satisfies $\max_i\{\lambda_i\}\leq t \max_{i,j=1,\dots,t}k(\xbm_i,\xbm_j)\leq t$.
   Thus, the eigenvalues of the symmetric positive definite matrix $(\Kbm_t+\sigma^2\Ibm)^{-1}$, denoted as $\tilde{\lambda}_i,i=1,\dots,n$, satisfies
  \begin{equation} \label{eqn:gp-sigma-bound-pf-3}
 \centering
  \begin{aligned}
     \tilde{\lambda_i} \geq \frac{1}{t+\sigma^2}.
  \end{aligned}
\end{equation}
     Therefore, we can write 
   \begin{equation} \label{eqn:gp-sigma-bound-pf-4}
 \centering
  \begin{aligned}
     \hbm_t^T(\xbm) \hbm(\xbm) =& \kbm_t^T(\xbm) (\Kbm_t+\sigma^2\Ibm)^{-1} (\Kbm_t+\sigma^2\Ibm)^{-1}\kbm_t(\xbm) \geq \left(\frac{1}{t+\sigma^2}\right)^2 \norm{\kbm_t(\xbm)}^2 \\
       \geq&  \left(\frac{1}{t+\sigma^2}\right)^2  t \min_{i=1,\dots,t} k^2(\xbm,\xbm_i).
  \end{aligned}
\end{equation}
    From the conditions of the kernel, there exists $c_k$ such that
    \begin{equation} \label{eqn:gp-sigma-bound-pf-5}
 \centering
  \begin{aligned}
           k^2(\xbm,\xbm') \geq  c_k,
   \end{aligned}
 \end{equation}
   for $\forall \xbm,\xbm'\in C$. Using~\eqref{eqn:gp-sigma-bound-pf-5} in~\eqref{eqn:gp-sigma-bound-pf-4}, we have 
   \begin{equation} \label{eqn:gp-sigma-bound-pf-6}
 \centering
  \begin{aligned}
     \norm{\hbm_t(\xbm)} >  \sqrt{c_k} \frac{\sqrt{t}}{t+\sigma^2} > \sqrt{c_k} \frac{1}{\sqrt{t+\sigma^2}} \frac{1}{\sqrt{1+\sigma^2}}.
  \end{aligned}
\end{equation}
   Thus, by Lemma~\ref{lem:bound-h-sigma},  
   \begin{equation} \label{eqn:gp-sigma-bound-pf-7}
 \centering
  \begin{aligned}
    \sigma_t(\xbm) >  c_{\sigma} \frac{\sigma}{\sqrt{t+\sigma^2}},
  \end{aligned}
\end{equation}
   where $c_{\sigma}=\frac{1}{\sqrt{1+\sigma^2}}\sqrt{c_k}$.
\end{proof}
For simplicity, we take $c_{\sigma} = 1$ in the following proof.
%The following lemma is on the cumulative probability distribution of the improvement function under the GP prior assumption.
%\begin{lemma}\label{lem:Icdf}
% Under Assumption~\ref{assp:gp}, the cumulative probability distribution of $I_t$ satisfies 
%\begin{equation} \label{eqn:Icdf-1}
% \centering
%  \begin{aligned}
%     \Pbb\{I_t(\xbm) \leq a\} = \begin{cases}
%                    0,   \ &a<0,\\
%                  \Phi\left(\frac{a}{\sigma_{t}(\xbm)}-z_{t}(\xbm)\right), \  &a\geq 0.
%                  \end{cases}
%  \end{aligned}
%\end{equation}
%\end{lemma}
%\begin{proof}
%   Under Assumption~\ref{assp:gp}, at a given $t$, $f(\xbm)\sim\mathcal{N}(\mu_{t}(\xbm),\sigma_{t}(\xbm))$. Since $I_t(\xbm)\geq 0$ for all $\xbm$,~\eqref{eqn:Icdf-1} follows immediately if $a<0$.
%   For $a\geq 0$, $I_t(\xbm) \leq a \Leftrightarrow \xi^+_{t}-f(\xbm)\leq a$. Thus, using basic properties of the CDF of the standard normal distribution, 
%\begin{equation} \label{eqn:Icdf-pf-1} 
% \centering
%  \begin{aligned}
%   \Pbb\{I_t(\xbm) \leq a\} = \Pbb\{\xi^+_{t}-f(\xbm) \leq a\} = 1- \Pbb\{f(\xbm) \leq \xi^+_{t}-a\}.
%  \end{aligned}
%\end{equation}
%  Then, from~\eqref{eqn:EI-z},
%\begin{equation} \label{eqn:Icdf-pf-2} 
% \centering
%  \begin{aligned}
%   1-\Pbb\{f(\xbm) \leq \xi^+_{t}-a\}=1-\Phi\left(\frac{\xi_t^+-a-\mu_t(\xbm)}{\sigma_t(\xbm)} \right)=\Phi\left(\frac{a-\xi_t^++\mu_t(\xbm)}{\sigma_t(\xbm)} \right).
%  \end{aligned}
%\end{equation}
%\end{proof}

%% file: Sections/Appx-instant.tex
\section{Instantaneous regret of Section~\ref{se:inst-regret}}\label{se:inst-regret-proof}
\subsection{Discretization}\label{se:discretization}
We define the time-dependent discretizations $\Cbb_t$ of $C$, a commonly used technique for BO analysis in the Bayesian setting~\citep{srinivas2009gaussian,chowdhury2017kernelized}.
For any $t\in\Nbb$, we take $\Cbb_t$ to be a set of finite number of points in $C$.
The discretization set $\Cbb_t$ aims to cover the compact set $C$ with just enough points so that for any $\xbm\in C$ and $t\in\Nbb$, the distance between $\xbm\in C$ and $\Cbb_t$ is small enough and controlled.  As earlier, denote the closest point of $\xbm$ to $\Cbb_t$ as 
\begin{equation} \label{eqn:close-point}
 \centering
  \begin{aligned}
 	[\xbm]_t := \underset{\substack{\wbm}\in \Cbb_t}{\text{minimize}} 
	   \norm{\xbm-\wbm} ,
  \end{aligned}
\end{equation}
and define the parameter $h_t$ as a measure of the distance between $\Cbb$ and $\Cbb_t$ by
\begin{equation} \label{eqn:compactdisc-1}
 \centering
  \begin{aligned}
     \norm{\xbm-[\xbm]_t}_1 \leq h_t, \ \forall \xbm\in C,
  \end{aligned}
\end{equation}
  where $\norm{\cdot}_1$ is the one-norm. 
In this paper, we adopt the discretizations with evenly spaced points and decreasing distances as in~\cite{srinivas2009gaussian}.
By~\eqref{eqn:compactdisc-1}, $\Cbb_t$ that corresponds to $h_t$ is constructed with $\frac{rd}{h_t}$ uniformly space points. 
Given Assumption~\ref{assp:constraint} and $C \subseteq [0,r]^d$, this leads to a choice of 
   \begin{equation} \label{eqn:disc-size-1}
  \centering
  \begin{aligned}
     |\Cbb_t| = \left( \frac{rd}{h_t}\right)^d,
   \end{aligned}
  \end{equation}
where the points in $\Cbb_t$ are evenly spaced. 
We emphasize that the discretization does not play a role in the GP-EI algorithm. Rather, it is an analytic tool to help determine $\beta_t$.
Moreover, we adopt the common choice of $h_t = \frac{1}{L t^{2}}$ such that by Assumptions~\ref{assp:gp} and~\eqref{eqn:compactdisc-1}
   \begin{equation} \label{eqn:ht-1}
  \centering
  \begin{aligned}
     |f(\xbm) - f([\xbm]_t)| \leq \frac{1}{t^2}.
   \end{aligned}
  \end{equation}
From~\eqref{eqn:disc-size-1}, we have $|\Cbb_t| = (Lrdt^2)^d$.

\subsection{Proof of instantaneous regret bound of three incumbents}\label{se:appx-proof-incumbent}

The proof of Lemma~\ref{lem:post-mean-instregret}, the instantaneous regret of BPMI, is given below.
\begin{proof}
  We consider two cases based on the sign of $\mu^+_{t-1}-f(\xbm_{t})$.
  First, consider $\mu^+_{t-1}-f(\xbm_{t})> 0$.
  Since $\mu^+_{t-1}-\mu_{t-1}(\xbm_{t})\leq 0$, we have 
  \begin{equation} \label{eqn:post-mean-instregret-pf-2}
  \centering
  \begin{aligned}
   EI_{t-1}(\xbm_{t}) =& (\mu^+_{t-1}-\mu_{t-1}(\xbm_{t}))\Phi(z_{t-1}(\xbm_{t}))+\sigma_{t-1}(\xbm_{t})\phi(z_{t-1}(\xbm_{t}))
          \leq \phi(0)\sigma_{t-1}(\xbm_{t}).\\ 
  \end{aligned}
  \end{equation}
 Then, we can write 
  \begin{equation} \label{eqn:post-mean-instregret-pf-6}
  \centering
  \begin{aligned}
        r_t =& f(\xbm_{t})-f(\xbm^*) =  f(\xbm_{t})-\mu^+_{t-1}+ \mu^+_{t-1}-f(\xbm^*)\leq f(\xbm_{t})-\mu^+_{t-1}+ I_t(\xbm^*)\\
              \leq& f(\xbm_{t})-\mu^+_{t-1}+EI_{t-1}([\xbm^*]_{t}) + c_{\alpha}\beta^{1/2}_t\sigma_{t-1}([\xbm^*]_t)+\frac{1}{t^2}\\
              \leq& EI_{t-1}(\xbm_{t}) + c_{\alpha}\beta^{1/2}_t\sigma_{t-1}([\xbm^*]_t)+\frac{1}{t^2}\\
               \leq&  \phi(0)\sigma_{t-1}(\xbm_{t})+ c_{\alpha}\beta^{1/2}_t\sigma_{t-1}([\xbm^*]_t)+\frac{1}{t^2}. 
  \end{aligned}
  \end{equation}
  The second inequality in~\eqref{eqn:post-mean-instregret-pf-6} is by Lemma~\ref{lem:compact-IL}, while the last inequality is due to~\eqref{eqn:post-mean-instregret-pf-2}. 
  In the second case, we consider $\mu^+_{t-1}-f(\xbm_{t})\leq 0$. From Lemma~\ref{lem:compact-IL} and $E^s(t)$,
  \begin{equation} \label{eqn:post-mean-instregret-pf-1}
  \centering
  \begin{aligned}
          r_t  =& f(\xbm_t) - f(\xbm^*)  =  f(\xbm_t) -\mu_{t-1}^{+}+\mu_{t-1}^{+}- f(\xbm^*) \leq f(\xbm_t) -\mu_{t-1}^{+}+ I_{t-1}(\xbm^*)\\
          \leq& f(\xbm_t)-\mu_{t-1}(\xbm_t)+\mu_{t-1}(\xbm_t) -\mu_{t-1}^{+} + EI_{t-1}([\xbm^*]_t)  + c_{\alpha}\beta_t^{1/2} \sigma_{t-1}([\xbm^*]_t)+\frac{1}{t^2}\\ 
          \leq& \mu_{t-1}(\xbm_t) -\mu_{t-1}^{+} + EI_{t-1}([\xbm^*]_t)  +\beta_t^{1/2}\sigma_{t-1}(\xbm_{t})+ c_{\alpha}\beta_t^{1/2} \sigma_{t-1}([\xbm^*]_t)+\frac{1}{t^2}\\ 
          \leq&  \mu_{t-1}(\xbm_t) -\mu_{t-1}^{+}+EI_{t-1}(\xbm_{t}) +\beta_t^{1/2}\sigma_{t-1}(\xbm_{t}) + c_{\alpha}\beta_t^{1/2} \sigma_{t-1}([\xbm^*]_t)+\frac{1}{t^2},\\ 
  \end{aligned}
  \end{equation}
  where the first inequality is by definition of $I_{t-1}$, the second inequality uses Lemma~\ref{lem:compact-IL}, the third inequality uses~\eqref{eqn:eft-2} and the last inequality uses the definition of $\xbm_t$.
  Next, we bound the first term on the last line of~\eqref{eqn:post-mean-instregret-pf-1}.
  Let $\xbm_t^+$ be the point such that $\mu_{t-1}(\xbm_t^+) = \mu^+_t$. Then, we know 
  \begin{equation} \label{eqn:post-mean-instregret-pf-3}
  \centering
  \begin{aligned}
   EI_{t-1}(\xbm_{t}) \geq& EI_{t-1}(\xbm_{t}^+) = (\mu^+_{t-1}-\mu_{t-1}(\xbm_{t}^+))\Phi(z_{t-1}(\xbm_{t}^+))+\sigma_{t-1}(\xbm_{t}^+)\phi(z_{t-1}(\xbm_{t}^+))\\
          =& \phi(0) \sigma_{t-1}(\xbm_{t}^+)\geq \phi(0)c_{\sigma}\sqrt{\frac{\sigma^2}{t-1+\sigma^2}},\\ 
  \end{aligned}
  \end{equation}
  where the last inequality is from Lemma~\ref{thm:gp-sigma-bound}.
  Since $c_{\sigma} =1$, $\phi(0)c_{\sigma}\sqrt{\frac{\sigma^2}{t-1+\sigma^2}}\leq \phi(0)=\frac{1}{\sqrt{2\pi}}$, which satisfies the condition of $\kappa_t$ in Lemma~\ref{lem:mu-bounded-EI}. 
  Then, by Lemma~\ref{lem:mu-bounded-EI},~\eqref{eqn:post-mean-instregret-pf-3} implies 
  \begin{equation} \label{eqn:post-mean-instregret-pf-4}
  \centering
  \begin{aligned}
          \mu^+_{t-1}-\mu_{t-1}(\xbm_t) \geq - \log^{\frac{1}{2}}\left(\frac{t-1+\sigma^2}{2\pi \phi(0)^2c_{\sigma}^2 \sigma^2}  \right) \sigma_{t-1}(\xbm_t).
  \end{aligned}
  \end{equation}
  Using~\eqref{eqn:post-mean-instregret-pf-2} and~\eqref{eqn:post-mean-instregret-pf-4} in~\eqref{eqn:post-mean-instregret-pf-1}, we have 
  \begin{equation} \label{eqn:post-mean-instregret-pf-5}
  \centering
  \begin{aligned}
          r_t  
          \leq&  c_{\mu}(t)\sigma_{t-1}(\xbm_t) +\phi(0) \sigma_{t-1}(\xbm_{t}) +\beta_t^{1/2}\sigma_{t-1}(\xbm_{t}) + c_{\alpha}\beta_t^{1/2} \sigma_{t-1}([\xbm^*]_t)+\frac{1}{t^2},\\ 
          =&  (c_{\mu}(t)  +\phi(0)  +\beta_t^{1/2})\sigma_{t-1}(\xbm_{t}) + c_{\alpha}\beta_t^{1/2} \sigma_{t-1}([\xbm^*]_t)+\frac{1}{t^2},\\ 
  \end{aligned}
  \end{equation}
  where $c_{\mu}(t) = \log^{\frac{1}{2}}\left(\frac{t-1+\sigma^2}{2\pi \phi^2(0)c_{\sigma}^2 \sigma^2}  \right)$.
Combine~\eqref{eqn:post-mean-instregret-pf-5} and~\eqref{eqn:post-mean-instregret-pf-6} and the proof is complete.
\end{proof}

The proof of Lemma~\ref{lem:sampled-post-mean-instregret}, the instantaneous regret of BSPMI, is given below.
\begin{proof}
  We consider two cases based on the sign of $\mu^m_{t-1}-f(\xbm_{t})$.
First, if $\mu^m_{t-1}-f(\xbm_{t})> 0$, we can write 
  \begin{equation} \label{eqn:sampled-post-mean-instregret-pf-6}
  \centering
  \begin{aligned}
        r_t =& f(\xbm_{t})-f(\xbm^*) =  f(\xbm_{t})-\mu^m_{t-1}+ \mu^m_{t-1}-f(\xbm^*)= f(\xbm_{t})-\mu^m_{t-1}+ I_t(\xbm^*)\\
              \leq& f(\xbm_{t})-\mu^m_{t-1}+EI_{t-1}([\xbm^*]_{t}) + c_{\alpha}\beta^{1/2}_t\sigma_{t-1}([\xbm^*]_t)+\frac{1}{t^2}\\
              \leq& f(\xbm_{t})-\mu^m_{t-1}+EI_{t-1}(\xbm_{t}) + c_{\alpha}\beta^{1/2}_t\sigma_{t-1}([\xbm^*]_t)+\frac{1}{t^2},\\
  \end{aligned}
  \end{equation}
  where the second inequality in~\eqref{eqn:sampled-post-mean-instregret-pf-6} is by Lemma~\ref{lem:compact-IL}.
   Further, we can write
\begin{equation} \label{eqn:sampled-post-mean-instregret-pf-6.5}
  \centering
  \begin{aligned}
   EI_{t-1}(\xbm_{t}) =& (\mu^m_{t-1}-\mu_{t-1}(\xbm_{t}))\Phi(z_{t-1}(\xbm_{t}))+\sigma_{t-1}(\xbm_{t})\phi(z_{t-1}(\xbm_{t}))\\
    =& (\mu^m_{t-1}-f(\xbm_t)+f(\xbm_t)-\mu_{t-1}(\xbm_{t}))\Phi(z_{t-1}(\xbm_{t}))+\sigma_{t-1}(\xbm_{t})\phi(z_{t-1}(\xbm_{t}))\\ 
    \leq& \mu^m_{t-1}-f(\xbm_t)+\beta_t^{1/2}\sigma_{t-1}(\xbm_{t})+ \sigma_{t-1}(\xbm_{t})\phi(0),\\ 
  \end{aligned}
  \end{equation}
  where the inequality is due to~\eqref{eqn:eft-2}, $\mu^m_{t-1}-f(\xbm_t)>0$, $\Phi(\cdot)\in(0,1)$, and $\phi(\cdot)<\phi(0)$. Applying~\eqref{eqn:sampled-post-mean-instregret-pf-6.5} to~\eqref{eqn:sampled-post-mean-instregret-pf-6}, we have  
 \begin{equation} \label{eqn:sampled-post-mean-instregret-pf-7}
  \centering
  \begin{aligned}
        r_t 
               \leq&  (\beta_t^{1/2}+\phi(0))\sigma_{t-1}(\xbm_{t})+ c_{\alpha}\beta^{1/2}_t\sigma_{t-1}([\xbm^*]_t)+\frac{1}{t^2}. 
  \end{aligned}
  \end{equation}
  
  The second case is when $\mu^m_{t-1}-f(\xbm_{t})\leq 0$. From Lemma~\ref{lem:compact-IL} and $E^s_1(t)$, we have 
  \begin{equation} \label{eqn:sampled-post-mean-instregret-pf-1}
  \centering
  \begin{aligned}
          r_t  =& f(\xbm_t) - f(\xbm^*)  =  f(\xbm_t) -\mu_{t-1}^{m}+\mu_{t-1}^{m}- f(\xbm^*) \leq f(\xbm_t) -\mu_{t-1}^{m}+ I_{t-1}(\xbm^*)\\
          \leq& f(\xbm_t)-\mu_{t-1}(\xbm_t)+\mu_{t-1}(\xbm_t) -\mu_{t-1}^{m} + EI_{t-1}([\xbm^*]_t)  + c_{\alpha}\beta_t^{1/2} \sigma_{t-1}([\xbm^*]_t)+\frac{1}{t^2}\\ 
          \leq& \mu_{t-1}(\xbm_t) -\mu_{t-1}^{m} + EI_{t-1}([\xbm^*]_t)  +\beta_t^{1/2}\sigma_{t-1}(\xbm_{t})+ c_{\alpha}\beta_t^{1/2} \sigma_{t-1}([\xbm^*]_t)+\frac{1}{t^2}\\ 
          \leq&  \mu_{t-1}(\xbm_t) -\mu_{t-1}^{m}+EI_{t-1}(\xbm_{t}) +\beta_t^{1/2}\sigma_{t-1}(\xbm_{t}) + c_{\alpha}\beta_t^{1/2} \sigma_{t-1}([\xbm^*]_t)+\frac{1}{t^2},\\ 
  \end{aligned}
  \end{equation}
  where the first inequality is by definition of $I_{t-1}$, the second inequality uses Lemma~\ref{lem:compact-IL}, the third inequality uses~\eqref{eqn:eft-2}, and the last inequality uses the definition of $\xbm_t$.
  
  We now bound the first and second terms on the last line in~\eqref{eqn:sampled-post-mean-instregret-pf-1}.
  Using~\eqref{eqn:eft-2}, $\phi(\cdot)\leq\phi(0)$, and $\mu_{t-1}^m\leq f(\xbm_t)$, we have 
  \begin{equation} \label{eqn:sampled-post-mean-instregret-pf-2}
  \centering
  \begin{aligned}
   EI_{t-1}(\xbm_{t}) =& (\mu^m_{t-1}-\mu_{t-1}(\xbm_{t}))\Phi(z_{t-1}(\xbm_{t}))+\sigma_{t-1}(\xbm_{t})\phi(z_{t-1}(\xbm_{t}))\\
    =& (\mu^m_{t-1}-f(\xbm_t)+f(\xbm_t)-\mu_{t-1}(\xbm_{t}))\Phi(z_{t-1}(\xbm_{t}))+\sigma_{t-1}(\xbm_{t})\phi(z_{t-1}(\xbm_{t}))\\ 
    \leq& \beta_t^{1/2}\sigma_{t-1}(\xbm_{t})+\sigma_{t-1}(\xbm_{t})\phi(0).\\ 
  \end{aligned}
  \end{equation}
  Further, denote $\xbm_t^+$ as the point such that $\mu_{t-1}(\xbm_t^+) = \mu^m_{t-1}$. Then, we know 
  \begin{equation} \label{eqn:sampled-post-mean-instregret-pf-3}
  \centering
  \begin{aligned}
   EI_{t-1}(\xbm_{t}) \geq& EI_{t-1}(\xbm_{t}^+) = (\mu^m_{t-1}-\mu_{t-1}(\xbm_{t}^+))\Phi(z_{t-1}(\xbm_{t}^+))+\sigma_{t-1}(\xbm_{t}^+)\phi(z_{t-1}(\xbm_{t}^+))\\
          =& \phi(0)\sigma_{t-1}(\xbm_{t}^+)\geq \phi(0)c_{\sigma}\sqrt{\frac{\sigma^2}{t-1+\sigma^2}},\\ 
  \end{aligned}
  \end{equation}
  where the last inequality is from Lemma~\ref{thm:gp-sigma-bound}.
  By Lemma~\ref{lem:mu-bounded-EI},~\eqref{eqn:sampled-post-mean-instregret-pf-3} implies 
  \begin{equation} \label{eqn:sampled-post-mean-instregret-pf-4}
  \centering
  \begin{aligned}
          \mu^m_{t-1}-\mu_{t-1}(\xbm_t) \geq - \log^{\frac{1}{2}}\left(\frac{t-1+\sigma^2}{2\pi \phi^2(0)c_{\sigma}^2 \sigma^2}  \right) \sigma_{t-1}(\xbm_t).
  \end{aligned}
  \end{equation}
  Using~\eqref{eqn:sampled-post-mean-instregret-pf-2} and~\eqref{eqn:sampled-post-mean-instregret-pf-4} in~\eqref{eqn:sampled-post-mean-instregret-pf-1}, we have 
  \begin{equation} \label{eqn:sampled-post-mean-instregret-pf-5}
  \centering
  \begin{aligned}
          r_t  
          \leq&  c_{\mu}(t)\sigma_{t-1}(\xbm_t) +\phi(0) \sigma_{t-1}(\xbm_{t}) +2\beta_t^{1/2}\sigma_{t-1}(\xbm_{t}) + c_{\alpha}\beta_t^{1/2} \sigma_{t-1}([\xbm^*]_t)+\frac{1}{t^2},\\ 
          =&  (c_{\mu}(t)  +\phi(0)  +2\beta_t^{1/2})\sigma_{t-1}(\xbm_{t}) + c_{\alpha}\beta_t^{1/2} \sigma_{t-1}([\xbm^*]_t)+\frac{1}{t^2},\\ 
  \end{aligned}
  \end{equation}
  where $c_{\mu}(t) = \log^{\frac{1}{2}}\left(\frac{t-1+\sigma^2}{2\pi \phi^2(0)c_{\sigma}^2 \sigma^2}  \right)$.

 Combine~\eqref{eqn:sampled-post-mean-instregret-pf-5} and~\eqref{eqn:sampled-post-mean-instregret-pf-7} and the proof is complete.
\end{proof}

The proof of Lemma~\ref{lem:best-observation-instregret}, the instantaneous regret of BOI, is given next.
\begin{proof}
  We consider two cases based on the value of $y^+_{t-1}-f(\xbm_t)$. 
   First, if $y^+_{t-1}-f(\xbm_{t})> 0$, we have  
   \begin{equation} \label{eqn:best-observation-instregret-pf-14}
  \centering
  \begin{aligned}
        r_t =& f(\xbm_{t})-f(\xbm^*) =  f(\xbm_{t})-y^+_{t-1}+ y^+_{t-1}-f(\xbm^*)\\
              \leq& f(\xbm_{t})-y^+_{t-1}+EI_{t-1}([\xbm^*]_{t}) + c_{\alpha}\beta^{1/2}_t\sigma_{t-1}([\xbm^*]_t)+\frac{1}{t^2}\\
              \leq& f(\xbm_{t})-y^+_{t-1}+EI_{t-1}(\xbm_{t}) + c_{\alpha}\beta^{1/2}_t\sigma_{t-1}([\xbm^*]_t)+\frac{1}{t^2}.\\
  \end{aligned}
  \end{equation}
  The first inequality in~\eqref{eqn:best-observation-instregret-pf-14} is by Lemma~\ref{lem:compact-IL}.
  Further, we can write 
  \begin{equation} \label{eqn:best-observation-instregret-pf-15}
 \centering
 \begin{aligned}
          EI_{t-1}(\xbm_t) 
          =& (y^+_{t-1}-\mu_{t-1}(\xbm_{t}))\Phi(z_{t-1}(\xbm_{t}))+\sigma_{t-1}(\xbm_{t})\phi(z_{t-1}(\xbm_{t}))\\ 
          \leq& (y^+_{t-1}-f(\xbm_{t})+f(\xbm_{t})-\mu_{t-1}(\xbm_{t}))\Phi(z_{t-1}(\xbm_{t}))+\phi(0)\sigma_{t-1}(\xbm_{t})\\ 
          \leq& y^+_{t-1}-f(\xbm_{t})+\beta_t^{1/2}\sigma_{t-1}(\xbm_{t})+\phi(0)\sigma_{t-1}(\xbm_{t}). 
  \end{aligned}
  \end{equation}
   Applying~\eqref{eqn:best-observation-instregret-pf-15} to~\eqref{eqn:best-observation-instregret-pf-14}, we have  
   \begin{equation} \label{eqn:best-observation-instregret-pf-16}
  \centering
  \begin{aligned}
        r_t  \leq& (\beta_t^{1/2} +\phi(0))\sigma_{t-1}(\xbm_{t})+ c_{\alpha}\beta^{1/2}_t\sigma_{t-1}([\xbm^*]_t)+\frac{1}{t^2}. 
  \end{aligned}
  \end{equation}
  
  Second, $ y^+_{t-1} \leq f(\xbm_{t})$.
  From Lemma~\ref{lem:compact-IL} and $E^s(t)$,
   \begin{equation} \label{eqn:best-observation-instregret-pf-1}
  \centering
  \begin{aligned}
          r_t  =& f(\xbm_t) - f(\xbm^*) =f(\xbm_t) - y^+_{t-1} + y^+_{t-1}  - f(\xbm^*)\leq f(\xbm_t) - y^+_{t-1} + I_{t-1}(\xbm^*)\\ 
          \leq& f(\xbm_t) -\mu_{t-1}(\xbm_t)+\mu_{t-1}(\xbm_t)- y^+_{t-1}+EI_{t-1}([\xbm^*]_t)  + c_{\alpha}\beta_t^{1/2} \sigma_{t-1}([\xbm^*]_t)+\frac{1}{t^2}\\ 
          \leq&  \mu_{t-1}(\xbm_t) - y^+_{t-1}+EI_{t-1}(\xbm_{t})  + \beta_t^{1/2}\sigma_{t-1}(\xbm_t)+ c_{\alpha}\beta_t^{1/2} \sigma_{t-1}([\xbm^*]_t)+\frac{1}{t^2},\\ 
  \end{aligned}
  \end{equation}
   where the second inequality uses Lemma~\ref{lem:compact-IL} and the third inequality uses the definition of $\xbm_t$ and $E^s(t)$.

  Next, we aim to provide a lower bound for $EI_{t-1}(\xbm_t)$. Unlike the BPMI and BSPMI cases, we compare $EI_{t-1}(\xbm_t)$ to $EI_{t-1}(\xbm^*)$ to find such a bound instead of $EI_{t-1}(\xbm_t^+)$, due to the noise at $y^+_{t-1}$.
  For $EI_{t-1}(\xbm^*)$, we can write via $E^s_2(t)$ that  
  \begin{equation} \label{eqn:best-observation-instregret-pf-3.1}
 \centering
 \begin{aligned}
   EI_{t-1}(\xbm^*) = \sigma_{t-1}(\xbm^*)\tau(z_{t-1}(\xbm^*)) =& \sigma_{t-1}(\xbm^*)\tau\left(\frac{y^+_{t-1}-f(\xbm^*)+f(\xbm^*)-\mu_{t-1}(\xbm^*)}{\sigma_{t-1}(\xbm^*)}\right)\\
      \geq& \sigma_{t-1}(\xbm^*)\tau\left(\frac{y^+_{t-1}-f(\xbm^*)}{\sigma_{t-1}(\xbm^*)} - \beta_t^{1/2}\right).
   \end{aligned}
  \end{equation}
  By  $E^y(t)$~\eqref{eqn:eft-y}, we can further write~\eqref{eqn:best-observation-instregret-pf-3.1} as 
  \begin{equation} \label{eqn:best-observation-instregret-pf-3.2}
 \centering
 \begin{aligned}
   &EI_{t-1}(\xbm^*) 
      \geq \sigma_{t-1}(\xbm^*)\tau\left(\frac{\beta_t^{1/2}\sigma_{t-1}(\xbm_t)}{\sigma_{t-1}(\xbm^*)} - \beta_t^{1/2}\right)
      \geq \sigma_{t-1}(\xbm^*)\tau(-\beta_t^{1/2}).
   \end{aligned}
  \end{equation}
  Similarly, for $EI_{t-1}(\xbm_t)$, we can write via $E^1_s(t)$ that  
  \begin{equation} \label{eqn:best-observation-instregret-pf-4}
 \centering
 \begin{aligned}
   EI_{t-1}(\xbm_t) =& \sigma_{t-1}(\xbm_t)\tau(z_{t-1}(\xbm_t)) = \sigma_{t-1}(\xbm_t)\tau\left(\frac{y^+_{t-1}-f(\xbm_t)+f(\xbm_t)-\mu_{t-1}(\xbm_t)}{\sigma_{t-1}(\xbm_t)}\right)\\
      \leq& \sigma_{t-1}(\xbm_t)\tau\left(\frac{y^+_{t-1}-f(\xbm_t)}{\sigma_{t-1}(\xbm_t)} + \beta_t^{1/2}\right).
   \end{aligned}
  \end{equation}
  We further consider two cases.
  First,  $y^+_{t-1}-f(\xbm_t) <  -2\beta_t^{1/2}\sigma_{t-1}(\xbm_t)$. Then,~\eqref{eqn:best-observation-instregret-pf-4} implies 
  \begin{equation} \label{eqn:best-observation-instregret-pf-7}
 \centering
 \begin{aligned}
   EI_{t-1}(\xbm_t) 
      <& \sigma_{t-1}(\xbm_t)\tau( -2\beta_t^{1/2} + \beta_t^{1/2}) \leq \sigma_{t-1}(\xbm_t)\tau(-\beta_t^{1/2}).
   \end{aligned}
  \end{equation}
  Second, $y^+_{t-1}-f(\xbm_t)\geq  -2\beta_t^{1/2}\sigma_{t-1}(\xbm_t)$. 
  Then, from $E^s_1(t)$, we have 
  \begin{equation*} \label{eqn:best-observation-instregret-pf-5}
 \centering
 \begin{aligned}
      y^+_{t-1}-\mu_{t-1}(\xbm_t) = y^+_{t-1}-f(\xbm_t)+f(\xbm_t)-\mu_{t-1}(\xbm_t)\geq -3\beta_t^{1/2}\sigma_{t-1}(\xbm_t).
   \end{aligned} 
  \end{equation*}
   Thus, 
   \begin{equation} \label{eqn:best-observation-instregret-pf-6}
 \centering
 \begin{aligned}
      \mu_{t-1}(\xbm_t)- y^+_{t-1}\leq 3\beta_t^{1/2}\sigma_{t-1}(\xbm_t).
   \end{aligned} 
  \end{equation}
  Using~\eqref{eqn:best-observation-instregret-pf-3.2} and~\eqref{eqn:best-observation-instregret-pf-7}, we have
   \begin{equation} \label{eqn:best-observation-instregret-pf-8}
 \centering
 \begin{aligned}
   EI_{t-1}(\xbm_t)< \frac{\sigma_{t-1}(\xbm_t)}{\sigma_{t-1}(\xbm^*)} \sigma_{t-1}(\xbm^*)\tau(-\beta_t^{1/2})< \frac{\sigma_{t-1}(\xbm_t)}{\sigma_{t-1}(\xbm^*)} EI_{t-1}(\xbm^*).
   \end{aligned}
  \end{equation}
  However, by definition $EI_{t-1}(\xbm_t)\geq EI_{t-1}(\xbm^*)$. Thus,~\eqref{eqn:best-observation-instregret-pf-8} implies  
   \begin{equation} \label{eqn:best-observation-instregret-pf-9}
 \centering
 \begin{aligned}
          \sigma_{t-1}(\xbm_t)> \sigma_{t-1}(\xbm^*).
   \end{aligned}
  \end{equation}
  From $E^y(t)$, $y^+_{t-1}-f(\xbm^*)\geq \beta_{t}^{1/2}\sigma_{t-1}(\xbm_t)> \beta_{t}^{1/2}\sigma_{t-1}(\xbm^*)$.
  Then, by~\eqref{eqn:best-observation-instregret-pf-3.1} and $E^y(t)$, we can write 
\begin{equation} \label{eqn:best-observation-instregret-pf-10}
 \centering
 \begin{aligned}
   EI_{t-1}(\xbm^*) >\sigma_{t-1}(\xbm^*)\tau(0)\geq \phi(0) c_{\sigma}\sqrt{\frac{\sigma^2}{t+\sigma^2}} .
   \end{aligned}
  \end{equation}
  where the last inequality is from Lemma~\ref{thm:gp-sigma-bound}.
  By Lemma~\ref{lem:mu-bounded-EI},~\eqref{eqn:best-observation-instregret-pf-10} implies 
\begin{equation} \label{eqn:best-observation-instregret-pf-11}
  \centering
  \begin{aligned}
          y^+_{t-1}-\mu_{t-1}(\xbm_t) \geq - \log^{\frac{1}{2}}\left(\frac{t-1+\sigma^2}{2\pi \phi^2(0)c_{\sigma}^2 \sigma^2}  \right) \sigma_{t-1}(\xbm_t).
  \end{aligned}
  \end{equation}
  Combining~\eqref{eqn:best-observation-instregret-pf-6} and~\eqref{eqn:best-observation-instregret-pf-11}, we have 
\begin{equation} \label{eqn:best-observation-instregret-pf-12}
  \centering
  \begin{aligned}
         \mu_{t-1}(\xbm_t)- y^+_{t-1}\leq c_y(t)  \sigma_{t-1}(\xbm_t).
  \end{aligned}
  \end{equation}
  where $c_y(t)= \max\{\log^{\frac{1}{2}}(\frac{t-1+\sigma^2}{2\pi \phi^2(0)c_{\sigma}^2 \sigma^2}),3\}$.
  
  Now, consider the second term on the last line of~\eqref{eqn:best-observation-instregret-pf-1}.  Using $E^s(t)$, $\Phi(\cdot)<1$, and $y^+_{t-1}-f(\xbm_{t})\leq 0$, we have 
  \begin{equation} \label{eqn:best-observation-instregret-pf-2}
 \centering
 \begin{aligned}
          EI_{t-1}(\xbm_t) 
          =& (y^+_{t-1}-\mu_{t-1}(\xbm_{t}))\Phi(z_{t-1}(\xbm_{t}))+\sigma_{t-1}(\xbm_{t})\phi(z_{t-1}(\xbm_{t}))\\ 
          \leq& (y^+_{t-1}-f(\xbm_{t})+f(\xbm_{t})-\mu_{t-1}(\xbm_{t}))\Phi(z_{t-1}(\xbm_{t}))+\phi(0)\sigma_{t-1}(\xbm_{t})\\ 
          \leq& \beta_t^{1/2}\sigma_{t-1}(\xbm_{t})+\phi(0)\sigma_{t-1}(\xbm_{t}). 
  \end{aligned}
  \end{equation}
  Applying\eqref{eqn:best-observation-instregret-pf-12} and~\eqref{eqn:best-observation-instregret-pf-2}  to~\eqref{eqn:best-observation-instregret-pf-1}, we have 
   \begin{equation} \label{eqn:best-observation-instregret-pf-13}
  \centering
  \begin{aligned}
          r_t  
          \leq&  (c_y(t) +\phi(0)+2\beta_t^{1/2})\sigma_{t-1}(\xbm_t)   + c_{\alpha}\beta_t^{1/2} \sigma_{t-1}([\xbm^*]_t)+\frac{1}{t^2},\\ 
  \end{aligned}
  \end{equation}  
   Combine~\eqref{eqn:best-observation-instregret-pf-13} and~\eqref{eqn:best-observation-instregret-pf-16} and the proof is complete.
\end{proof}
\subsection{Proof of Theorem~\ref{theorem:EI-sigma-star}}\label{se:appx-proof-sigma-star}
The proof of Lemma~\ref{lem:ei-r-1} is given below.
\begin{proof}
   Since $f(\xbm)\sim\mathcal{N}(\mu_{t-1}(\xbm), \sigma_{t-1}^2(\xbm))$, from Lemma~\ref{lem:phi}, we have 
       \begin{equation} \label{eqn:ft-bound-RKHS-pf-1}
  \centering
  \begin{aligned}
        \Pbb\{|f(\xbm)-\mu_{t-1}(\xbm)| > \sqrt{\alpha\log(t)} \sigma_{t-1}(\xbm)\}\leq& e^{-\frac{1}{2} \alpha\log( t)}
        = \frac{1}{t^{\frac{\alpha}{2}}}.
   \end{aligned}
  \end{equation} 
\end{proof}

The proof of Lemma~\ref{lem:EI-regret-parameters}, which shows the relationship between the parameters $\alpha_t$, $\beta_t$, $\zeta_t$, and $\eta_t$, is provided next.
\begin{proof} 
   Consider first $t\geq 2$. From Lemma~\ref{lem:Philowbound}, 
   \begin{equation} \label{eqn:EI-sigma-star-pm-pf-5}
   \centering
   \begin{aligned}
       \Phi(-\alpha_t^{1/2}) > \frac{\alpha_t^{1/2}}{1+\alpha_t}\frac{1}{\sqrt{2\pi}}e^{-\alpha_t/2}=\frac{1}{\sqrt{2\pi}}\frac{\sqrt{\alpha\log(t)}}{1+\alpha\log(t)}\left(\frac{1}{t}\right)^{\frac{\alpha}{2}}. 
    \end{aligned}
   \end{equation}
   Recall that without losing genearlity, $Lrd\geq 1$ from Assumption~\ref{assp:gp}.
  From the definition~\eqref{def:beta} and $0<\delta\leq 1$, 
   \begin{equation} \label{eqn:EI-sigma-star-pm-pf-5.5}
   \centering
   \begin{aligned}
      \beta_t^{1/2}  =& \sqrt{2\log\left(8 \frac{\pi^2t^2}{6} (Lrt^2d)^d /\delta\right)} 
         > \sqrt{2\log(\pi^2  t^{2+2d} (Lrd)^d/\delta)}\\
         =&\sqrt{4(1+d)\log( t)+2d\log( Lrd)+2\log(\pi^2/\delta)}\\
         >&\sqrt{4(1+d)\log( t)+2\log(\pi^2)+2d\log(Lrd)}.
    \end{aligned}
   \end{equation}
   Using~\eqref{eqn:EI-sigma-star-pm-pf-5},~\eqref{eqn:EI-sigma-star-pm-pf-5.5}, $\pi>3$, $\alpha\leq 1$, $d\geq 1$, and $t\geq 2$, we can write 
   \begin{equation} \label{eqn:EI-sigma-star-pm-pf-6}
   \centering
   \begin{aligned}
      \beta_t^{1/2}& \Phi(-\alpha_t^{1/2}) > 
          \sqrt{4(1+d)\log( t) + 2d\log (L r d )+2\log(\pi^2)}\frac{1}{\sqrt{2\pi}}\frac{\sqrt{\alpha\log(t)}}{1+\alpha\log(t)}\left(\frac{1}{t}\right)^{\frac{\alpha}{2}}\\
         >& \frac{\sqrt{\alpha}}{\sqrt{2\pi}} [4(1+d) \log^2( t) + 2d\log (Lr d )\log(t)+2\log(\pi^2)\log(t)]^{\frac{1}{2}}\frac{1}{1+\alpha\log(t)}\left(\frac{1}{t}\right)^{\frac{\alpha}{2}}\\
         >& \frac{\sqrt{\alpha}}{\sqrt{2\pi}}\left(\frac{1}{t}\right)^{\frac{\alpha}{2}},
    \end{aligned}
   \end{equation}
   where the last inequality is obvious if one takes the square of $1+\alpha \log(t)$.
    If $t=1$, $\alpha_t=0$, $\Phi(-\alpha_t^{1/2})=\frac{1}{2}$, and $\beta_t^{1/2}\geq \sqrt{2\log(8\pi^2/6)}$, which leads to~\eqref{eqn:EI-sigma-star-pm-pf-6} trivially.  
  By the choice of $\zeta_t^{1/2}$,~\eqref{eqn:EI-sigma-star-pm-pf-6} leads to~\eqref{eqn:EI-regret-parameters}
\end{proof}

The proof of Theorem~\ref{theorem:EI-sigma-star}, the transformed instantaneous regret bound, is provided next.
\begin{proof}
   We first note that $E^r(t)$ being true means that~\eqref{eqn:eft-1} and~\eqref{eqn:eft-2} with $\ubm_t=\xbm_t$ hold for all three incumbents.
   From Lemma~\ref{lem:post-mean-instregret},~\ref{lem:sampled-post-mean-instregret}, and~\ref{lem:best-observation-instregret}, we have 
      \begin{equation} \label{eqn:EI-sigma-star-pm-pf-1}
   \centering
   \begin{aligned}
        r_t \leq (c_{\xi}(t)+\phi(0)+2\beta_t^{1/2}) \sigma_{t-1}(\xbm_{t}) + c_{\alpha}\beta_t^{1/2}\sigma_{t-1}([\xbm^*]_t)+\frac{1}{t^2}, 
    \end{aligned}
   \end{equation} 
   where $c_{\alpha}=1.326$, $c_{\xi}(t)=c_{\mu}(t)=\log^{\frac{1}{2}}(\frac{t-1+\sigma^2}{2\pi \phi^2(0)c_{\sigma}^2 \sigma^2})$ for both BPMI and BSPMI, and $c_{\xi}=c_y(t)= \max\{c_{\mu}(t),3\}$ for BOI. Further,~\eqref{eqn:EI-sigma-star-pm-pf-1} only stands for BOI when $E^y(t)$ is true as well.
 
   We consider the behavior of $\sigma_{t-1}([\xbm^*]_t)$.
   By the definition of $\xbm_{t}$, it generates the largest $EI_{t-1}$, \textit{i.e.},
   \begin{equation} \label{eqn:EI-sigma-star-pm-pf-4}
  \centering
  \begin{aligned}
        EI_{t-1}(\xbm_{t}) \geq EI_{t-1}([\xbm^*]_t).
    \end{aligned}
  \end{equation}
   We aim to show that~\eqref{eqn:EI-sigma-star-pm-pf-4} guarantees a small $\sigma_{t-1}([\xbm^*]_t)$ in some cases.   
  We consider two scenarios regarding the term $f(\xbm_{t})-f([\xbm^*]_t)$ at given $t$.

   \textbf{Scenario A.} Suppose first that $f(\xbm_{t})-f([\xbm^*]_t)$ reduces to $0$ at a fast enough rate to satisfy
   \begin{equation} \label{eqn:EI-sigma-star-pm-pf-8}
   \centering
   \begin{aligned}
      f(\xbm_{t})-f([\xbm^*]_t) \leq  c_1(t)\max\{\xi^+_{t-1}-\mu_{t-1}(\xbm_{t}),0\}+c_2(t)\sigma_{t-1}(\xbm_{t}).
    \end{aligned}
   \end{equation}
    Then, from the definition of $r_t$ and the discretization $\Cbb_t$~\eqref{eqn:ht-1}, we have  
   \begin{equation} \label{eqn:EI-sigma-star-pm-pf-9}
   \centering
   \begin{aligned}
         r_t =& f(\xbm_{t})-f(\xbm^*) =f(\xbm_{t})-f([\xbm^*]_t)+ f([\xbm^*]_t)-f(\xbm^*)\\
        \leq& c_1(t)\max\{\xi^+_{t-1}-\mu_{t-1}(\xbm_{t}),0\}+c_2(t)\sigma_{t-1}(\xbm_{t}) +\frac{1}{t^2},
    \end{aligned}
   \end{equation}
   which then proves~\eqref{eqn:EI-sigma-star-1}. 
   
  \textbf{Scenario B.} Next, we analyze the B scenario where
   \begin{equation} \label{eqn:EI-sigma-star-pm-pf-10}
   \centering
   \begin{aligned}
     f(\xbm_{t})-f([\xbm^*]_t) > c_1(t)\max\{\xi^+_{t-1}-\mu_{t-1}(\xbm_{t}),0\}+ c_2(t) \sigma_{t-1}(\xbm_{t}). 
    \end{aligned}
   \end{equation}
    Our idea is to show that given our choice of $\alpha_t$, $\beta_t$, and $\eta_t$, the exploitation and exploration trade-off of EI and~\eqref{eqn:EI-sigma-star-pm-pf-4} would ensure $\sigma_{t-1}([\xbm^*]_t)$ is reduced, namely, $\sigma_{t-1}([\xbm^*]_t)\leq \sigma_{t-1}(\xbm_{t})$,  with high probability under scenario B. 
  We prove by contradiction. Suppose on the contrary,   
   \begin{equation} \label{eqn:EI-sigma-star-pm-pf-premise}
  \centering
  \begin{aligned}
     \sigma_{t-1}(\xbm_{t})     < \sigma_{t-1}([\xbm^*]_t)\leq 1,
  \end{aligned}
  \end{equation}
   with probability $1-\delta_{\sigma}$.

   We derive the relationship between the exploitation $\xi^+_{t-1}-\mu_{t-1}(\xbm_{t})$ at $\xbm_t$ and exploitation $\xi^+_{t-1}-\mu_{t-1}([\xbm^*]_t)$ at $[\xbm^*]_t$. Using Lemma~\ref{lem:ei-r-1} at $[\xbm^*]_t$ and $t-1$, with probability $\geq1-\frac{1}{2}(\frac{1}{t})^{\frac{\alpha}{2}}$,
   \begin{equation} \label{eqn:EI-sigma-star-pm-pf-2}
   \centering
   \begin{aligned}
        f([\xbm^*]_t)-\mu_{t-1}([\xbm^*]_t) > - \alpha_t^{1/2}\sigma_{t-1}([\xbm^*]_t). 
    \end{aligned}
   \end{equation} 
   Given~\eqref{eqn:eft-2} (with $\ubm_t=\xbm_t$)  and~\eqref{eqn:EI-sigma-star-pm-pf-2}, with probability $\geq 1-\frac{1}{2 t^{\frac{\alpha}{2}}}$,
   \begin{equation} \label{eqn:EI-sigma-star-pm-pf-3}
   \centering
   \begin{aligned}
      &\xi^+_{t-1}-\mu_{t-1}([\xbm^*]_t)> \xi^+_{t-1}-f([\xbm^*]_t)- \alpha_t^{1/2}\sigma_{t-1}([\xbm^*]_t) \\
          =& \xi^+_{t-1}-\mu_{t-1}(\xbm_{t})+\mu_{t-1}(\xbm_{t})-f(\xbm_{t})+f(\xbm_{t})-f([\xbm^*]_t)- \alpha_t^{1/2}\sigma_{t-1}([\xbm^*]_t)\\
           \geq& \xi^+_{t-1}-\mu_{t-1}(\xbm_{t}) +f(\xbm_{t})-f([\xbm^*]_t)- \alpha_t^{1/2}\sigma_{t-1}([\xbm^*]_t)-\beta_t^{1/2}\sigma_{t-1}(\xbm_{t}).
    \end{aligned}
   \end{equation}

   To proceed, we use the \textit{trade-off} notation of exploitation and exploration for EI in~\eqref{eqn:EI-ab}. Let $a_{t}=\xi^+_{t-1}-\mu_{t-1}(\xbm_{t})$ and $b_{t}= \sigma_{t-1}(\xbm_{t})$, \textit{i.e.}, $EI(a_{t},b_{t})=EI_{t-1}(\xbm_{t})$.
   Based on~\eqref{eqn:EI-sigma-star-pm-pf-10} and~\eqref{eqn:EI-sigma-star-pm-pf-premise}, let $b^*_{t}=\sigma_{t-1}([\xbm^*]_t)$ and  
   \begin{equation} \label{eqn:EI-sigma-star-pm-pf-11}
   \centering
   \begin{aligned}
          a^*_{t} &= \xi^+_{t-1}-\mu_{t-1}(\xbm_{t})+f(\xbm_{t})-f([\xbm^*]_t)-\alpha_t^{1/2}\sigma_{t-1}([\xbm^*]_t)-\beta_t^{1/2}\sigma_{t-1}(\xbm_{t})\\
      &= a_t+f(\xbm_{t})-f([\xbm^*]_t)-\alpha_t^{1/2}b_t^*-\beta_t^{1/2}b_t\\
      &> a_t+c_1(t)\max\{a_t,0\}+c_2(t) b_t -\alpha_t^{1/2} b^*_t-\beta^{1/2}_tb_t\\
     & >a_{t}+c_1(t)\max\{a_t,0\}- \alpha_t^{1/2} b^*_{t}+ \eta_t^{1/2} b_{t}, 
    \end{aligned}
   \end{equation}
   where the last inequality uses $c_2(t)=\eta_t^{1/2}+4\beta_t^{1/2}>\beta_t^{1/2}+\eta_t^{1/2}$.
   Using~\eqref{eqn:EI-sigma-star-pm-pf-3}, we have 
   \begin{equation} \label{eqn:EI-sigma-star-pm-pf-12}
   \centering
   \begin{aligned}
      \xi^+_{t-1}-\mu_{t-1}([\xbm^*]_t)\geq a^*_{t},
    \end{aligned}
   \end{equation}
with probability $\geq 1-\frac{1}{2t^{\frac{\alpha}{2}}} $.
   Since EI is monotonically increasing with respect to both $a_{t}$ and $b_{t}$ by Lemma~\ref{lem:EI-ms}, from~\eqref{eqn:EI-sigma-star-pm-pf-12}, we have  
   \begin{equation} \label{eqn:EI-sigma-star-pm-pf-13}
   \centering
   \begin{aligned}
          EI_{t-1}([\xbm^*]_t) =EI(\xi^+_{t-1}-\mu_{t-1}([\xbm^*]_t),\sigma_{t-1}([\xbm^*]_t)) \geq EI(a^*_{t},b^*_{t}),
    \end{aligned}
   \end{equation}
  with probability $\geq 1-\frac{1}{2t^{\frac{\alpha}{2}}}  $.
  Moreover, $b^*_{t}>b_{t}$ with probability $1-\delta_{\sigma}$ by~\eqref{eqn:EI-sigma-star-pm-pf-premise}.
   We consider two cases regarding $b_{t}^*$ and~\eqref{eqn:EI-sigma-star-pm-pf-11}. We aim to show that in both cases, $EI(a^*_{t},b^*_{t}) > EI(a_{t},b_{t})$  stands with probability $1-\delta_{\sigma}$.
   
   \textbf{Case 1} Consider first the case where  
   \begin{equation} \label{eqn:EI-sigma-star-pm-pf-14}
   \centering
   \begin{aligned}
          b^*_{t} \leq \frac{c_1(t)}{\alpha^{1/2}_t}\max\{a_{t},0\}+\frac{\eta_t^{1/2}}{\alpha_t^{1/2}}b_{t}.
    \end{aligned}
   \end{equation}
   Then, by~\eqref{eqn:EI-sigma-star-pm-pf-11}, $a^*_{t} > a_{t}$. Further, from~\eqref{eqn:EI-sigma-star-pm-pf-premise}, with probability $1-\delta_{\sigma}$,
   \begin{equation} \label{eqn:EI-sigma-star-pm-pf-15}
   \centering
   \begin{aligned}
       EI(a^*_{t},b^*_{t}) > EI(a_{t},b_{t})=EI_{t-1}(\xbm_{t}).
    \end{aligned}
   \end{equation}
   
   \textbf{Case 2.} Next, we consider the case where 
   \begin{equation} \label{eqn:EI-sigma-star-pm-pf-16}
   \centering
   \begin{aligned}
          b^*_{t} > \frac{c_1(t)}{\alpha^{1/2}_t}\max\{a_{t},0\}+\frac{\eta_t^{1/2}}{\alpha_t^{1/2}}b_{t}\geq \frac{\eta_t^{1/2}}{\alpha_t^{1/2}}b_{t}.
    \end{aligned}
   \end{equation} 
   By definition, $\frac{\eta_t^{1/2} }{\alpha_t^{1/2}} >1$. 
   Thus, it is clear that in \textbf{Case 2},~\eqref{eqn:EI-sigma-star-pm-pf-premise} always holds. 
   We further distinguish between two cases based on the sign of $a_{t}$.
   
   \textbf{Case 2.1} 
   Consider $a_{t}<0$ and $z_{t}=\frac{a_{t}}{b_{t}}<0$. 
   Using~\eqref{eqn:EI-sigma-star-pm-pf-11} and~\eqref{eqn:EI-sigma-star-pm-pf-16}, 
   \begin{equation} \label{eqn:EI-sigma-star-pm-pf-17}
   \centering
   \begin{aligned}
            EI(a^*_{t},b^*_{t}) \geq& EI(a_{t}+\eta_t^{1/2} b_t - \alpha_t^{1/2} b^*_{t},b^*_{t}) = b^*_{t} \tau\left(\frac{a_{t}+ \eta_t^{1/2} b_{t}- \alpha_t^{1/2} b^*_{t}}{b^*_{t}}\right) \\
             =& \frac{1}{\rho_{t}}  b_{t} \tau\left((z_{t}+\eta_t^{1/2} )\rho_{t}- \alpha_t^{1/2}\right).
    \end{aligned}
   \end{equation}
   where $\rho_{t}=\frac{b_{t}}{b^*_{t}}$. 
Consider two cases based on the value of $z_{t}$.

   \textbf{Case 2.1.1} 
   Suppose $z_{t}> -\eta_t^{1/2}$. Define function $\tilde{\tau}:\Rbb\to\Rbb$ at $z,\eta_t,\alpha_t$ as 
   \begin{equation} \label{eqn:EI-sigma-star-pm-pf-bartau}
   \centering
   \begin{aligned}
    \bar{\tau}(\rho;z,\eta_t,\alpha_t) = 
\frac{1}{\rho} \tau\left( (z+ \eta_t^{1/2})\rho-\alpha_t^{1/2}\right),
    \end{aligned}
   \end{equation}
   and $\rho\in (0,\frac{\alpha_t^{1/2}}{\eta_t^{1/2}}), - \eta_t^{1/2}\leq z < 0$. 
   We aim to find the minimum of~\eqref{eqn:EI-sigma-star-pm-pf-bartau} given $z$.
   Taking the derivative with $\rho$ using Lemma~\ref{lem:tau}, we have 
   \begin{equation} \label{eqn:EI-sigma-star-pm-pf-18}
   \centering
   \begin{aligned}
    \frac{d \bar{\tau}}{d \rho} =& -\frac{1}{\rho^2}\tau\left( (z+ \eta_t^{1/2})\rho-\alpha_t^{1/2}\right)+\frac{1}{\rho}\Phi((z+ \eta_t^{1/2})\rho-\alpha_t^{1/2})(z+\eta_t^{1/2})\\
          =&-\frac{1}{\rho^2}[((z+ \eta_t^{1/2})\rho-\alpha_t^{1/2})\Phi\left( (z+ \eta_t^{1/2})\rho-\alpha_t^{1/2}\right)+\phi((z+ \eta_t^{1/2})\rho-\alpha_t^{1/2})]\\
          &+\frac{1}{\rho}\Phi((z+ \eta_t^{1/2})\rho-\alpha_t^{1/2})(z+\eta_t^{1/2})\\
         =&-\frac{1}{\rho^2}[-\alpha_t^{1/2}\Phi\left( (z+ \eta_t^{1/2})\rho-\alpha_t^{1/2}\right)+\phi((z+ \eta_t^{1/2})\rho-\alpha_t^{1/2})].
    \end{aligned}
   \end{equation}
   Define the function $\theta(\rho;z,\eta_t,\alpha_t) = -\alpha_t^{1/2}\Phi\left( (z+ \eta_t^{1/2})\rho-\alpha_t^{1/2}\right)+\phi((z+ \eta_t^{1/2})\rho-\alpha_t^{1/2})$.
   From~\eqref{eqn:EI-sigma-star-pm-pf-18}, the sign of $\frac{d \bar{\tau}}{d \rho}$ is determined by $\theta$.
  The derivative of $\theta$ with $\rho$ is
   \begin{equation} \label{eqn:EI-sigma-star-pm-pf-19}
   \centering
   \begin{aligned}
    \frac{d \theta}{d \rho} 
         =&  - \alpha_t^{1/2}\phi((z+ \eta_t^{1/2})\rho-\alpha_t^{1/2})(z+\eta_t^{1/2}) - \phi((z+ \eta_t^{1/2})\rho-\alpha_t^{1/2})((z+ \eta_t^{1/2})\rho-\alpha_t^{1/2})(z+\eta_t^{1/2})\\
          =& -\phi((z+ \eta_t^{1/2})\rho-\alpha_t^{1/2})(z+\eta_t^{1/2})^2\rho<0, 
    \end{aligned}
   \end{equation}
  for $\forall \rho \in(0,\frac{\alpha_t^{1/2}}{\eta_t^{1/2}})$. Therefore, $\theta$ is monotonically decreasing with $\rho$. 
  Now let $\rho\to 0$. We have $\theta(\rho;z,\eta_t,\alpha_t)\to -\alpha_t^{1/2}\Phi(-\alpha_t^{1/2})+\phi(-\alpha_t^{1/2})=\tau(-\alpha_t^{1/2})>0$.
  Further, let $\rho \to \frac{\alpha_t^{1/2}}{\eta_t^{1/2}}$. Then, $\theta(\rho;z,\eta_t,\alpha_t)\to -\alpha_t^{1/2} \Phi\left(\frac{\alpha_t^{1/2}}{\eta_t^{1/2}}z\right) + \phi\left(\frac{\alpha_t^{1/2}}{\eta_t^{1/2}}z\right)$. 
  Depending on the sign of $-\alpha_t^{1/2} \Phi\left(\frac{\alpha_t^{1/2}}{\eta_t^{1/2}}z\right) + \phi\left(\frac{\alpha_t^{1/2}}{\eta_t^{1/2}}z\right)$, $\bar{\tau}$ has two different types of behavior.  

   \textbf{Case 2.1.1.1} 
   First, $-\alpha_t^{1/2} \Phi\left(\frac{\alpha_t^{1/2}}{\eta_t^{1/2}}z\right) + \phi\left(\frac{\alpha_t^{1/2}}{\eta_t^{1/2}}z\right) \geq 0$, which means $\theta(\rho;z,\eta_t,\alpha_t)>0$ for $\forall \rho\in(0,\frac{\alpha_t^{1/2}}{\eta_t^{1/2}})$. By~\eqref{eqn:EI-sigma-star-pm-pf-18}, $\bar{\tau}(\rho;z,w)$ is monotonically decreasing with $\rho$ and 
   \begin{equation} \label{eqn:EI-sigma-star-pm-pf-20}
   \centering
   \begin{aligned}
    \bar{\tau}(\rho;z,\eta_t,\alpha_t)  > \bar{\tau}\left(\frac{\alpha_t^{1/2}}{\eta_t^{1/2}};z,\eta_t,\alpha_t\right) = \frac{\eta_t^{1/2}}{\alpha_t^{1/2}} \tau\left( \frac{\alpha_t^{1/2}}{\eta_t^{1/2}}z\right).
    \end{aligned}
   \end{equation}
   Using $\frac{\alpha_t^{1/2}}{\eta_t^{1/2}}<1$, $z<0$,  and the monotonicity of $\tau$ (Lemma~\ref{lem:tau}), we have 
   \begin{equation} \label{eqn:EI-sigma-star-pm-pf-21}
   \centering
   \begin{aligned}
    \bar{\tau}(\rho;z,\eta_t,\alpha_t)  > \frac{\eta_t^{1/2}}{\alpha_t^{1/2}}  \tau\left( \frac{\alpha_t^{1/2}}{\eta_t^{1/2}}z\right)>\tau\left(\frac{\alpha_t^{1/2}}{\eta_t^{1/2}} z\right)>\tau(z).
    \end{aligned}
   \end{equation}
  Applying~\eqref{eqn:EI-sigma-star-pm-pf-21} to~\eqref{eqn:EI-sigma-star-pm-pf-17} with $z=z_{t}$,
   \begin{equation} \label{eqn:EI-sigma-star-pm-pf-22}
   \centering
   \begin{aligned}
    EI(a^*_{t},b^*_{t}) \geq  b_{t}\bar{\tau}(\rho_{t};z_{t},\eta_t,\alpha_t)  > b_{t} \tau(z_{t})=EI(a_{t},b_{t}).
    \end{aligned}
   \end{equation}

   \textbf{Case 2.1.1.2} 
   In the second case, $-\alpha_t^{1/2} \Phi\left(\frac{\alpha_t^{1/2}}{\eta_t^{1/2}}z\right) + \phi\left(\frac{\alpha_t^{1/2}}{\eta_t^{1/2}}z\right) < 0$, which means there exists a unique $\bar{\rho}\in\left(0,\frac{\alpha_t^{1/2}}{\eta_t^{1/2}}\right)$ as a stationary point for $\theta(\rho;z,\eta_t,\alpha_t)$, \textit{i.e.},
   \begin{equation} \label{eqn:EI-sigma-star-pm-pf-23}
   \centering
   \begin{aligned}
      -\alpha_t^{1/2}\Phi((z+\eta_t^{1/2})\bar{\rho}-\alpha_t^{1/2}) + \phi((z+\eta_t^{1/2})\bar{\rho}-\alpha_t^{1/2}) =0.
    \end{aligned}
   \end{equation}
We claim $\bar{\rho}$ is a local minimum of $\bar{\tau}$ via the second derivative test.  
  From~\eqref{eqn:EI-sigma-star-pm-pf-18},~\eqref{eqn:EI-sigma-star-pm-pf-19}, and~\eqref{eqn:EI-sigma-star-pm-pf-23}, the second-order derivative of $\bar{\tau}$ with $\rho$ at $\bar{\rho}$ is
   \begin{equation} \label{eqn:EI-sigma-star-pm-pf-24}
   \centering
   \begin{aligned}
    \frac{d^2 \bar{\tau}}{d \rho^2}|_{\bar{\rho}} =& \frac{2}{\bar{\rho}^3}\theta(\bar{\rho};z,\eta_t,\alpha_t)-\frac{1}{\bar{\rho}^2}\frac{d\theta}{d\rho}(\bar{\rho};z,\eta_t,\alpha_t) = \frac{1}{\rho}\phi((z+ \eta_t^{1/2})\bar{\rho}-\alpha_t^{1/2})(z+\eta_t^{1/2})^2>0.
    \end{aligned}
   \end{equation}
   It is well-known from second derivative test that $\bar{\rho}$ is a local minimum of $\bar{\tau}$.
  Further, $\frac{d\bar{\tau}}{d\rho}<0$ for $\rho<\bar{\rho}$ and  $\frac{d\bar{\tau}}{d\rho}>0$ for $\rho>\bar{\rho}$. Therefore, $\bar{\rho}$ is the global minimum on $\left(0,\frac{\alpha_t^{1/2}}{\eta_t^{1/2}}\right)$ and by~\eqref{eqn:EI-sigma-star-pm-pf-23},
   \begin{equation} \label{eqn:EI-sigma-star-pm-pf-25}
   \centering
   \begin{aligned}
    \bar{\tau}(\bar{\rho};z,\eta_t,\alpha_t) &= \frac{1}{\bar{\rho}} [((z+ \eta_t^{1/2})\bar{\rho}-\alpha_t^{1/2})\Phi\left( (z+ \eta_t^{1/2})\bar{\rho}-\alpha_t^{1/2}\right) + \phi\left( (z+ \eta_t^{1/2})\bar{\rho}-\alpha_t^{1/2}\right)]\\
         &=\frac{1}{\bar{\rho}} [((z+ \eta_t^{1/2})\bar{\rho}-\alpha_t^{1/2})\Phi\left( (z+ \eta_t^{1/2})\bar{\rho}-\alpha_t^{1/2}\right) + \alpha_t^{1/2}\Phi\left( (z+ \eta_t^{1/2})\bar{\rho}-\alpha_t^{1/2}\right)] \\&
         = (z+ \eta_t^{1/2}) \Phi\left( (z+ \eta_t^{1/2})\bar{\rho}-\alpha_t^{1/2}\right).
    \end{aligned}
   \end{equation}
   Denote the corresponding $\bar{\rho}$ for $z_{t}$ as $\bar{\rho}_{t}$. By~\eqref{eqn:EI-sigma-star-pm-pf-25},
   \begin{equation} \label{eqn:EI-sigma-star-pm-pf-26}
   \centering
   \begin{aligned}
     \bar{\tau}(\rho_{t};z_{t},\eta_t,\alpha_t) -\tau(z_{t}) >& (z_{t}+ \eta_t^{1/2}) \Phi\left( (z_{t}+ \eta_t^{1/2})\bar{\rho}_{t}-\alpha_t^{1/2}\right) - z_{t}\Phi(z_{t}) -\phi(z_{t})\\
    \end{aligned}
   \end{equation}
If $z_{t} \in \left(\frac{\eta_t^{1/2}\bar{\rho}_{t}-\alpha_t^{1/2}}{1-\bar{\rho}_{t}} ,0\right)$, we know $(z_{t}+ \eta_t^{1/2})\bar{\rho}_{t}-\alpha_t^{1/2}<z_{t}$. 
   Thus,~\eqref{eqn:EI-sigma-star-pm-pf-26} becomes
   \begin{equation} \label{eqn:EI-sigma-star-pm-pf-27}
   \centering
   \begin{aligned}
     \bar{\tau}(\rho_{t};z_{t},\eta_t,\alpha_t) -\tau(z_{t}) 
      \geq& z_{t}[\Phi( (z_{t}+ \eta_t^{1/2})\bar{\rho}_{t}-\alpha_t^{1/2}) -\Phi(z_{t})] \\
      &+ \eta_t^{1/2} \Phi\left( (z_{t}+ \eta_t^{1/2})\bar{\rho}_{t}-\alpha_t^{1/2}\right)  -\phi(z_{t})\\
      \geq&  \eta_t^{1/2} \Phi\left( (z_{t}+ \eta_t^{1/2})\bar{\rho}_{t}-\alpha_t^{1/2}\right)  - \phi(0)>\eta_t^{1/2}\Phi(-\alpha_t^{1/2})-\phi(0) \geq 0,
    \end{aligned}
   \end{equation}
   where the last inequality is from Lemma~\ref{lem:EI-regret-parameters}.
   
   If $z_{t} \leq \frac{\eta_t^{1/2}\bar{\rho}_{t}-\alpha_t^{1/2}}{1-\bar{\rho}_{t}}$, we have $(z_{t}+ \eta_t^{1/2})\bar{\rho}_{t}-\alpha_t^{1/2}\geq z_{t}$. 
   Suppose now $z_{t}>-\eta_t^{1/2}+1$. By~\eqref{eqn:EI-sigma-star-pm-pf-26} and Lemma~\ref{lem:tauvsPhi},
   \begin{equation} \label{eqn:EI-sigma-star-pm-pf-28}
   \centering
   \begin{aligned}
     \bar{\tau}(\rho_{t};z_{t},\eta_t,\alpha_t) -\tau(z_{t}) 
       \geq& (z_{t}+ \eta_t^{1/2}) \Phi\left( z_{t}\right) - \tau(z_{t})>\Phi(z_{t})-\tau(z_{t}) >0.\\
    \end{aligned}
   \end{equation}
   On the other hand, we note that $\eta_t^{1/2}>\alpha_t^{1/2}+1$ by definition~\eqref{def:eta}.
   Thus, if $z_{t}\leq -\eta_t^{1/2}+1$, we have $z_{t}<-\alpha_t^{1/2}$. 
   By~\eqref{eqn:EI-sigma-star-pm-pf-bartau} and the monotonicity of $\tau(\cdot)$, 
   \begin{equation} \label{eqn:EI-sigma-star-pm-pf-29}
   \centering
   \begin{aligned}
     \bar{\tau}(\rho_{t};z_{t},\eta_t,\alpha_t) -\tau(z_{t}) 
       \geq& \frac{\eta_t^{1/2}}{\alpha_t^{1/2}} \tau\left( -\alpha_t^{1/2}\right) - \tau(z_{t})> \frac{\eta_t^{1/2}}{\alpha_t^{1/2}} \tau\left( -\alpha_t^{1/2}\right) - \tau(-\alpha_t^{1/2})>0.\\
    \end{aligned}
   \end{equation}
    Combining all three cases~\eqref{eqn:EI-sigma-star-pm-pf-27},~\eqref{eqn:EI-sigma-star-pm-pf-28} and~\eqref{eqn:EI-sigma-star-pm-pf-29},~\eqref{eqn:EI-sigma-star-pm-pf-26} leads to
   \begin{equation} \label{eqn:EI-sigma-star-pm-pf-30}
   \centering
   \begin{aligned}
     \bar{\tau}(\rho_{t};z_{t},\eta_t,\alpha_t) -\tau(z_{t})>0. 
    \end{aligned}
   \end{equation}
   From~\eqref{eqn:EI-sigma-star-pm-pf-17} and~\eqref{eqn:EI-sigma-star-pm-pf-bartau}, we again have~\eqref{eqn:EI-sigma-star-pm-pf-22} for \textbf{Case 2.1.1.2}.
  Therefore,~\eqref{eqn:EI-sigma-star-pm-pf-22} is valid for \textbf{Case 2.1.1}. 

   \textbf{Case 2.1.2} 
   Next, consider $z_{t}<- \eta_t^{1/2} $. By~\eqref{eqn:EI-sigma-star-pm-pf-16}, we know 
  $\frac{b_{t}}{b^*_{t}}\leq \frac{\alpha_t^{1/2}}{\eta_t^{1/2}}<1$.
    Using the monotonicity of $\tau(\cdot)$,~\eqref{eqn:EI-sigma-star-pm-pf-17} implies  
   \begin{equation} \label{eqn:EI-sigma-star-pm-pf-31}
   \centering
   \begin{aligned}
            EI(a^*_{t},b^*_{t}) \geq\frac{1}{\rho_{t}}b_{t} \tau\left( (z_{t}+\eta_t^{1/2})\rho_{t}-\alpha_t^{1/2}\right)
             \geq& \frac{\eta_t^{1/2}}{\alpha_t^{1/2}} b_{t} \tau\left( (z_{t}+\eta_t^{1/2})\frac{\alpha_t^{1/2}}{\eta_t^{1/2}}-\alpha_t^{1/2} \right) \\
             > & b_{t} \tau\left(z_{t} \right)= EI(a_{t},b_{t}).
    \end{aligned}
   \end{equation}
   That is, we have~\eqref{eqn:EI-sigma-star-pm-pf-22} for \textbf{Case 2.1.2}.
   Combining \textbf{Case 2.1.1} and \textbf{Case 2.1.2}, we have $EI(a^*_{t},b^*_{t})>EI(a_{t},b_{t})$ for \textbf{Case 2.1}. 
   
\textbf{Case 2.2}
   Next, we consider $a_{t}\geq 0$.  
   By definition~\eqref{eqn:EI-sigma-star-pm-pf-11}, 
   \begin{equation} \label{eqn:EI-sigma-star-pm-pf-32}
   \centering
   \begin{aligned}
            EI(a^*_{t},b^*_{t}) \geq& EI((c_1(t)+1)a_{t}+\eta_t^{1/2} b_{t} - \alpha_t^{1/2} b^*_{t},b^*_{t}) 
              \geq b^*_{t} 
\tau\left((c_1(t) z_{t}+\eta_t^{1/2}) \rho_{t} -\alpha_t^{1/2} \right)\\
             =& b_{t} \frac{1}{\rho_{t}} \tau\left( c_1(t) z_{t}\rho_{t}+ \eta_t^{1/2}\rho_{t}-\alpha_t^{1/2}\right),
    \end{aligned}
   \end{equation}
   where $\rho_{t}=\frac{b_{t}}{b^*_{t}}$.
    By~\eqref{eqn:EI-sigma-star-pm-pf-16}, we again have 
       $\eta_t^{1/2} \rho_{t}- \alpha_t^{1/2} \leq 0$.
  Define function $\tilde{\tau}:\Rbb^2\to\Rbb$ as 
   \begin{equation} \label{eqn:EI-sigma-star-pm-pf-tildetau}
   \centering
   \begin{aligned}
    \tilde{\tau}(\rho,z;\eta_t,\alpha_t) = 
\frac{1}{\rho} \tau\left( c_1(t) z\rho+ \eta_t^{1/2}\rho-\alpha_t\right),
    \end{aligned}
   \end{equation}
   for $\rho\in (0,\frac{\alpha_t^{1/2}}{\eta_t^{1/2}}), z\geq 0$.
   From~\eqref{eqn:EI-sigma-star-pm-pf-32}, $EI(a^*_{t},b^*_{t})\geq b_{t}\tilde{\tau}(\rho_{t},z_{t};\eta_t,\alpha_t)$.
    We now compare $ \tilde{\tau}(\rho_{t},z_{t};\eta_t,\alpha_t)$ and $\tau(z_{t})$. For $\forall \rho_{t}\in \left(0,\frac{\alpha_t^{1/2}}{\eta_t^{1/2}}\right)$,  Lemma~\ref{lem:EI-regret-parameters} and Lemma~\ref{lem:tau} lead to the derivative 
   \begin{equation} \label{eqn:EI-sigma-star-pm-pf-34}
   \centering
   \begin{aligned}
         \frac{\partial \tilde{\tau}}{\partial z}|_{(\rho_{t},z_{t}) } = c_1(t) \Phi\left(c_1(t)z_{t}\rho_{t}+\eta_t^{1/2}\rho_{t}-\alpha_t^{1/2}\right)>c_1(t)\Phi(-\alpha_t^{1/2})\geq 1 > \Phi(z_{t})=\frac{d \tau}{d z}|_{z_{t}}.
    \end{aligned}
   \end{equation}
   Therefore, both $\tilde{\tau}(\cdot,\cdot)$ and $\tau(\cdot)$ are monotonically increasing with $z_{t}$ and $\tilde{\tau}$ always has a larger positive derivative. For any $\rho_{t}$, this leads to
   \begin{equation} \label{eqn:EI-sigma-star-pm-pf-35}
   \centering
   \begin{aligned}
      \tilde{\tau}(\rho_{t},z_{t};\eta_t,\alpha_t) - \tau(z_{t}) \geq \tilde{\tau}(\rho_{t},0;\eta_t,\alpha_t)-\tau(0)=\frac{1}{\rho_{t}}\tau\left(\eta_t^{1/2}\rho_{t}-\alpha_t^{1/2} \right)-\tau(0).
    \end{aligned}
    \end{equation}
  Notice that the right-hand side of~\eqref{eqn:EI-sigma-star-pm-pf-35} can be written via~\eqref{eqn:EI-sigma-star-pm-pf-bartau} as
   \begin{equation} \label{eqn:EI-sigma-star-pm-pf-36}
   \centering
   \begin{aligned}
      \bar{\tau}(\rho_{t};0,\eta_t,\alpha_t)-\tau(0).
    \end{aligned}
    \end{equation}
   Therefore, we can follow the analysis in \textbf{Case 2.1.1} with $z=0$. Notice that $-\alpha_t^{1/2}\Phi(0)+\phi(0)<0$ for $\forall \alpha_t^{1/2}>1$. Thus,~\eqref{eqn:EI-sigma-star-pm-pf-36} follows \textbf{Case 2.1.1.2} where a minimum exists at $\bar{\rho}_t\in (0,\frac{\alpha_t^{1/2}}{\eta_t^{1/2}})$ which satisfies $\alpha_t^{1/2}\Phi(\eta_t^{1/2}\bar{\rho}_t-\alpha_t^{1/2})= \phi(\eta_t^{1/2}\bar{\rho}_t-\alpha_t^{1/2})$. Hence, we have  
   \begin{equation} \label{eqn:EI-sigma-star-pm-pf-37}
   \centering
   \begin{aligned}
      \bar{\tau}(\rho_{t};0,\eta_t,\alpha_t)-\tau(0)=& \frac{1}{\rho_{t}}\tau\left(\eta_t^{1/2}\rho_{t}-\alpha_t^{1/2} \right)-\tau(0)> \frac{1}{\bar{\rho}_t} \tau(\eta_t^{1/2}\bar{\rho}_t-\alpha_t^{1/2}) - \phi(0) \\ 
=&\frac{1}{\bar{\rho}_t} ((\eta_t^{1/2}\bar{\rho}_t-\alpha_t^{1/2}) \Phi(\eta_t^{1/2}\bar{\rho}_t-\alpha_t^{1/2})+\phi(\eta_t^{1/2}\bar{\rho}_t-\alpha_t^{1/2}))-\phi(0)\\
         =& \eta_t^{1/2} \Phi(\eta_t^{1/2}\bar{\rho}_t-\alpha_t^{1/2})-\phi(0) > \eta_t^{1/2}\Phi(-\alpha_t^{1/2})-\phi(0).
    \end{aligned}
    \end{equation}
  Given  Lemma~\ref{lem:EI-regret-parameters}, 
   \begin{equation} \label{eqn:EI-sigma-star-pm-pf-38}
   \centering
   \begin{aligned}
      \bar{\tau}(\rho_{t};0,\eta_t,\alpha_t)-\tau(0) > \eta_t^{1/2} \Phi(-\alpha_t^{1/2})-\phi(0) > 0.
    \end{aligned}
    \end{equation}
   By~\eqref{eqn:EI-sigma-star-pm-pf-32},~\eqref{eqn:EI-sigma-star-pm-pf-tildetau},~\eqref{eqn:EI-sigma-star-pm-pf-35},~\eqref{eqn:EI-sigma-star-pm-pf-36}, and~\eqref{eqn:EI-sigma-star-pm-pf-38}, we have 
   \begin{equation} \label{eqn:EI-sigma-star-pm-pf-39}
   \centering
   \begin{aligned}
            EI(a^*_{t},b^*_{t})- EI(a_{t},b_{t})\geq b_{t} \tilde{\tau}(\rho_{t},0;\eta_t,\alpha_t)-b_{t}\tau(0)> 0.
    \end{aligned}
   \end{equation}

    To summarize, under the \textbf{Scenario B Case 2}, for all cases given~\eqref{eqn:EI-sigma-star-pm-pf-16}, 
   $EI(a_{t}^*,b_{t}^*)>EI(a_{t},b_{t})$.
   Combined with~\eqref{eqn:EI-sigma-star-pm-pf-15}, under \textbf{Scenario B}, $EI(a_{t}^*,b_{t}^*)>EI(a_{t},b_{t})$ with probability $\geq 1-\delta_{\sigma}$.
   Combined with~\eqref{eqn:EI-sigma-star-pm-pf-13}, with probability $\geq 1-\frac{1}{2t^{\frac{\alpha}{2}}}-\delta_{\sigma}$, $EI_{t-1}([\xbm^*]_t)>EI_{t-1}(\xbm_{t})$. However, this is a contradiction of~\eqref{eqn:EI-sigma-star-pm-pf-4}, a sure event.
   Therefore, we have   
   \begin{equation*} \label{eqn:EI-sigma-star-pm-pf-40}
   \centering
   \begin{aligned}
 1-\frac{1}{2t^{\frac{\alpha}{2}}} -\delta_{\sigma}\leq 0.
    \end{aligned}
   \end{equation*}
   That is, $\delta_{\sigma}\geq 1-\frac{1}{2t^{\frac{\alpha}{2}}}$. From~\eqref{eqn:EI-sigma-star-pm-pf-premise},
   \begin{equation} \label{eqn:EI-sigma-star-pm-pf-41}
   \centering
   \begin{aligned}
           \sigma_{t-1}(\xbm_{t})\geq \sigma_{t-1}([\xbm^*]_t),
    \end{aligned}
   \end{equation}
   with probability $\geq 1-\frac{1}{2 t^{\frac{\alpha}{2}}}$.
   Under \textbf{Scenario B}, by~\eqref{eqn:EI-sigma-star-pm-pf-41} and~\eqref{eqn:EI-sigma-star-pm-pf-1}, with probability $\geq 1-\frac{1}{2 t^{\frac{\alpha}{2}}}$,  
   \begin{equation} \label{eqn:EI-sigma-star-pm-pf-42}
   \centering
   \begin{aligned}
       r_t \leq (c_{\xi}(t)+\phi(0)+(c_{\alpha}+2) \beta_t^{1/2})\sigma_{t-1}(\xbm_{t})+\frac{1}{t^2}.
     \end{aligned}
   \end{equation}
    By~\eqref{eqn:EI-sigma-star-assp}, $(\zeta_t^{1/2}+4)\beta_t^{1/2} > c_{\xi}(t)+\phi(0)+(c_{\alpha}+2) \beta_t^{1/2}$. 
    Combining \textbf{Scenario A}~\eqref{eqn:EI-sigma-star-pm-pf-9} and \textbf{Scenario B}~\eqref{eqn:EI-sigma-star-pm-pf-42}, we obtain~\eqref{eqn:EI-sigma-star-1}. 
\end{proof}

%% file: Sections/Appx-regret.tex
\section{Cumulative regret bound proof}\label{se:cumu-regret-proof}
In this section, we provide the proofs of the cumulative regret bounds in Section~\ref{se:regret}.  
\subsection{BPMI regret proof}\label{se:bmpi-regret-proof}
Since $\mu_{t-1}^+-\mu_{t-1}(\xbm)\leq 0$ for $\forall \xbm\in C$, the following corollary is a direct result of Theorem~\ref{theorem:EI-sigma-star} for BPMI.
\begin{cor}\label{cor:EI-sigma-star-pm}
   For any filtration $\mathcal{F}_{t-1}'$, if $E^r(t)$ is true,
   then,  with probability $\geq 1-\frac{1}{2 t^{\frac{\alpha}{2}}}$,
      \begin{equation} \label{eqn:EI-sigma-star-pm-1}
  \centering
  \begin{aligned}
      r_t  \leq c_2(t) \sigma_{t-1}(\xbm_{t})+\frac{1}{t^2}, 
   \end{aligned}
  \end{equation}
   for BPMI, where the parameters are in Theorem~\ref{theorem:EI-sigma-star}.
\end{cor}

The remainder of the section follows similar notations and analysis framework to GP-TS in~\cite{chowdhury2017kernelized}. 
We consider the conditional expectation of $r_t$ in the following lemma.
\begin{lem}\label{lem:ei-pm-r-exp}
   For any filtration $\mathcal{F}_{t-1}'$, if $E^r(t)$ is true,
   \begin{equation} \label{eqn:ei-pm-r-exp-1}
  \centering
  \begin{aligned}
      \Ebb[r_t|\mathcal{F}'_{t-1}] \leq (\eta_t^{1/2}+4\beta_t^{1/2})\sigma_{t-1}(\xbm_{t})+\frac{1}{t^2} + \frac{B}{t^{\frac{\alpha}{2}}}.
   \end{aligned}
  \end{equation} 
\end{lem}
\begin{proof}
   We note that $\xbm_t$ is deterministic given $\mathcal{F}_{t-1}'$.  
   Define the event $E^{\sigma}(t)$ to be when~\eqref{eqn:EI-sigma-star-pm-1} holds. 
   Since $E^r(t)$ is true, by Corollary~\ref{cor:EI-sigma-star-pm}, $E^{\sigma}(t)$ is true with probability $\geq 1-\frac{1}{2 t^{\frac{\alpha}{2}}} $. 
  Since $f$ is bounded, $r_t = f(\xbm_{t})-f(\xbm^*) \leq 2 \sup_{\xbm\in C} |f(\xbm)|\leq 2B$. 
  Taking the conditional expectation of $r_t$ for the event $E^{\sigma}(t)$ and considering $\overline{E^{\sigma}(t)}$, we have  
   \begin{equation} \label{eqn:ei-pm-r-exp-pf-1}
  \centering
  \begin{aligned}
     \Ebb\left[ r_t|\mathcal{F}'_{t-1}\right] 
                   \leq&  \Ebb\left[  c_{2}(t)\sigma_{t-1}(\xbm_{t})+\frac{1}{t^2} |\mathcal{F}'_{t-1}\right]+2B\Pbb\left[ \overline{E^{\sigma_t}(t)}|\mathcal{F}'_{t-1}\right]\\
           \leq&  c_{2}(t) \sigma_{t-1}(\xbm_{t})+\frac{1}{t^2} + \frac{B}{t^{\frac{\alpha}{2}}}.\\
   \end{aligned}
  \end{equation} 
  Using the definition of $c_2(t)=\eta_t^{1/2}+4\beta_t^{1/2}$ completes the proof.
\end{proof}

We now define several quantities based on the probability of $E^r(t)$, a standard step in the GP-TS analysis framework in~\cite{chowdhury2017kernelized}:  
    \begin{equation} \label{eqn:ei-pm-def-2}
  \centering
  \begin{aligned}
   \bar{r}_t =& r_t \cdot I\{E^r(t)\},\\
  X_t =& \bar{r}_t -  (\eta_t^{1/2}+4\beta_t^{1/2})\sigma_{t-1}(\xbm_t) -\frac{B}{t^{\frac{\alpha}{2}}}-\frac{1}{t^{2}},\\
     Y_t =& \sum_{i=1}^t X_i,
    \end{aligned}
  \end{equation} 
where $Y_0=0$ and $I\{\cdot\}$ is an indication function that takes value of $1$ if $E^r(t)$ is true and $0$ otherwise. 
To proceed, we use the super-martingale definition~\ref{def:supmart} and properties from Lemma~\ref{lem:supmart}.

\begin{lem}\label{lem:ei-pm-superm}
     The sequence $Y_t$, $t=0,\dots,T$, is a super-martingale with respect to $\mathcal{F}_{t-1}'$.
\end{lem}
\begin{proof}
   From the definition of $Y_t$, 
    \begin{equation} \label{eqn:ei-bpmi-superm-pf-1}
  \centering
  \begin{aligned}
   &\Ebb[Y_t-Y_{t-1}|\mathcal{F}'_{t-1}] = \Ebb[X_t|\mathcal{F}'_{t-1}] \\
      =& \Ebb\left[\bar{r}_t - (\eta_t^{1/2}+4\beta_t^{1/2})\sigma_{t-1}(\xbm_t) -\frac{B}{t^{\frac{\alpha}{2}}}-\frac{1}{t^{2}}|\mathcal{F}'_{t-1}\right] \\
   \end{aligned}
  \end{equation}
  If $E^{r}(t)$ is false, $\bar{r}_t = 0$ and~\eqref{eqn:ei-bpmi-superm-pf-1} $\leq 0$.
  If $E^{r}(t)$ is true, by Lemma~\ref{lem:ei-pm-r-exp},~\eqref{eqn:ei-bpmi-superm-pf-1} $\leq 0$.
\end{proof}

The proof of Lemma~\ref{lemma:ei-pm-regret-1} is next.
\begin{proof}
      From the definition of $Y_t$, we have
     \begin{equation} \label{eqn:ei-pm-regret-pf-1}
  \centering
  \begin{aligned}
           |Y_t-Y_{t-1}| = |X_t| \leq |\bar{r}_t| + (\eta_t^{1/2}+4\beta_t^{1/2})\sigma_{t-1}(\xbm_t) +\frac{B}{t^{\frac{\alpha}{2}}}+\frac{1}{t^{2}}.
    \end{aligned}
  \end{equation} 
  Using the upper bound $\bar{r}_t\leq r_t \leq 2B$, $\sigma_{t-1}(\xbm)\leq 1$, and $\eta_t^{1/2}>1$, we can write
     \begin{equation} \label{eqn:ei-pm-regret-pf-2}
  \centering
  \begin{aligned}
           |Y_t-Y_{t-1}| \leq& 2B+  (\eta^{1/2}_t +4\beta_t^{1/2}) + \frac{B}{t^{\frac{\alpha}{2}}}+\frac{1}{t^2}
                  \leq 3B+ 5 \eta_t^{1/2} + 1\\
                  \leq& (3B+1)\eta_t^{1/2}+ 5 \eta_t^{1/2}=c_B \eta_t^{1/2},
    \end{aligned}
  \end{equation}
  where $c_B=3B+6$. The second inequality in~\eqref{eqn:ei-pm-regret-pf-2} uses $t\geq 1$ and $\beta_t\leq \eta_t$. 

   Applying Lemma~\ref{lem:supmart} with $3\delta/4$, we have with probability $\geq 1-3\delta/4$,
     \begin{equation} \label{eqn:ei-pm-regret-pf-3}
  \centering
  \begin{aligned}
     \sum_{t=1}^T\bar{r}_t \leq&  \sum_{t=1}^T (\eta_t^{1/2}+4\beta_t^{1/2})\sigma_{t-1}(\xbm_t) +\sum_{t=1}^T \frac{1}{t^2}  
       +\sum_{t=1}^T \frac{B}{t^{\frac{\alpha}{2}}}+\sqrt{2\log(\frac{4}{3\delta}) \sum_{t=1}^T c_B^2 \eta_t}\\
     \leq& (\eta_T^{1/2}+4\beta_T^{1/2})\sum_{t=1}^T \sigma_{t-1}(\xbm_t) + \frac{\pi^2}{6}  + 2 B T^{1-\frac{\alpha}{2}} 
       + \eta_T^{1/2}c_B\sqrt{2\log(\frac{4}{3\delta}) T }.\\
    \end{aligned}
  \end{equation}
  The second inequality in~\eqref{eqn:ei-pm-regret-pf-3} uses $\sum_{t=1}^T\frac{1}{t^{\frac{\alpha}{2}}}\leq 2T^{1-\frac{\alpha}{2}}$, which can be proven via the integral of function $\frac{1}{x^{\frac{\alpha}{2}}}$ for $x\in [1,T]$.
  Let $ c_{ B}(\delta)= c_B\sqrt{2\log(\frac{4}{3\delta})}$.
   By Lemma~\ref{lem:eftprob}, $r_t=\bar{r}_t$ with probability $\geq 1-\delta/4$. Since $R_T = \sum_{t=1}^T r_t$, we obtain~\eqref{eqn:ei-pm-regret-1} with union bound. 
\end{proof}

The proof of Theorem~\ref{thm:pm-cumu-regret}, the cumulative regret bound for BPMI, is presented next.
\begin{proof}
    From Lemma~\ref{lem:variancebound}, we know $\sum_{t=1}^T\sigma_{t-1}(\xbm_t) = \mathcal{O}(\sqrt{T\gamma_T})$. By~\eqref{def:eta}, we can write 
    \begin{equation} \label{eqn:ei-pm-cumu-pf-1}
  \centering
  \begin{aligned}
       \beta_T^{1/2} = \mathcal{O}( \log^{\frac{1}{2}}(T)), \
       \eta_T^{1/2}= \mathcal{O}(\beta_T^{1/2}\zeta_T^{1/2}) = \mathcal{O}(T^{\frac{\alpha}{2}} \log^{\frac{1}{2}}(T)). 
    \end{aligned}
  \end{equation} 
     From Lemma~\ref{lemma:ei-pm-regret-1}, the regret bound is of 
    \begin{equation} \label{eqn:ei-pm-cumu-pf-2}
  \centering
  \begin{aligned}
         \mathcal{O}(R_T)=&\mathcal{O} ( T^{\frac{\alpha}{2}} \log^{\frac{1}{2}}(T) \sqrt{T\gamma_T}+T^{1-\frac{\alpha}{2}}+ T^{\frac{\alpha}{2}} \log^{\frac{1}{2}}(T) \sqrt{T})\\ 
              =& \mathcal{O}(\max\{T^{\frac{\alpha}{2}+\frac{1}{2}}\log^{\frac{1}{2}}(T) \sqrt{\gamma_T}, T^{1-\frac{\alpha}{2}}\}).\\
    \end{aligned}
  \end{equation} 
  Using $\gamma_T=\mathcal{O}(\log^{d+1}(T))$ for SE kernel, we choose $\alpha=\frac{1}{2}$ and $R_T=\mathcal{O}(T^{\frac{3}{4}} \log^{\frac{d+2}{2}}(T))$.
  Using $\gamma_T=\mathcal{O}(T^{\frac{d}{2\nu+d}}\log^{\frac{2\nu}{2\nu+d}}(T))$ for Matérn kernel from~\cite{vakili2021information}, we choose $\alpha=\frac{\nu}{2\nu+d}$. Thus, $R_T=\mathcal{O}(T^{\frac{3\nu+2d}{4\nu+2d}}\log^{\frac{4\nu+d}{4\nu+2d}}(T))$.
\end{proof}

\subsection{BSPMI regret proof}\label{se:bsmpi-regret-proof}
The general analysis framework is similar to that of BPMI in~\ref{se:bmpi-regret-proof}.
\begin{lem}\label{lem:bspm-r-exp}
   For any filtration $\mathcal{F}_{t-1}'$, if $E^r(t)$ is true,
   \begin{equation} \label{eqn:bspm-r-exp-1}
  \centering
  \begin{aligned}
      \Ebb[r_t|\mathcal{F}'_{t-1}] \leq c_1(t) \max\{\mu^m_{t-1}-\mu_{t-1}(\xbm_{t}),0\} + c_2(t) \sigma_{t-1}(\xbm_{t})+\frac{1}{t^2} + \frac{B}{t^{\frac{\alpha}{2}}}.\\
   \end{aligned}
  \end{equation} 
\end{lem}
\begin{proof}
   Given $\mathcal{F}_{t-1}'$, $\xbm_t$ is deterministic. 
   Define the event $E^{\sigma}(t)$ to be when~\eqref{eqn:EI-sigma-star-1} holds. 
   Since $E^r(t)$ is true, by Theorem~\ref{theorem:EI-sigma-star}, $E^{\sigma}(t)$ is true with probability $\geq 1-\frac{1}{2t^{\frac{\alpha}{2}}}$. 
  Since $f$ is bounded, $r_t \leq f(\xbm_{t})-f(\xbm^*) \leq 2B$. 
  For any filtration $\mathcal{F}'_{t-1}$ such that $E^r(t)$ is true, taking the conditional expectation of $r_t$ for the event $E^{\sigma}(t)$ and considering $\overline{E^{\sigma}(t)}$, we have  
   \begin{equation*} \label{eqn:bspm-r-exp-pf-1}
  \centering
  \begin{aligned}
     \Ebb\left[ r_t|\mathcal{F}'_{t-1}\right] 
                   \leq&  \Ebb\left[ c_1(t)\max\{\mu^m_{t-1}-\mu_{t-1}(\xbm_{t}),0\} + c_2(t)\sigma_{t-1}(\xbm_{t})+\frac{1}{t^2} |\mathcal{F}'_{t-1}\right]+2B\Pbb\left[ \overline{E^{\sigma_t}(t)}|\mathcal{F}'_{t-1}\right]\\
           \leq& c_1(t) \max\{\mu^m_{t-1}-\mu_{t-1}(\xbm_{t}),0\} + c_2(t) \sigma_{t-1}(\xbm_{t})+\frac{1}{t^2} + \frac{B}{t^{\frac{\alpha}{2}}}.\\
   \end{aligned}
  \end{equation*}
\end{proof}

Similar to BPMI, we define several quantities based on probability of $E^r(t)$: 
    \begin{equation} \label{eqn:bspm-def-barr}
  \centering
  \begin{aligned}
   \bar{r}_t =& r_t \cdot I\{E^r(t)\},\\
  X_t =& \bar{r}_t - c_1(t)\max\{\mu_{t-1}^m-\mu_{t-1}(\xbm_{t}),0\} - c_2(t)\sigma_{t-1}(\xbm_t) -\frac{B}{t^{\frac{\alpha}{2}}}-\frac{1}{t^{2}},\\
     Y_t =& \sum_{i=1}^t X_i,
    \end{aligned}
  \end{equation} 
where $Y_0=0$. 

\begin{lem}
     The sequence $Y_t$, $t=0,\dots,T$, is a super-martingale with respect to $\mathcal{F}_t'$.
\end{lem}
\begin{proof}
   From the definition of $Y_t$, 
    \begin{equation} \label{eqn:ei-bspmi-superm-pf-1}
  \centering
  \begin{aligned}
   &\Ebb[Y_t-Y_{t-1}|\mathcal{F}'_{t-1}] = \Ebb[X_t|\mathcal{F}'_{t-1}] \\
      =& \Ebb\left[\bar{r}_t - c_1(t)\max\{\mu_{t-1}^m-\mu_{t-1}(\xbm_{t}),0\} - c_2(t)\sigma_{t-1}(\xbm_t) -\frac{B}{t^{\frac{\alpha}{2}}}-\frac{1}{t^{2}} |\mathcal{F}'_{t-1}\right] \\
   \end{aligned}
  \end{equation}
  If $E^{r}(t)$ is false, $\bar{r}_t = 0$ and~\eqref{eqn:ei-bspmi-superm-pf-1} $\leq 0$.
  If $E^{r}(t)$ is true, by Lemma~\ref{lem:bspm-r-exp},~\eqref{eqn:ei-bspmi-superm-pf-1} $\leq 0$.
\end{proof}

Next, we present the proof of Lemma~\ref{lemma:bspm-inst-regret}.
\begin{proof}
      From the definition of $Y_t$, we have
     \begin{equation} \label{eqn:bspm-inst-regret-pf-1}
  \centering
  \begin{aligned}
           |Y_t-Y_{t-1}| = |X_t| \leq |\bar{r}_t| + c_1(t)\max\{\mu_{t-1}^m-\mu_{t-1}(\xbm_{t}),0\} + c_2(t)\sigma_{t-1}(\xbm_t) +\frac{B}{t^{\frac{\alpha}{2}}}+\frac{1}{t^{2}}.
    \end{aligned}
  \end{equation} 
  Since $E^r(t)$ is true and $\sigma_{t-1}(\xbm)\leq 1$, we have
     \begin{equation} \label{eqn:bspm-inst-regret-pf-2}
  \centering
  \begin{aligned}
              \mu_{t-1}^m  -\mu_{t-1}(\xbm_{t})\leq& \mu_{t-1}(\xbm_{t-1})  -\mu_{t-1}(\xbm_{t})\leq f(\xbm_{t-1})  -f(\xbm_{t}) + \beta_t^{1/2}(\sigma_{t-1}(\xbm_{t-1})+\sigma_{t-1}(\xbm_{t}) )\\
           \leq& 2B+ 2\beta_t^{1/2}.
   \end{aligned}
  \end{equation} 
  Then, using $\bar{r}_t\leq r_t \leq 2B$ and $\sigma_{t-1}(\xbm)\leq 1$, we can write
     \begin{equation} \label{eqn:bspm-inst-regret-pf-3}
  \centering
  \begin{aligned}
           |Y_t-Y_{t-1}| \leq& 2B+  c_1(t) 2(B+\beta_t^{1/2}) +c_2(t)+ \frac{B}{t^{\frac{\alpha}{2}}}+\frac{1}{t^2}\\
       \leq& 2B+  2 c_1(t) (B+\beta_t^{1/2}) + 4\beta_t^{1/2}+\eta_t^{1/2}+B+1
       \leq c_B c_1(t) \beta_t^{1/2},
    \end{aligned}
  \end{equation}
 where $c_B=5B+8$. The last inequality in~\eqref{eqn:bspm-inst-regret-pf-3} uses $c_1(t)\geq 1$ and $\beta_t^{1/2}\leq \eta_t^{1/2}\leq c_1(t)$.
   Next, we consider the sum 
     \begin{equation} \label{eqn:bspm-inst-regret-pf-4}
  \centering
  \begin{aligned}
     &\sum_{t=1}^T \eta_t^{1/2} \max\{\mu_{t-1}^m-\mu_{t-1}(\xbm_{t}),0\} \leq \eta_T^{1/2} \sum_{t=1}^T \max\{\mu_{t-1}^m-\mu_{t-1}(\xbm_{t}),0\}.\\
    \end{aligned}
  \end{equation}
   Let $P_T\subseteq \{1,\dots,T\}$ be the ordered index set such that $\mu_{t-1}^m-\mu_{t-1}(\xbm_{t})>0$. 
   Then,
     \begin{equation} \label{eqn:bspm-inst-regret-pf-5}
  \centering
  \begin{aligned}
     &\sum_{t=1}^T \eta_t^{1/2} \max\{\mu_{t-1}^m-\mu_{t-1}(\xbm_{t}),0\} \leq \eta_T^{1/2} \sum_{i=1}^{|P_T|} (\mu_{t_i-1}^m-\mu_{t_i-1}(\xbm_{t_i})),
    \end{aligned}
  \end{equation}
  where $t_i\in P_T$ and $t_i<t_{i+1}$.
  By Definition~\ref{def:eft-bspmi} of $E^s_1(t)$ and $E^s_2(t)$,~\eqref{eqn:eft-2} with $\ubm_t=\xbm_t$ and $\ubm_t=\xbm_{t-1}$ are satisfied. Therefore,
     \begin{equation} \label{eqn:bspm-inst-regret-pf-6}
  \centering
  \begin{aligned}
         \mu^m_{t_i-1}-\mu_{t_i-1}&(\xbm_{t_i}) \leq \mu_{t_i-1}(\xbm_{t_i-1}) -\mu_{t_i-1}(\xbm_{t_i})\leq f(\xbm_{t_i-1}) +\beta_{t_i}^{1/2}\sigma_{t_i-1}(\xbm_{t_i-1})-f(\xbm_{t_i}) \\&+ \beta_{t_i}^{1/2}\sigma_{t_i-1}(\xbm_{t_i})
           \leq f(\xbm_{t_i-1}) -f(\xbm_{t_i})+\beta_{t_i}^{1/2}\sigma_{t_i-2}(\xbm_{t_i-1}) + \beta_{t_i}^{1/2}\sigma_{t_i-1}(\xbm_{t_i}).
   \end{aligned}
  \end{equation}
  The last inequality of~\eqref{eqn:bspm-inst-regret-pf-6} uses the monotonic decreasing property of $\sigma_t(\xbm)$ at a given $\xbm$ as $t$ increases~\citep{vivarelli1998studies}.
  Applying~\eqref{eqn:bspm-inst-regret-pf-5} and~\eqref{eqn:bspm-inst-regret-pf-6} to~\eqref{eqn:bspm-inst-regret-pf-4}, we have  
     \begin{equation} \label{eqn:bspm-inst-regret-pf-7}
  \centering
  \begin{aligned}
      \sum_{t=1}^T \eta_t^{1/2} &\max\{\mu_{t-1}^m-\mu_{t-1}(\xbm_{t}),0\} \leq  \eta_T^{1/2} \sum_{i=1}^{|P_T|}\left[f(\xbm_{t_i-1}) -f(\xbm_{t_i})+\beta_{t_i}^{1/2}\sigma_{t_i-1}(\xbm_{t_i-1}) + \beta_{t_i}^{1/2}\sigma_{t_i-1}(\xbm_{t_i})\right]\\
         \leq&\eta_T^{1/2} \sum_{i=1}^{|P_T|}\left[f(\xbm_{t_i-1}) -f(\xbm_{t_i})\right]+\eta_T^{1/2}\beta_{T}^{1/2} \left(\sum_{t=2}^{T}\sigma_{t-2}(\xbm_{t-1})+\sigma_{0}(\xbm_0)\right) + \eta_T^{1/2}\beta_{T}^{1/2} \sum_{t=1}^{T}\sigma_{t-1}(\xbm_{t})\\
        \leq& \eta_T^{1/2} 2B +2 \eta_T^{1/2}\beta_{T}^{1/2} \left(\sum_{t=1}^{T}\sigma_{t-1}(\xbm_{t})+1\right),\\
  \end{aligned}
 \end{equation}
 where the second line uses $P_T\subset \{1,\dots,T\}$ and $|P_T|\leq T$.
 
   Applying Lemma~\ref{lem:supmart} on $\bar{r}_t$ with $5\delta/8$ and~\eqref{eqn:bspm-inst-regret-pf-3}, we have with probability $\geq 1-5\delta/8$,
     \begin{equation} \label{eqn:bspm-inst-regret-pf-8}
  \centering
  \begin{aligned}
     \sum_{t=1}^T\bar{r}_t \leq& \sum_{t=1}^T c_1(t) \max\{\mu_{t-1}^m-\mu_{t-1}(\xbm_{t}),0\}+ \sum_{t=1}^T c_2(t)\sigma_{t-1}(\xbm_t) +\sum_{t=1}^T \frac{1}{t^2}  
       +\sum_{t=1}^T \frac{B}{t^{\frac{\alpha}{2}}}\\
       &+\sqrt{2\log(\frac{8}{5\delta}) \sum_{t=1}^T c_B^2c_1^2(t)\beta_t}\\
     \leq& 5\eta_T^{1/2}(B+\beta_T^{1/2}\sum_{t=1}^T\sigma_{t-1}(\xbm_t)+\beta_T^{1/2})+(\eta_T^{1/2}+4\beta_T^{1/2})\sum_{t=1}^T \sigma_{t-1}(\xbm_t)   + \frac{\pi^2}{6}  + 2B T^{1-\frac{\alpha}{2}}\\ 
       &+ c_1(T)c_B \beta_T^{1/2}\sqrt{2\log(\frac{8}{5\delta}) T}\\
     \leq& 5\eta_T^{1/2}(B+\beta_T^{1/2})+(5\beta_T^{1/2}\eta_T^{1/2}+\eta_T^{1/2}+4\beta_T^{1/2})\sum_{t=1}^T\sigma_{t-1}(\xbm_t)  + \frac{\pi^2}{6}\\  &+ 2B T^{1-\frac{\alpha}{2}} 
       + c_B(\delta)\eta_T^{1/2} \beta_T^{1/2}\sqrt{T},\\
    \end{aligned}
  \end{equation}
   where $c_B(\delta)= \sqrt{2\log(\frac{8}{5\delta})}c_B/\phi(0)$, and the second inequality uses~\eqref{eqn:bspm-inst-regret-pf-7}.
   In~\eqref{eqn:bspm-inst-regret-pf-8}, we use $\phi(0)<0.4$ and thus $c_1(t)\leq 2.5 \eta_t^{1/2}$.
   Since $R_T = \sum_{t=1}^T r_t$ and $r_t=\bar{r}_t$ with probability $\geq 1-3\delta/8$, we obtain~\eqref{eqn:bspm-inst-regret-1} with probability $\geq 1-\delta$ with union bound. 
\end{proof} 

The proof of Theorem~\ref{thm:bspmi-cumu-regret} is given next.
\begin{proof}
    From Lemma~\ref{lem:variancebound}, we know $\sum_{t=1}^T\sigma_{t-1}(\xbm_t) = \mathcal{O}(\sqrt{T\gamma_T})$. By~\eqref{def:eta}, we can write 
    \begin{equation*} \label{eqn:ei-bspm-cumu-pf-1}
  \centering
  \begin{aligned}
       \eta_T^{1/2}= \mathcal{O}(T^{\frac{\alpha}{2}} \log^{\frac{1}{2}}(T)), \ 
       \beta_T^{1/2} = \mathcal{O}( \log^{\frac{1}{2}}(T)).
    \end{aligned}
  \end{equation*} 
     From Lemma~\ref{lemma:bspm-inst-regret}, the regret bound is  
    \begin{equation} \label{eqn:ei-bspm-cumu-pf-2}
  \centering
  \begin{aligned}
         \mathcal{O}(R_T)=&\mathcal{O} ( T^{\frac{\alpha}{2}} \log(T) \sqrt{T\gamma_T}+T^{1-\frac{\alpha}{2}}+T^{\frac{\alpha}{2}} \log(T)\sqrt{T})\\ 
              =& \mathcal{O}(\max\{T^{\frac{1}{2}+\frac{\alpha}{2}} \log(T)\sqrt{\gamma_T}, T^{1-\frac{\alpha}{2}}\} ).\\
    \end{aligned}
  \end{equation} 
  Using $\gamma_T=\mathcal{O}(\log^{d+1}(T))$ for SE kernel, we choose $\alpha=\frac{1}{2}$ and $R_T=\mathcal{O}(T^{\frac{3}{4}} \log^{\frac{d+3}{2}}(T))$.
  Using $\gamma_T=\mathcal{O}(T^{\frac{d}{2\nu+d}}\log^{\frac{2\nu}{2\nu+d}}(T))$ for Matérn kernel, we choose $\alpha=\frac{\nu}{2\nu+d}$. Thus, $R_T=\mathcal{O}(T^{\frac{3\nu+2d}{4\nu+2d}}\log^{\frac{3\nu+d}{2\nu+d}}(T))$.
\end{proof}

\subsection{BOI regret proof}\label{se:boi-regret-proof}
\begin{lem}\label{lem:boi-r-exp}
   For any filtration $\mathcal{F}_{t-1}'$, if $E^r(t)$ and $E^y(t)$ are both true, then
   \begin{equation} \label{eqn:boi-r-exp-1}
  \centering
  \begin{aligned}
      \Ebb[r_t|\mathcal{F}'_{t-1}] \leq c_1(t) (y^+_{t-1}-f(\xbm_{t}))  + (c_1(t)(\beta_t^{1/2}+c_y(t))+c_2(t))\sigma_{t-1}(\xbm_{t})+\frac{1}{t^2} + \frac{B}{t^{\frac{\alpha}{2}}}.
   \end{aligned}
  \end{equation} 
\end{lem}
\begin{proof}
   Similar to the proof of Lemma~\ref{lem:bspm-r-exp}, given that $E^r(t)$ and $E^y(t)$ are true, 
   by Theorem~\ref{theorem:EI-sigma-star},~\eqref{eqn:EI-sigma-star-1} holds with probability $\geq 1-\frac{1}{2t^{\frac{\alpha}{2}}}$. Further, $r_t \leq f(\xbm_{t})-f(\xbm^*) \leq 2B$. 
  Taking the conditional expectation of $r_t$ for the event $E^{\sigma}(t)$, which is  when~\eqref{eqn:EI-sigma-star-1} holds, and considering $\overline{E^{\sigma}(t)}$, we have  
    \begin{equation} \label{eqn:ei-r-exp-pf-1}
  \centering
  \begin{aligned}
     \Ebb\left[ r_t|\mathcal{F}'_{t-1}\right] 
                   \leq&  \Ebb\left[ c_1(t)\max\{y^+_{t-1}-\mu_{t-1}(\xbm_{t}),0\} + c_2(t)\sigma_{t-1}(\xbm_{t})+\frac{1}{t^2} |\mathcal{F}'_{t-1}\right]\\
                   &+2B\Pbb\left[ \overline{E^{\sigma_t}(t)}|\mathcal{F}'_{t-1}\right]\\
           \leq& c_1(t) \max\{y^+_{t-1}-\mu_{t-1}(\xbm_{t}),0\} + c_2(t) \sigma_{t-1}(\xbm_{t})+\frac{1}{t^2} + \frac{B}{t^{\frac{\alpha}{2}}}.\\
   \end{aligned}
  \end{equation}
  From~\eqref{eqn:best-observation-instregret-pf-12}, we know that if $E^r(t)$ and $E^y(t)$ are true,
  $y^+_{t-1}-\mu_{t-1}(\xbm_{t})\geq -c_y(t)\sigma_{t-1}(\xbm_t)$ when $y^+_{t-1}-\mu_{t-1}(\xbm_{t})<0$. Thus, 
  \begin{equation} \label{eqn:ei-r-exp-pf-2}
  \centering
  \begin{aligned}
\max\{y^+_{t-1}-\mu_{t-1}(\xbm_{t}),0\} \leq& y^+_{t-1}-\mu_{t-1}(\xbm_{t})+c_y(t)\sigma_{t-1}(\xbm_t)\\
=& y^+_{t-1}-f(\xbm_t)+f(\xbm_t)-\mu_{t-1}(\xbm_{t})+c_y(t)\sigma_{t-1}(\xbm_t)\\
\leq& y^+_{t-1}-f(\xbm_t)+(c_y(t)+\beta_t^{1/2})\sigma_{t-1}(\xbm_t).
     \end{aligned}
  \end{equation}
  Using~\eqref{eqn:ei-r-exp-pf-2} in~\eqref{eqn:ei-r-exp-pf-1}, we have
  \begin{equation*} \label{eqn:ei-r-exp-pf-3}
  \centering
  \begin{aligned}
     \Ebb\left[ r_t|\mathcal{F}'_{t-1}\right] 
           \leq & c_1(t) (y^+_{t-1}-f(\xbm_{t})) + (c_1(t)c_y(t)+c_1(t)\beta_t^{1/2}+c_2(t)) \sigma_{t-1}(\xbm_{t})+\frac{1}{t^2} + \frac{B}{t^{\frac{\alpha}{2}}}.\\
   \end{aligned}
  \end{equation*}
\end{proof}

We now define several quantities based on probabilities of $E^r(t)$ and $E^y(t)$. 
    \begin{equation} \label{def:boi-def-1}
  \centering
  \begin{aligned}
   \bar{r}_t =& r_t \cdot I\{E^r(t)\}\cdot I\{E^y(t)\},\\
  X_t =& \bar{r}_t - c_1(t) (y_{t-1}^+-f(\xbm_{t})) - (c_1(t)\beta_t^{1/2}+c_1(t)c_y(t)+c_2(t))\sigma_{t-1}(\xbm_t) -\frac{B}{t^{\frac{\alpha}{2}}}-\frac{1}{t^{2}},\\
     Y_t =& \sum_{i=1}^t X_i,
    \end{aligned}
  \end{equation} 
where $Y_0=0$. 
   Further, we define the instantaneous regret when only $E^r(t)$ is true as 
    \begin{equation} \label{def:boi-def-2}
  \centering
  \begin{aligned}
   \tilde{r}_t =& r_t \cdot I\{E^r(t)\}.\\
    \end{aligned}
  \end{equation} 
\begin{lem}
     The sequence $Y_t$, $t=0,\dots,T$, is a super-martingale with respect to $\mathcal{F}_t'$.
\end{lem}
\begin{proof}
   From the definition of $Y_t$, 
    \begin{equation} \label{eqn:ei-boi-superm-pf-1}
  \centering
  \begin{aligned}
   &\Ebb[Y_t-Y_{t-1}|\mathcal{F}'_{t-1}] = \Ebb[X_t|\mathcal{F}'_{t-1}] \\
      =& \Ebb\left[\bar{r}_t - c_1(t) (y_{t-1}^+-f(\xbm_{t})) -(c_1(t)\beta_t^{1/2}+c_1(t)c_y(t)+c_2(t))\sigma_{t-1}(\xbm_t) -\frac{B}{t^{\frac{\alpha}{2}}}-\frac{1}{t^{2}}|\mathcal{F}'_{t-1}\right] \\
   \end{aligned}
  \end{equation}
  If $E^{r}(t)\cap E^y(t)$ is false, $\bar{r}_t = 0$ and~\eqref{eqn:ei-boi-superm-pf-1} $\leq 0$.
  If $E^{r}(t)\cap E^y(t)$ is true, by Lemma~\ref{lem:boi-r-exp},~\eqref{eqn:ei-boi-superm-pf-1} $\leq 0$.
\end{proof}

The proof of Lemma~\ref{lemma:boi-inst-regret} is given next.
\begin{proof}
      From the definition of $Y_t$, we have
     \begin{equation} \label{eqn:boi-inst-regret-pf-1}
  \centering
  \begin{aligned}
           |Y_t-Y_{t-1}| = |X_t| \leq& |\bar{r}_t| + c_1(t) (y_{t-1}^+-f(\xbm_{t}))+ (c_1(t)\beta_t^{1/2}+c_1(t)c_y(t)+c_2(t)) \sigma_{t-1}(\xbm_t) \\
           &+\frac{B}{t^{\frac{\alpha}{2}}}+\frac{1}{t^{2}}.
    \end{aligned}
  \end{equation} 
   Notice that 
     \begin{equation} \label{eqn:boi-inst-regret-pf-2}
  \centering
  \begin{aligned}
              y_{t-1}^+  -f(\xbm_{t})\leq  y_{t-1}  -f(\xbm_{t})
           =  f(\xbm_{t-1})+\epsilon_{t-1}-f(\xbm_{t})
           \leq& 2B+\epsilon_{t-1}.
   \end{aligned}
  \end{equation} 
  Then, using $\bar{r}_t\leq r_t \leq 2B$ and $\sigma_{t-1}(\xbm)\leq 1$,~\eqref{eqn:boi-inst-regret-pf-1} implies
     \begin{equation} \label{eqn:boi-inst-regret-pf-3}
  \centering
  \begin{aligned}
           |Y_t-Y_{t-1}| \leq& 2B+  c_1(t)(2B+\epsilon_{t-1})+ c_1(t)\beta_t^{1/2}+c_1(t)c_y(t)+c_2(t) + B+1\\
      \leq& 3B+ 1+c_1(t) 2B+c_1(t)|\epsilon_{t-1}|+ c_1(t)\beta_t^{1/2}+c_1(t)c_y(t)+4\beta_t^{1/2}+\eta^{1/2}_t  \\
      \leq& (5B+8) c_1(t) c_{\beta}(t)+ 2.5\eta_t^{1/2}|\epsilon_{t-1}|
      \leq c_B \eta_t^{1/2}c_{\beta}(t)+ 2.5\eta_t^{1/2}|\epsilon_{t-1}|,
    \end{aligned}
  \end{equation}
  where $c_B=12.5B+20$ and $c_{\beta}(t)=\max\{\beta_t^{1/2},c_y(t)\}$. The third inequality of~\eqref{eqn:boi-inst-regret-pf-3} uses $c_1(t)\geq \eta_t^{1/2} \geq 1$, $\eta_t>\beta_t$, and the last one uses $c_1(t)\leq 2.5\eta_t^{1/2}$.
  We note that for simplicity, the bounds are relaxed generously in the third inequality of~\eqref{eqn:boi-inst-regret-pf-3}, \textit{e.g.}, using $4\beta_t^{1/2}\leq 4 c_1(t)c_{\beta}(t)$.
  Consider the i.i.d. Gaussian noise $\epsilon_t\sim\mathcal{N}(0,\sigma^2)$. 
  Similar to Lemma~\ref{lem:fmu-t}, we have   
     \begin{equation} \label{eqn:boi-inst-regret-pf-4}
  \centering
  \begin{aligned}
           |\epsilon_t| \leq \sqrt{2\log(a\pi_t/\delta)} \sigma
    \end{aligned}
  \end{equation}
   with probability $\geq 1-\delta/a$ for all $t\in\Nbb$ and $a>0$.
  Using~\eqref{eqn:boi-inst-regret-pf-4} with $a=8$ in~\eqref{eqn:boi-inst-regret-pf-3}, we have 
      \begin{equation} \label{eqn:boi-inst-regret-pf-5}
  \centering
  \begin{aligned}
           |Y_t-Y_{t-1}|  \leq& c_B \eta_t^{1/2}c_{\beta}(t)+ 2.5\eta_t^{1/2} \sqrt{2\log(8\pi_t/\delta)} \sigma\\
             \leq&  \eta_t^{1/2}(c_{\beta}(t)c_B+ 2.5 \sqrt{2\log(8\pi_t/\delta)} \sigma)\leq c_{B\sigma}\eta_t^{1/2}c_{\beta}(t),
    \end{aligned}
  \end{equation}
  where $c_{B\sigma} = c_B+2.5\sigma$, 
  with probability $\geq 1-\delta/8$. The last inequality of~\eqref{eqn:boi-inst-regret-pf-5} uses $\sqrt{2\log(8\pi_t/\delta)} \leq 
  \beta_t^{1/2} \leq c_{\beta}(t)$.

   To obtain a bound on the sum of $\bar{r}_t$, we consider the term $\sum_{t=1}^T \eta_t^{1/2} (y_{t-1}^+-f(\xbm_{t}))$. 
  Using $y_t^+\leq y_t$ and $y_t=f(\xbm_t)+\epsilon_t$ for $\forall t\in\Nbb$, we have
     \begin{equation} \label{eqn:boi-inst-regret-pf-5.2}
  \centering
  \begin{aligned}
         y^+_{t-1}-f(\xbm_{t}) =& y_{t-1}^+- y^+_{t}+ y^+_{t} -f(\xbm_{t})\leq y_{t-1}^+- y^+_{t}+ y_{t} -f(\xbm_{t})\\
           \leq&y_{t-1}^+- y^+_{t}+ \epsilon_{t}.
   \end{aligned}
  \end{equation}
 By~\eqref{eqn:boi-inst-regret-pf-5.2}, we can write
  \begin{equation} \label{eqn:boi-inst-regret-pf-5.3}
  \centering
  \begin{aligned}
     &\sum_{t=1}^T \eta_t^{1/2} (y_{t-1}^+-f(\xbm_{t})) \leq \eta_T^{1/2} \sum_{t=1}^{T} (y_{t-1}^+-y_{t}^+)+\eta_T^{1/2} \sum_{t=1}^{T}\epsilon_{t}.
    \end{aligned}
  \end{equation}
  Using~\eqref{eqn:boi-inst-regret-pf-4} with $a=8$,~\eqref{eqn:boi-inst-regret-pf-5.3} implies that with probability $\geq 1-\delta/8$,
     \begin{equation} \label{eqn:boi-inst-regret-pf-6}
  \centering
  \begin{aligned}
     \sum_{t=1}^T \eta_t^{1/2} (y_{t-1}^+-f(\xbm_{t}))  
       \leq& \eta_T^{1/2} (y_{0}-y_{T}^+) + \eta_T^{1/2}\sum_{t=1}^{T}\epsilon_{t}
       \leq \eta_T^{1/2} (2B +\epsilon_{0}^+-\epsilon_{T}^+) + \eta_T^{1/2}\sum_{t=1}^{T}\epsilon_{t}\\
       \leq& \eta_T^{1/2} (2B + 2\sqrt{2\log(8\pi_t/\delta)}\sigma) + \eta_T^{1/2}\sum_{t=1}^{T}\epsilon_{t},\\
    \end{aligned}
  \end{equation}  
  where $\epsilon_T^+$ is the noise corresponding to $y_T^+$.
  It is well-known that the sum of Gaussian distributions is still a Gaussian distribution. Indeed, $\sum_{t=1}^T \epsilon_{t}\sim\mathcal{N}(0, T\sigma^2)$. 
  Thus, similar to Lemma~\ref{lem:fmu-t}, with probability $\geq 1-\delta/8$,
     \begin{equation} \label{eqn:boi-inst-regret-pf-7}
  \centering
  \begin{aligned}
       \sum_{t=1}^T \epsilon_{t} \leq \sqrt{2\log(8/\delta)}\sqrt{T}\sigma.
    \end{aligned}
  \end{equation} 
  Applying~\eqref{eqn:boi-inst-regret-pf-7} to~\eqref{eqn:boi-inst-regret-pf-6}, we have 
     \begin{equation} \label{eqn:boi-inst-regret-pf-8}
  \centering
  \begin{aligned}
     \sum_{t=1}^T \eta_t^{1/2} (y_{t-1}^+-f(\xbm_{t})) 
       \leq& \eta_T^{1/2} (2B + 2\sqrt{2\log(8\pi_t/\delta)}\sigma) + \eta_T^{1/2} \sqrt{2\log(8/\delta)}\sqrt{T}\sigma,\\
       \leq&  c_{B2} \eta_T^{1/2} \sqrt{2\log(8/\delta)}\sqrt{T}\sigma,\\
    \end{aligned}
  \end{equation}
  where $c_{B2} = 2B+3$, with probability $\geq 1-\delta/4$. The last inequality in~\eqref{eqn:boi-inst-regret-pf-8} uses $1<\sqrt{2\log(8\pi_t/\delta)}\leq \sqrt{2\log(8/\delta)}\sqrt{T}$.
   From Lemma~\ref{lem:supmart} with $\delta/4$,~\eqref{eqn:boi-inst-regret-pf-5}, and~\eqref{eqn:boi-inst-regret-pf-8}, we have with probability $\geq 1-5\delta/8$ (using union bound),
     \begin{equation} \label{eqn:boi-inst-regret-pf-9}
  \centering
  \begin{aligned}
     \sum_{t=1}^T\bar{r}_t \leq& \sum_{t=1}^T 2.5\eta_t^{1/2} (y_{t-1}^+-f(\xbm_{t}))+\sum_{t=1}^T (c_1(t)\beta_t^{1/2}+c_1(t)c_y(t)+c_2(t))\sigma_{t-1}(\xbm_t) +\sum_{t=1}^T \frac{1}{t^2}  \\
       &+\sum_{t=1}^T \frac{B}{t^{\frac{\alpha}{2}}}+\sqrt{2\log(4/\delta) \sum_{t=1}^T \eta_t c_{\beta}^2(t) c_{B\sigma}^2}\\
     \leq& 2.5c_{B2}\eta_T^{1/2} \sqrt{2\log(8/\delta)}\sqrt{T}\sigma+ 10\eta_T^{1/2}c_{\beta}(T)\sum_{t=1}^T \sigma_{t-1}(\xbm_t) + \frac{\pi^2}{6}  + 2B T^{1-\frac{\alpha}{2}} \\
       &+ \eta_T^{1/2} c_{\beta}(T) c_{B\sigma} \sqrt{2\log(4/\delta) T}.\\
    \end{aligned}
  \end{equation}
  The second inequality in~\eqref{eqn:boi-inst-regret-pf-9} uses $c_1(t)\leq 2.5\eta_t^{1/2}$, $\beta_t^{1/2}\geq 1$, $c_2(t)\leq 5 \eta_t^{1/2}\beta_t^{1/2}$, and thus 
  \begin{equation*} \label{eqn:boi-inst-regret-pf-10}
  \centering
  \begin{aligned}
  c_1(t)\beta_t^{1/2}+c_1(t)c_y(t)+c_2(t) \leq& 2.5\eta_t^{1/2}\beta_t^{1/2} + 2.5\eta_t^{1/2}c_y(t)+5\eta_t^{1/2}\beta_t^{1/2}\\
  =& 7.5\eta_t^{1/2}\beta_t^{1/2}+2.5\eta_t^{1/2}c_y(t)
  \leq 10\eta_t^{1/2} c_{\beta}(t).
  \end{aligned}
  \end{equation*}
   Let $c_{B}(\delta)=2.5c_{B2}\sqrt{2\log(8/\delta)}$, $c_3=10$, and $c_{B\sigma}(\delta)=c_{B\sigma}\sqrt{2\log(4/\delta)}$. 
  By~\eqref{eqn:boi-inst-regret-pf-8}, with probability $\geq 1-5\delta/8$, we have 
     \begin{equation} \label{eqn:boi-inst-regret-pf-11}
  \centering
  \begin{aligned}
     \sum_{t=1}^T\bar{r}_t 
     \leq& 
      c_{B}(\delta) \eta_T^{1/2} \sqrt{T}\sigma+ c_3\eta_T^{1/2} c_{\beta}(T)\sum_{t=1}^T \sigma_{t-1}(\xbm_t) + \frac{\pi^2}{6}  +2 B T^{1-\frac{\alpha}{2}} 
       + \eta_T^{1/2}c_{\beta}(T)c_{B\sigma}(\delta)\sqrt{T}.\\
    \end{aligned}
  \end{equation}
   Next, we consider the case when $E^r(t)$ is true, but $y^+_{t-1}-f(\xbm^*)\leq \beta_t^{1/2}\sigma_{t-1}(\xbm_t)$. 
   The instantaneous regret $r_t$ can be bounded by $2B$. Therefore, by definition~\eqref{def:boi-def-2}, 
     \begin{equation} \label{eqn:boi-inst-regret-pf-12}
  \centering
  \begin{aligned}
     \sum_{t=1}^T\tilde{r}_t  \leq \sum_{t=1}^T\bar{r}_t  + 2B n_y(T),
    \end{aligned}
  \end{equation}
  where $0\leq n_y(T)\leq T$ and $n_y(T)\in\Nbb$.
   Since $R_T = \sum_{t=1}^T r_t$ and $r_t=\tilde{r}_t$ with probability $\geq 1-3\delta/8$, we obtain~\eqref{eqn:boi-inst-regret-1} with probability $\geq 1-\delta$ with union bound. 
\end{proof} 

Proof of Theorem~\ref{thm:boi-cumulative-regret} is given next.
\begin{proof}
    We first prove scenario 1.
    From Lemma~\ref{lem:variancebound}, we know $\sum_{t=1}^T\sigma_{t-1}(\xbm_t) = \mathcal{O}(\sqrt{T\gamma_T})$. By~\eqref{def:eta} and definition of $c_y(t)$ in~\eqref{eqn:EI-sigma-star-assp}, we can write 
    \begin{equation} \label{eqn:ei-boi-cumu-pf-1}
  \centering
  \begin{aligned}
       \eta_T^{1/2}= \mathcal{O}(T^{\frac{\alpha}{2}} \log^{\frac{1}{2}}(T)), \ 
       \beta_T^{1/2} = \mathcal{O}( \log^{\frac{1}{2}}(T)),
       \ c_y(T) = \mathcal{O}( \log^{\frac{1}{2}}(T)),\ c_{\beta}(T) = \mathcal{O}( \log^{\frac{1}{2}}(T)).
    \end{aligned}
  \end{equation} 
     From Lemma~\ref{lemma:boi-inst-regret}, the cumulative regret bound is 
    \begin{equation} \label{eqn:ei-boi-cumu-pf-2}
  \centering
  \begin{aligned}
         \mathcal{O}(R_T)=&\mathcal{O} ( T^{\frac{\alpha}{2}} \log(T) \sqrt{T\gamma_T}+T^{1-\frac{\alpha}{2}}+T^{\frac{\alpha}{2}} \log(T)\sqrt{T}+n_y(T))\\ 
              =& \mathcal{O}( \max\{T^{\frac{1}{2}+\frac{\alpha}{2}} \log(T)\sqrt{\gamma_T}, T^{1-\frac{\alpha}{2}},n_y(T)\} ).\\
    \end{aligned}
  \end{equation} 
  Using $\gamma_T=\mathcal{O}(\log^{d+1}(T))$ for SE kernel, we choose $\alpha=\frac{1}{2}$ and $R_T=\mathcal{O}(\max\{T^{\frac{3}{4}} \log^{\frac{d+3}{2}}(T),n_y(T)\})$.
  Using $\gamma_T=\mathcal{O}(T^{\frac{d}{2\nu+d}}\log^{\frac{2\nu}{2\nu+d}}(T))$ for Matérn kernel, we choose $\alpha=\frac{\nu}{2\nu+d}$. Thus, $R_T=\mathcal{O}(\max\{T^{\frac{3\nu+2d}{4\nu+2d}}\log^{\frac{3\nu+d}{2\nu+d}}(T),n_y(T)\})$.

   We note that if $n_y(T)=o(T)$ does not hold, then there exists $c_n>0$ such that $n_y(T)\geq c_n T$.
   For scenario 2, by definition~\eqref{eqn:eft-y}, there are $c_n T$ times that $y^+_{t-1}-f(\xbm^*) < \beta_t^{1/2}\sigma_{t-1}(\xbm_t)$. Denote the index set where $E^y(t)$ is false as $I^y(t)$. At $T$ samples, we have $|I^y(T)|\geq c_n T$.
   From Lemma~\ref{lem:variancebound}, we have~\eqref{eqn:var-1}. For any $n\in\Nbb$ and $n\leq T$,  $\sigma_{t-1}^2(\xbm_t)\geq \frac{C_\gamma \gamma_T}{n}, t\leq T$ at most $n$ times. Let $n = [c_n T/2]$, where $[\cdot]$ is the largest integer smaller than $[\cdot]$. Then, $c_nT-1 \leq 2n\leq c_n T$. 
    Consider the indices $i\in I^y(T)$ and their corresponding $\xbm_i$. There exists $n\leq t\leq 2n$, such that $t\in I^y(T)$ and $\sigma_{t-1}^2(\xbm_{t})<\frac{C_\gamma \gamma_T}{n}\leq \frac{2C_\gamma \gamma_T}{c_n T-1}$.
  Therefore,
    \begin{equation} \label{eqn:ei-boi-cumu-pf-3}
  \centering
  \begin{aligned}
r_T^s=y^+_{T-1}-f(\xbm^*) \leq y^+_{t-1}-f(\xbm^*) < \beta_{t}^{1/2} \sigma_{t-1}(\xbm_{t}) <\beta_{t}^{1/2} \sqrt{\frac{C_\gamma \gamma_T}{n}}.
    \end{aligned}
  \end{equation}
  By Lemma~\ref{lem:gammarate},  for SE kernel $\gamma_T=\mathcal{O}(\log^{d+1}(T))$, and the noisy simple regret bound is $\mathcal{O}(T^{-\frac{1}{2}}\log^{\frac{d+2}{2}}(T))$. For Matérn kernel, the noisy simple regret bound is $\mathcal{O}(T^{\frac{-\nu}{2\nu+d}}\log^{\frac{2\nu+0.5d}{2\nu+d}}(T))$.  
\end{proof}

%% file: Sections/Appx-experiments.tex
\section{Additional experiments with the square exponential (SE) kernel} 
\label{appdx:experiments}
\begin{figure}[!t]
\centering
\begin{subfigure}{0.33\linewidth}
\includegraphics[width=\linewidth]{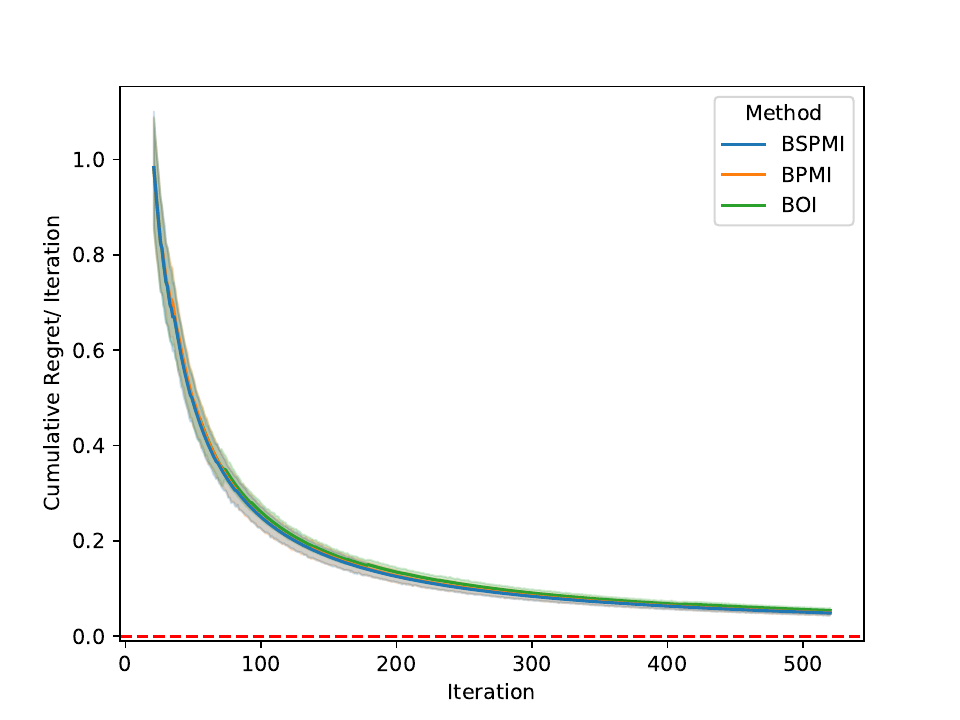} 
\caption{Branin 2D, $\sigma = 0.001$}
\end{subfigure}\hfill
\begin{subfigure}{0.33\linewidth}
\includegraphics[width=\linewidth]{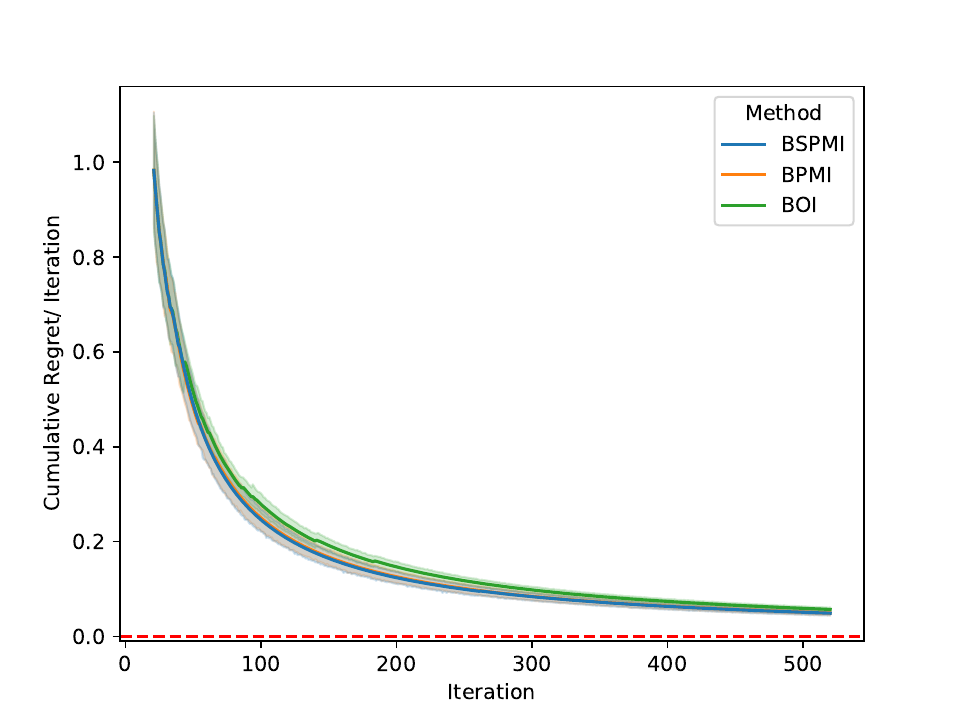}
\caption{Branin 2D, $\sigma = 0.01$}
\end{subfigure}\hfill
\begin{subfigure}{0.33\linewidth}
\includegraphics[width=\linewidth]{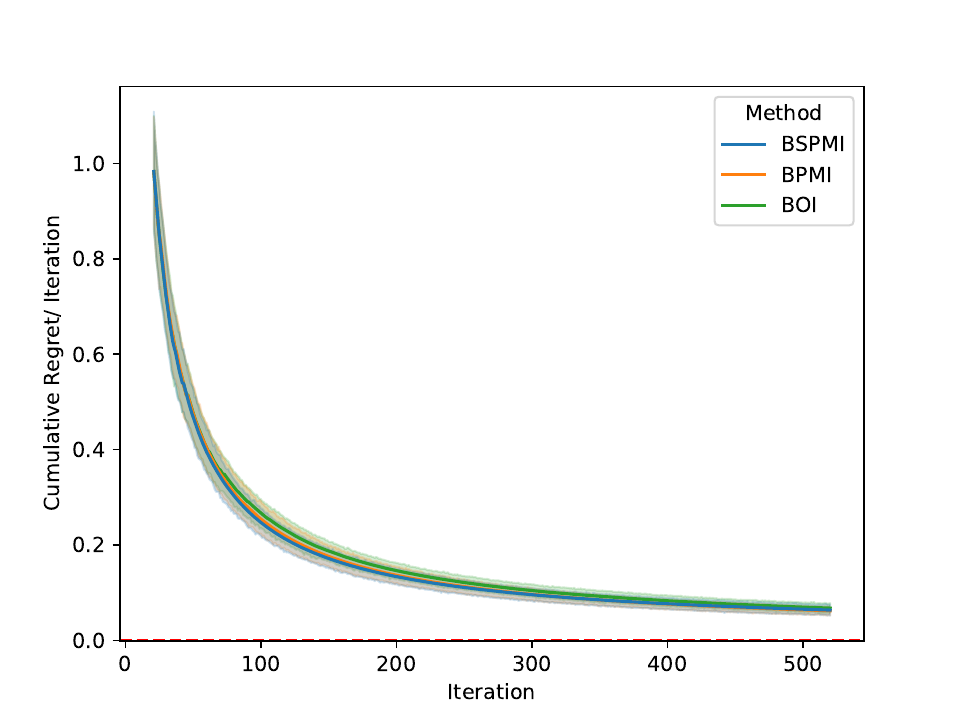}
\caption{Branin 2D, $\sigma = 0.1$}
\end{subfigure}

\begin{subfigure}{0.33\linewidth}
\includegraphics[width=\linewidth]{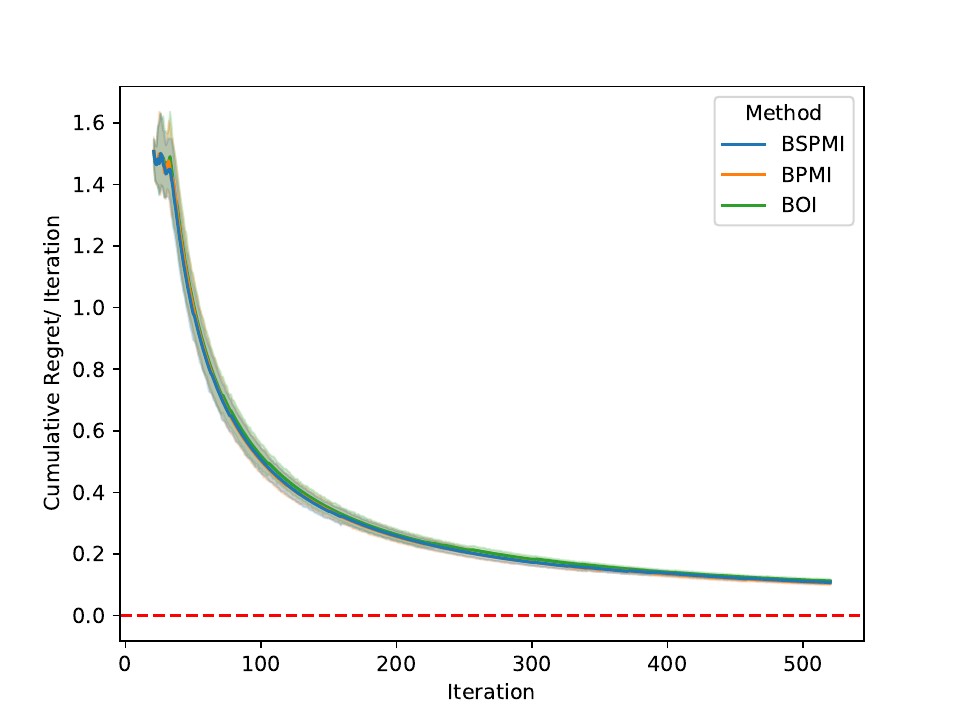} 
\caption{Styblinski-Tang 2D, $\sigma = 0.001$}
\end{subfigure}\hfill
\begin{subfigure}{0.33\linewidth}
\includegraphics[width=\linewidth]{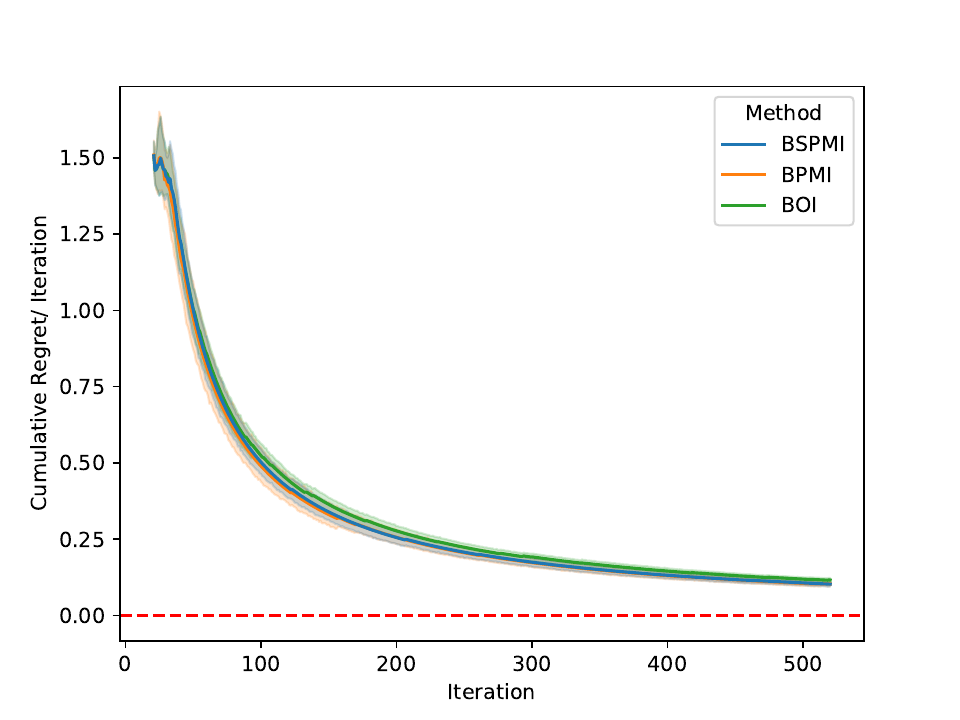}
\caption{Styblinski-Tang 2D, $\sigma = 0.01$}
\end{subfigure}
\begin{subfigure}{0.33\linewidth}
\includegraphics[width=\linewidth]{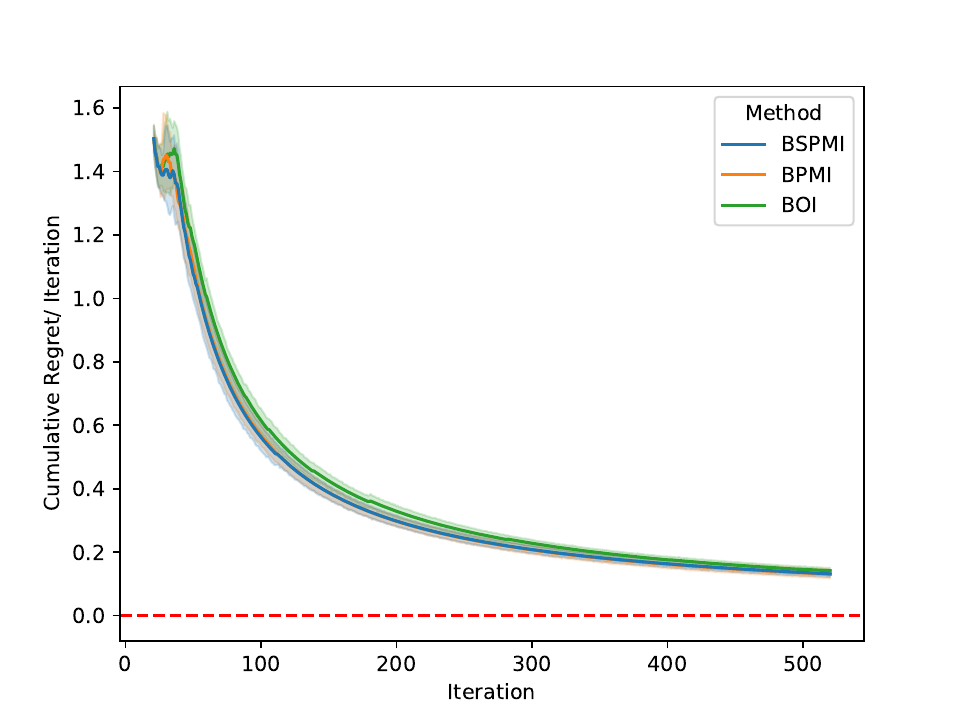}
\caption{Styblinski-Tang 2D, $\sigma = 0.1$}
\end{subfigure}

\begin{subfigure}{0.33\linewidth}
\includegraphics[width=\linewidth]{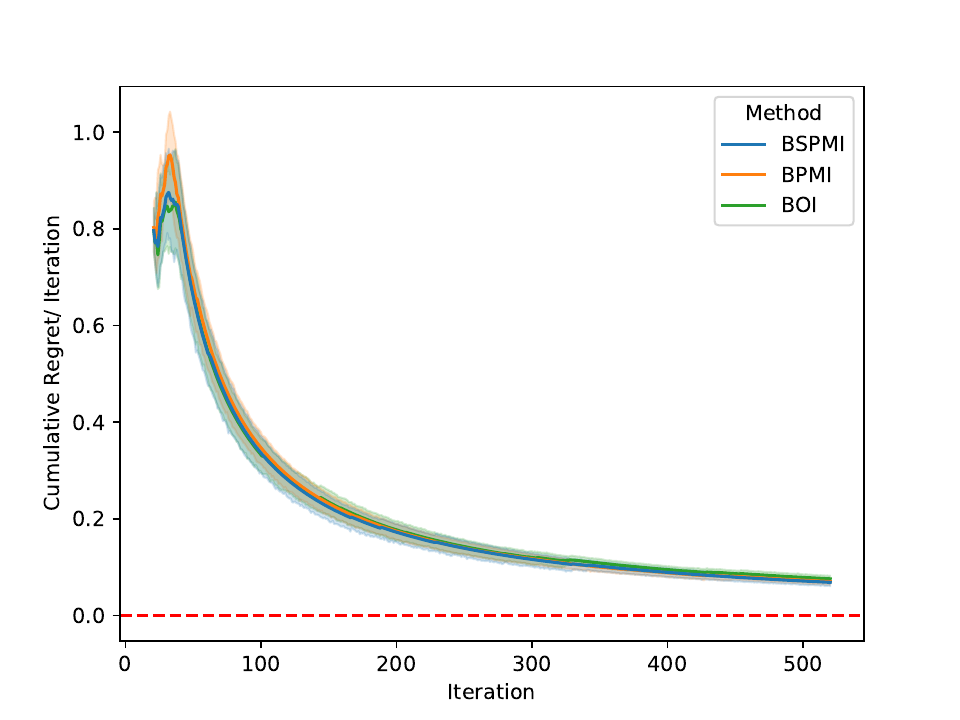} 
\caption{Camel 2D, $\sigma = 0.001$}
\end{subfigure}\hfill
\begin{subfigure}{0.33\linewidth}
\includegraphics[width=\linewidth]{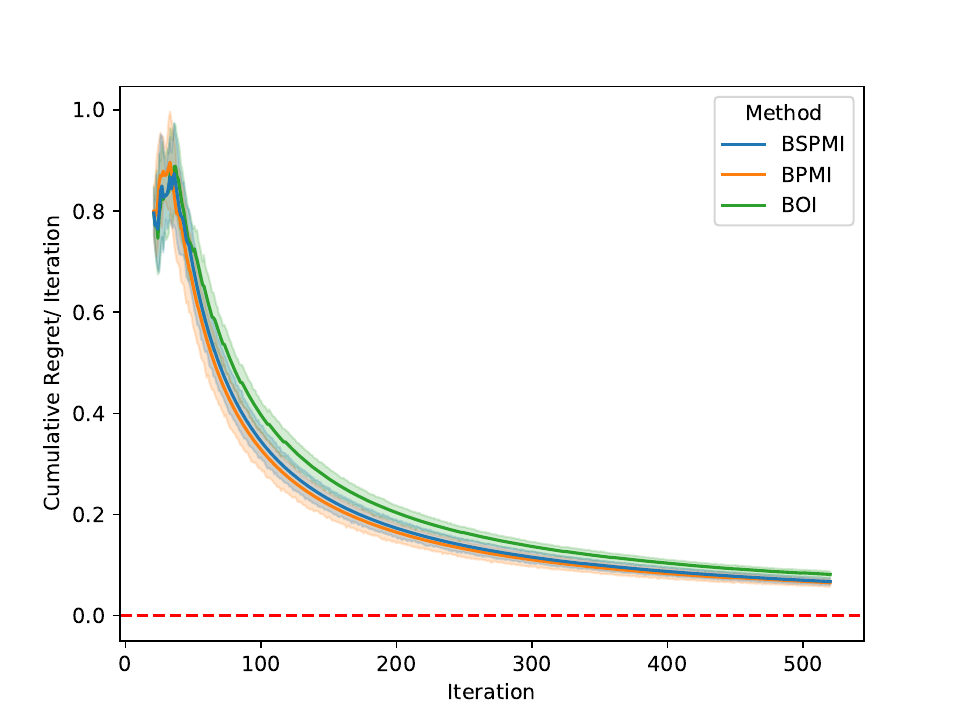}
\caption{Camel 2D, $\sigma = 0.01$}
\end{subfigure}
\begin{subfigure}{0.33\linewidth}
\includegraphics[width=\linewidth]{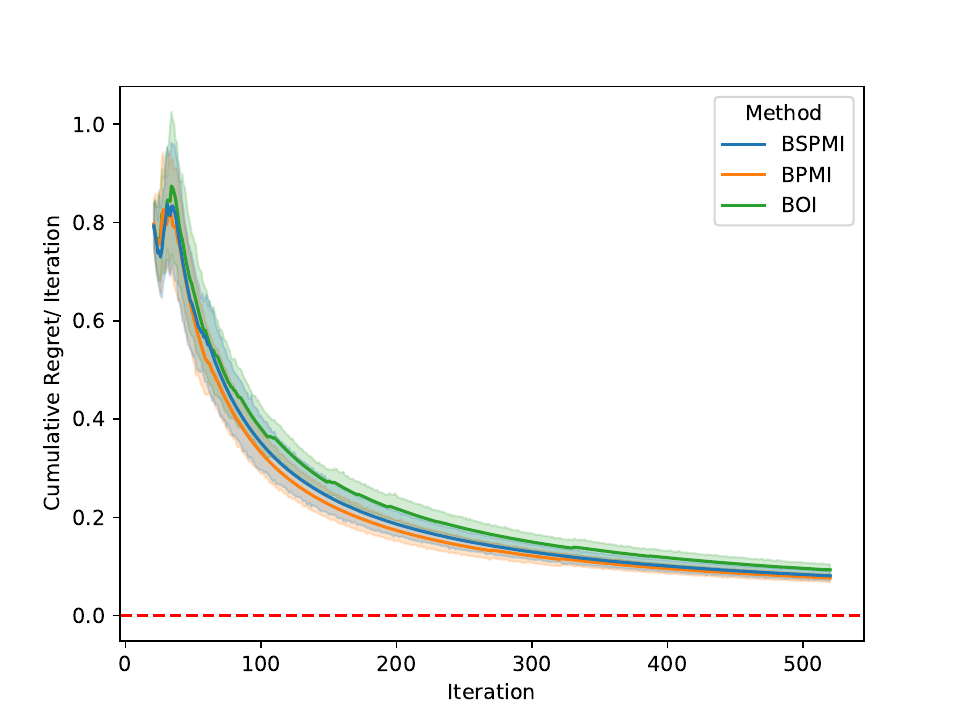}
\caption{Camel 2D, $\sigma = 0.1$}
\end{subfigure}

\begin{subfigure}{0.33\linewidth}
\includegraphics[width=\linewidth]{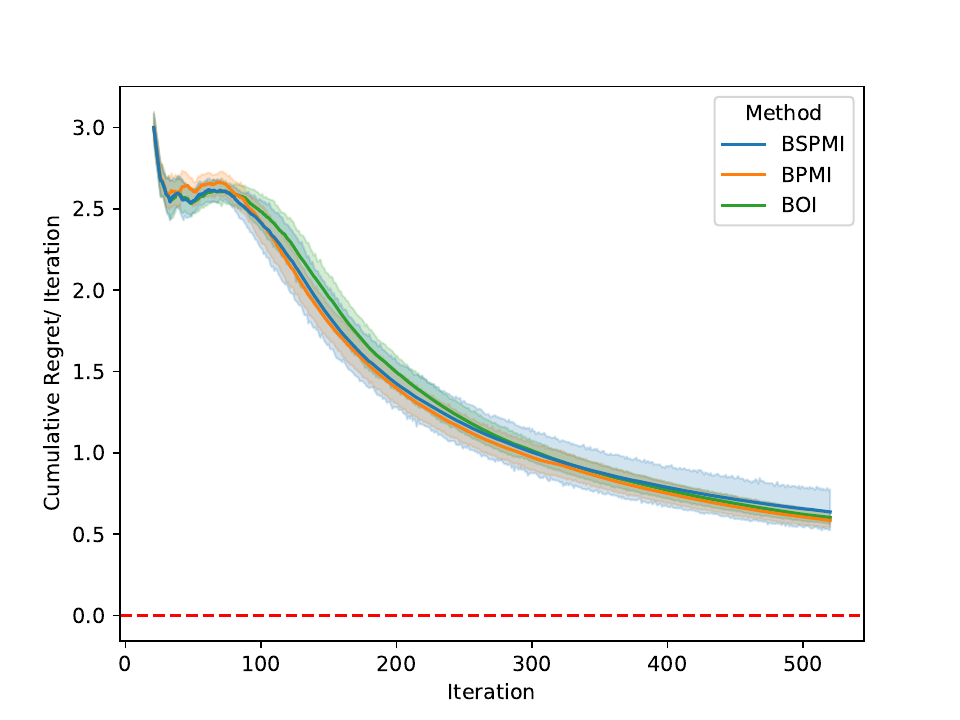} 
\caption{Schwefel 2D, $\sigma = 0.001$}
\end{subfigure}\hfill
\begin{subfigure}{0.33\linewidth}
\includegraphics[width=\linewidth]{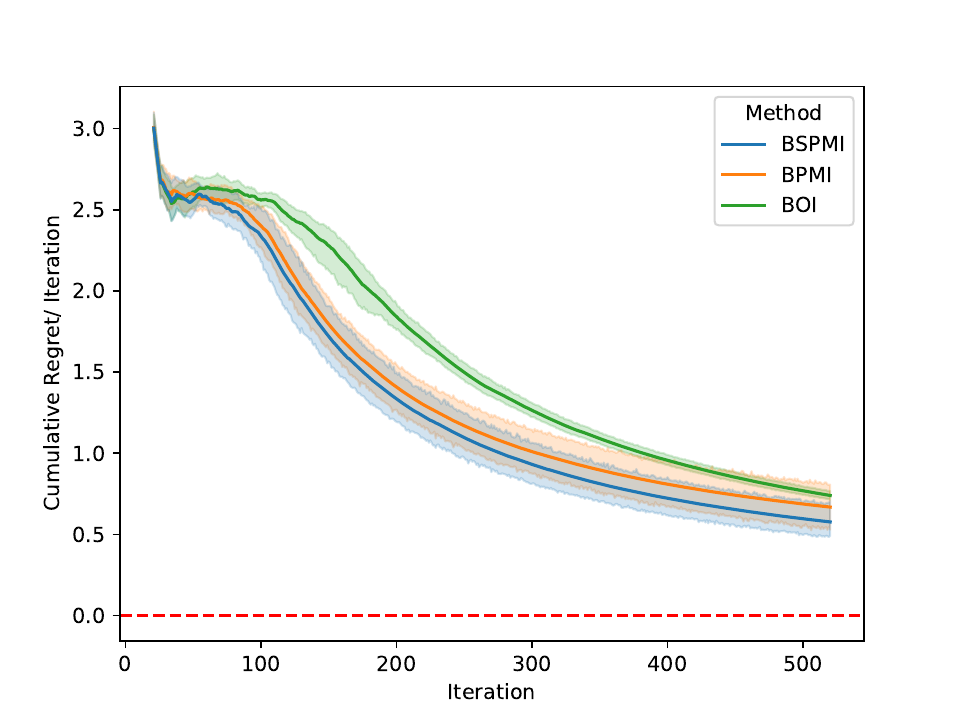}
\caption{Schwefel 2D, $\sigma = 0.01$}
\end{subfigure}
\begin{subfigure}{0.33\linewidth}
\includegraphics[width=\linewidth]{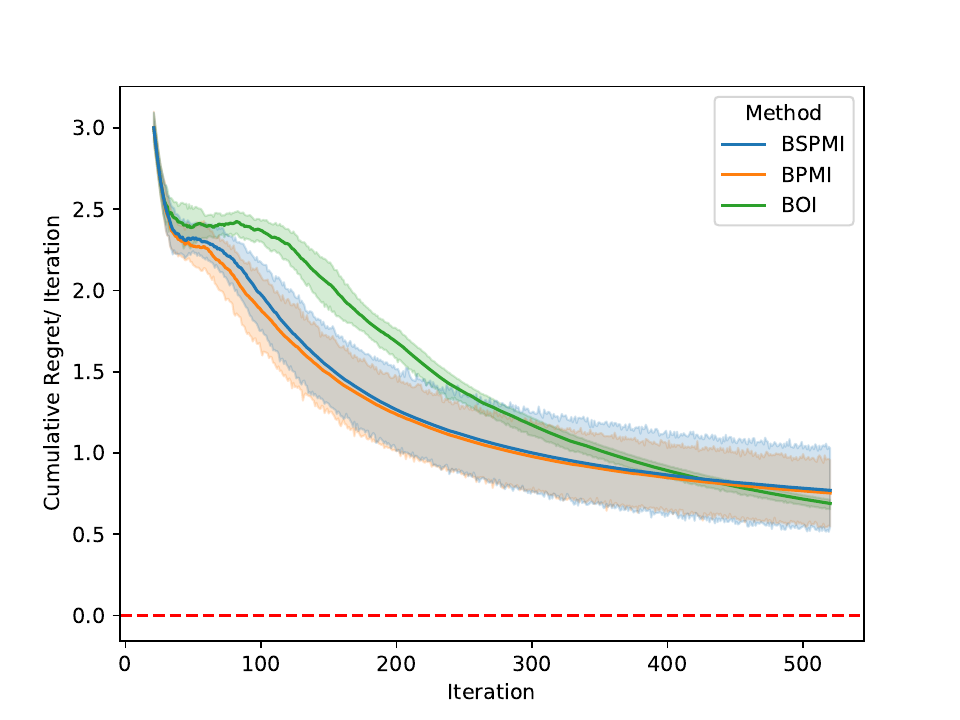}
\caption{Schwefel 2D, $\sigma = 0.1$}
\end{subfigure}

\caption{Cumulative regret/the number of iterations for different GP-EI algorithms with different incumbents on four 2D test functions (Branin 2D, Styblinski-Tang 2D, Camel 2D and Schwefel 2D) with the square exponential (SE) kernel.}
\label{fig:results3}
\end{figure}

\begin{figure}[!t]
\centering
\begin{subfigure}{0.33\linewidth}
\includegraphics[width=\linewidth]{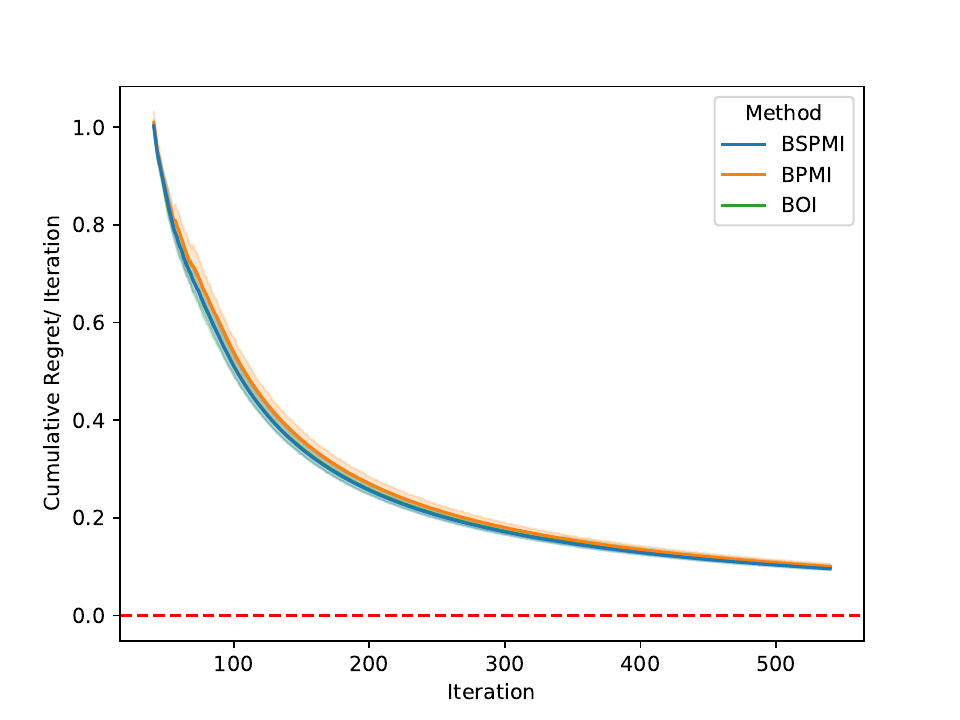} 
\caption{Rosenbrock 4D, $\sigma = 0.001$}
\end{subfigure}\hfill
\begin{subfigure}{0.33\linewidth}
\includegraphics[width=\linewidth]{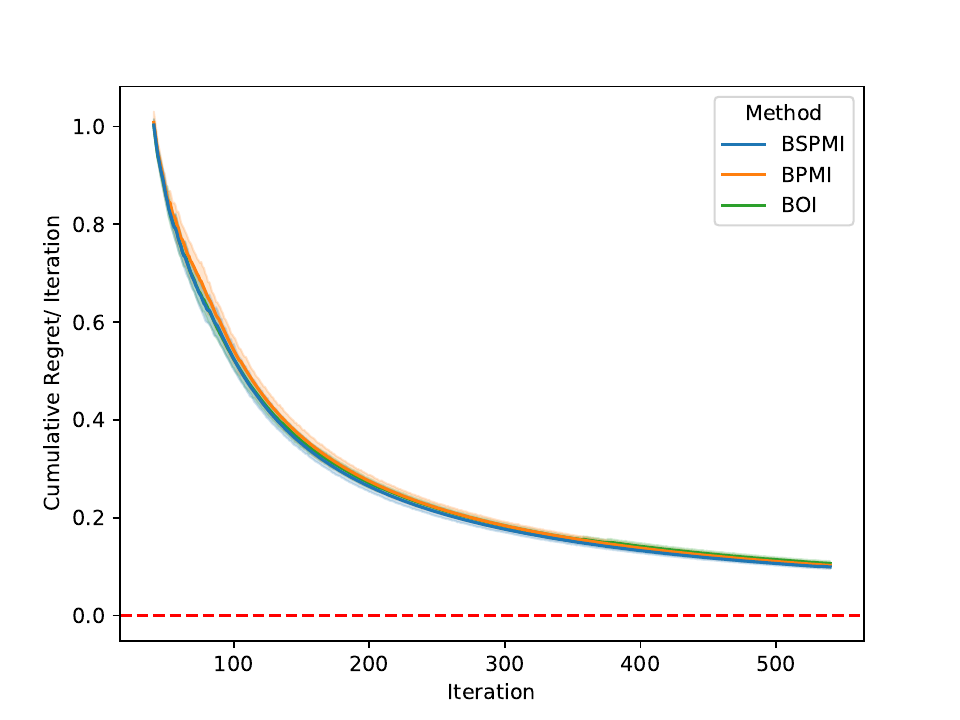}
\caption{Rosenbrock 4D, $\sigma = 0.01$}
\end{subfigure}
\begin{subfigure}{0.33\linewidth}
\includegraphics[width=\linewidth]{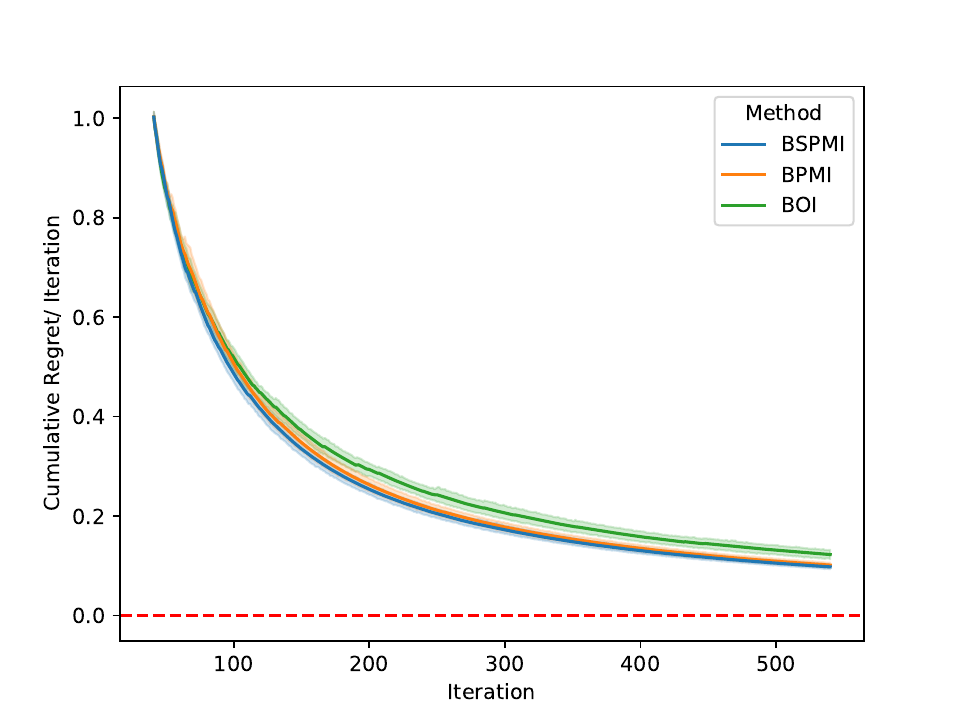}
\caption{Rosenbrock 4D, $\sigma = 0.1$}
\end{subfigure}

\begin{subfigure}{0.33\linewidth}
\includegraphics[width=\linewidth]{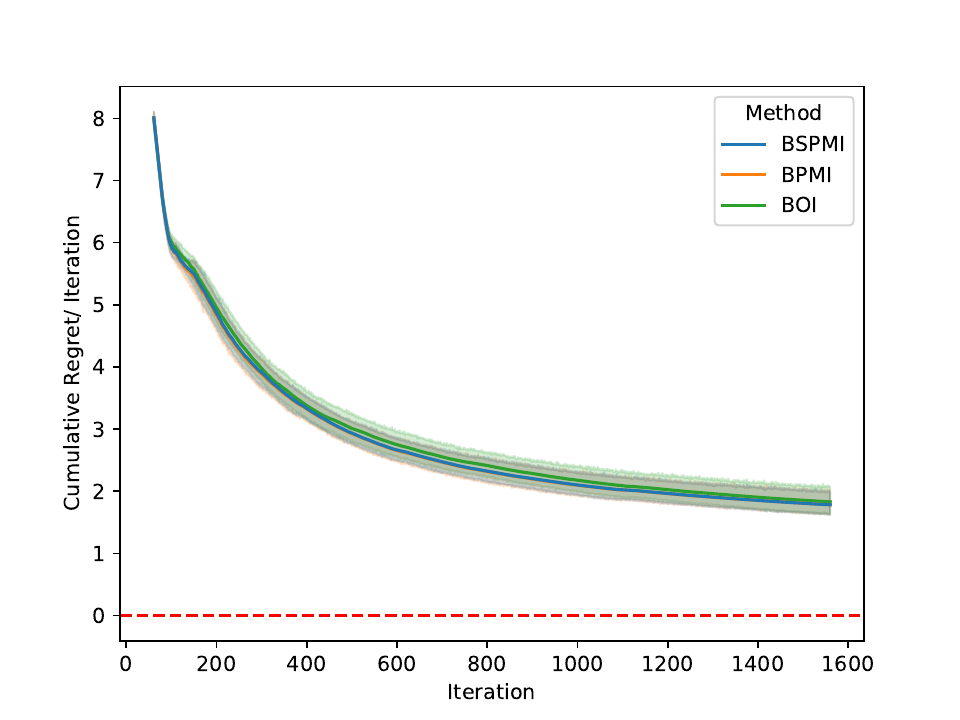} 
\caption{Hartmann 6D, $\sigma = 0.001$}
\end{subfigure}\hfill
\begin{subfigure}{0.33\linewidth}
\includegraphics[width=\linewidth]{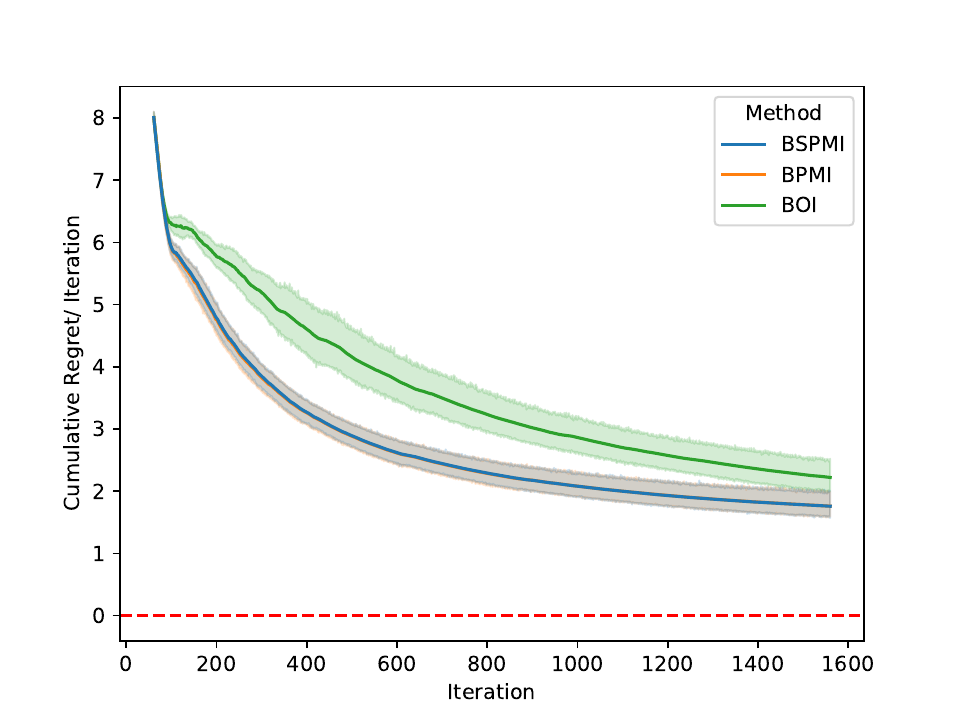}
\caption{Hartmann 6D, $\sigma = 0.01$}
\end{subfigure}
\begin{subfigure}{0.33\linewidth}
\includegraphics[width=\linewidth]{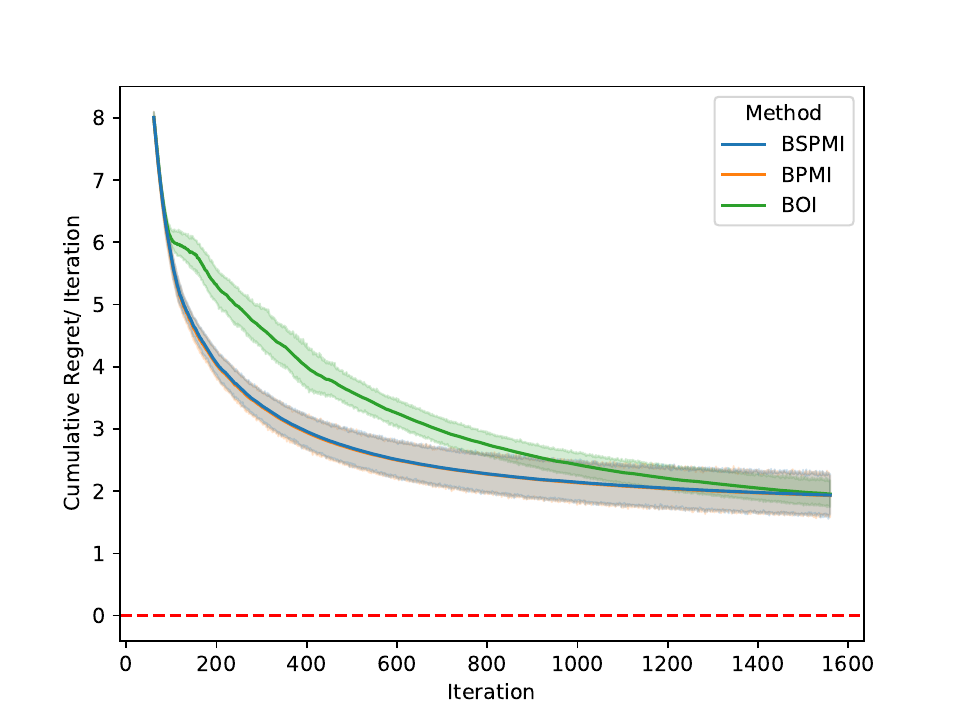}
\caption{Hartmann 6D, $\sigma = 0.1$}
\end{subfigure}

\caption{Cumulative regret/the number of iterations for different GP-EI algorithms with different incumbents on two higher dimension test functions (Rosenbrock 4D and Hartmann 6D) with the square exponential (SE) kernel.}
\label{fig:results4}
\end{figure}